\documentclass[11pt,letterpaper]{article}
\usepackage[margin=1in]{geometry}
\usepackage{natbib}

\usepackage{parskip}
\usepackage[utf8]{inputenc} 
\usepackage[T1]{fontenc}    

\usepackage{microtype}             
\usepackage{amsmath}               
\usepackage{amssymb}               
\usepackage{amsthm}                
\usepackage{mathtools}             
\usepackage{bm}                    
\usepackage{booktabs}
\usepackage{array}
\usepackage{graphicx}              
\usepackage{float}                 
\usepackage{wrapfig}               

\usepackage{algorithm}             
\usepackage{algpseudocode}         
\usepackage{multirow}              
\usepackage{tabularx}              
\usepackage{colortbl}              

\usepackage{hyperref}              
\hypersetup{
    colorlinks=true,
    linkcolor=blue,
    citecolor=blue,
    urlcolor=blue,
    pdfauthor={},
    pdfsubject={},
    pdfkeywords={}
}

\usepackage{enumitem}              

\usepackage{tikz}
\usetikzlibrary{shapes.geometric, arrows}

\tikzstyle{bubble} = [ellipse, draw=black, fill=blue!10, text width=6em, text centered, minimum height=2em]
\tikzstyle{line} = [draw, -stealth]

\usepackage{xcolor}

\newtheorem{theorem}{Theorem}[section]
\newtheorem{proposition}[theorem]{Proposition}
\newtheorem{lemma}[theorem]{Lemma}
\newtheorem{corollary}[theorem]{Corollary}
\theoremstyle{definition}
\newtheorem{definition}[theorem]{Definition}
\newtheorem{assumption}[theorem]{Assumption}

\newtheorem{remark}[theorem]{Remark}

\usepackage{comment}   

\newcommand{\R}{\mathbb R}
\newcommand{\KL}{\mathrm{KL}}
\newcommand{\Hel}{\mathrm{H}}

\newcommand{\kld}{\mathrm{kl}}
\newcommand{\Mix}{\mathrm{Mix}}
\newcommand{\E}[1]{\mathbb E\left[#1\right]}
\newcommand{\cG}{\mathcal G}
\newcommand{\argmin}{\mathop{\arg\min}}

\title{Intrinsic Information Rates of Summaries in\\ Simulation-Based Inference}

\author{Vansh Bansal\thanks{Corresponding author: \texttt{vansh@utexas.edu}} ,\quad Pradip Sharma Poudel, \quad James G. Scott\\ Department of Statistics and Data Science\\ University of Texas at Austin}

\date{}

\begin{document}
\maketitle
\begin{abstract}
Simulation-based inference often encodes observations $Y$ into a fixed-dimensional summary $S$, with the posterior conditioning only on $S$. Any such summary imposes a finite information budget on inference, whereas exact preservation requires infinite information whenever posteriors vary continuously with the data. The budget typically emerges from summary architecture and training rather than an explicit measurement or choice. Better performance may therefore reflect better encoding or simply a larger budget. 
In this paper, we make the budget explicit and ask two questions: how large must it be, and how best is it spent? To answer them, we model the summary as a (possibly stochastic) lossy channel from $Y$ to $S$. For the first question, we show that this budget, measured by the mutual information $I(Y;S)$, grows unboundedly as the posterior concentrates: in regular $d$-dimensional models, bounded posterior distortion requires $\tfrac d2\log n-O(1)$ nats, even when a fixed-dimensional sufficient statistic exists. More generally, the required budget is set by the distribution of the random posterior, and we bound it from above by quantizing posteriors and from below by converses that weight each posterior state by its probability. For the second question, we show that a finite budget is best spent stochastically. Under a common-support condition on the simulator, which covers the continuous observation spaces typical of simulators, a deterministic summary necessarily loses something: at every positive budget too small for exact preservation, some stochastic summary achieves \textit{strictly} lower distortion than every deterministic summary and every time-shared mixture of deterministic summaries within the same budget. Motivated by these results, we insert a binary bottleneck into amortized neural posterior estimators, measure each trained summary's rate, and compare summaries at matched rate, which
removes the budget as a confounder in architecture comparisons. Across four simulator benchmarks and several encoder families, stochastic bottlenecks achieve favorable rate--loss tradeoffs, and architecture rankings change at different information budgets.

\end{abstract}

\section{Introduction}
\label{sec:introduction}
Simulation-based inference (SBI) approximates the posterior for a parameter $\theta$ from simulated parameter--observation pairs when the likelihood is intractable \citep{cranmer2020frontier}. Classical
methods either keep parameter draws whose simulated data lie close to the
observed data, as in approximate Bayesian computation
\citep[ABC;][]{tavare1997inferring, marin2012approximate}, or fit a Gaussian model to
simulated summaries, as in synthetic likelihood \citep{wood2010statistical}.
Neural methods, many of them amortized, instead train a network and often approximate one of three objects: the
posterior \citep{papamakarios2016fast, radev2020bayesflow, chen2024conditional}, the likelihood
\citep{papamakarios19a, lueckmann19a}, or the
likelihood-to-evidence ratio \citep{thomas2022likelihood,
hermans2020ratio}; see
\citet{lueckmann2021benchmarking} and \citet{zammit2025neural} for
comparisons and reviews. In most of these pipelines, a summary $S$, either handcrafted
or learned with neural networks, first maps a data set $Y$ to a
fixed-dimensional representation, and the posterior decoder (an ABC
acceptance rule, a likelihood or ratio estimator, or a neural posterior estimator)
sees the data only through $S$. This makes the summary a central design choice, in that a summary allows for accurate posterior reconstruction only insofar as it can correctly distinguish data sets that correspond to different posteriors \citep{blum2013comparative, fearnhead2012constructing, jiang2017learning, chan2018likelihood, aakesson2021convolutional, hoffmann2026unifying}. 

This paper begins from the premise that every such pipeline already operates under an information budget, usually an implicit one. A summary passes only a limited amount of information about the data to the posterior decoder: finite precision, injected noise, bottleneck layers, and training all limit how much gets through.  Therefore, this budget is not chosen explicitly; instead, it emerges as a byproduct of decisions made for other reasons. The budget is not often measured explicitly, either, and the quantities usually reported, such as output dimension and parameter count, are poor proxies. Even one real-valued coordinate can carry an enormous amount of information, while a wide vector passed through noise can carry almost none. Furthermore, this budget is binding, in that whenever posteriors vary continuously with the data, preserving the posterior exactly requires infinite information, and no finite budget suffices (Proposition~\ref{prop:shape}). The upshot is that the information budget is typically a confounder for experiments on summary choices: when one summary outperforms another, we cannot tell whether it  distinguishes the data better, or merely spends a larger information budget. 

In this work, we make the information budget explicit by modeling the summary as a possibly randomized channel $Y\to S$. The budget a summary spends is its rate, the mutual information $I(Y;S)$, and its distortion is the expected KL divergence from the full-data posterior $\Pi_Y$ to the summary-based posterior $\Pi_S$. The intrinsic rate $R^\star(\varepsilon)$ is then the smallest budget sufficient for distortion at most $\varepsilon$. This quantity depends only on the prior and the simulator, not on the architecture or optimization procedure used to learn a summary or even the inference method.

We address two questions about this budget:
\begin{enumerate}[leftmargin=*,itemsep=1pt]
\item At a given tolerance $\varepsilon$, what properties of a Bayesian experiment govern the intrinsic rate $R^\star(\varepsilon)$?
\item At a prescribed rate $R$, can a stochastic summary achieve \textit{strictly} lower posterior distortion than every deterministic summary, including their time-sharing mixtures (i.e, a random choice among deterministic summaries, revealed to the decoder)?
\end{enumerate}

For the first question, one part of the answer is that the budget must grow with sample size. This holds even when a fixed-dimensional sufficient statistic exists: the statistic does not grow, but the precision at which it must be transmitted does. In regular $d$-dimensional models, keeping the posterior error bounded requires at least $\tfrac d2\log n-O(1)$ nats (Lemma~\ref{lem:cb_benchmark}), which follows by combining the identity between posterior error and lost parameter information \citep{tishby2000information,harremoes2007revisited} with a classical expansion of \citet{clarkeBarron1994}. Beyond parametric models, what sets the required budget is the collection of posteriors the experiment can produce, together with how often each occurs. The summary's job is, in effect, to tell the decoder which posterior to report. If the experiment produces many well-separated posteriors and visits them often, the summary needs many distinguishable messages. But if most posteriors are rarely visited, the summary can get by with less information. We make this intuition precise in both directions: quantizing the posteriors gives an achievable budget (Theorem~\ref{thm:avg_quant}), and converses that weight each posterior by its probability give lower bounds (Theorem~\ref{thm:smallball}).

For the second question, the answer is that a finite budget is best spent with a randomized summary. This is not obvious, and at first sight looks false: if we limited the number of values a summary may take, rather than its information, deterministic summaries would be optimal (Theorem~\ref{thm:avg_quant}). But an information budget behaves differently. The
Gaussian model gives a first hint, since its optimal summary at every positive rate is the sufficient sample mean plus deliberate noise (Proposition~\ref{thm:gaussian_exact}). Our main
theorem shows that the Gaussian hint is general, under the common-support condition that no observation can rule out any parameter value (Assumption~\ref{ass:common_support}), which typical noisy simulators satisfy. At every positive budget too small for exact posterior preservation, some stochastic summary achieves strictly lower posterior error than every deterministic summary and every time-shared mixture of them within the same budget (Theorem~\ref{thm:separation}). The intuition behind the proof is that we can occasionally replace a summary's label, without telling the decoder, by a random draw from the its own output distribution. This hidden swap saves more rate than it costs in posterior error, and spending the savings on a finer quantizer beats the original. The argument needs neither finite observation spaces nor the existence of optimal quantizers. The gap is also not a curvature accident: the two frontiers can share slopes at zero rate and still separate at every positive rate (Proposition~\ref{prop:gaussian_same_slope}). Finally, the common support assumption is sufficient, rather than a characterization of all experiments with a stochastic advantage; see Section \ref{sec:remarks_separation}. \textit{To our knowledge, this is the first formal result in the literature to establish a strict advantage in such generality.}

Finally, we test this empirically in amortized neural posterior estimation (NPE), where the summary and the posterior estimator are trained jointly on draws from the simulator's prior predictive
distribution. Because the information budget cannot be set directly, we measure it: every summary is trained with a binary bottleneck, giving all architectures a common code alphabet and decoder, and each trained model's deployed rate is estimated from its code distribution. Comparisons are then made at matched rate. A conjugate count benchmark enumerates all deterministic count encoders and verifies the second answer exactly, with decoder error far smaller than the gap. Four simulator benchmarks compare stochastic deployment against annealed training with hard deployment, over several encoders and three training seeds. Stochastic bottlenecks trace better rate--loss curves throughout, and architecture rankings change with the budget, so comparisons at unmatched budgets can pick different winners. These comparisons involve trained models, so optimization and decoder error contribute, whereas the count benchmark isolates the
gap exactly.

The paper proceeds from posterior preservation and its parametric consequence (Section~\ref{sec:preliminaries}), through posterior geometry (Section~\ref{sec:posterior_geometry}), to strict stochastic separation (Section~\ref{sec:randomization_essential}). Sections~\ref{sec:npe_protocol} and~\ref{sec:experiments} develop and evaluate the NPE comparison protocol. Appendix~\ref{app:related_works} gives a detailed account of the related work, including connections to literature on summary selection in likelihood-free inference, approximate sufficiency, information theory and compression.

\section{Posterior preservation as an information constraint}
\label{sec:preliminaries}
Let $\theta$ have prior $\pi$, let $Y\mid\theta$ follow a simulator, and let $S$ be generated by a measurable channel $q(s\mid y)$. Thus $\theta\to Y\to S$ is a Markov chain. We work on standard Borel spaces, assume regular conditional distributions and measurable divergences, and write $\Pi_Y=P_{\theta\mid Y}$ and $\Pi_S=P_{\theta\mid S}$. Although $I(\theta;Y)$ may grow with sample size, we assume it is finite for each experiment else every finite-rate summary will yield infinite posterior distortion. 
\begin{definition}[Intrinsic information rate]
\label{def:intrinsic_rate_new}
For a summary $S$, define its posterior distortion by $\Delta_\KL(S):=\mathbb E\,\KL(\Pi_Y\Vert\Pi_S)$ and information rate by. $R(S):=I(Y;S)$
The intrinsic rate required for tolerance $\varepsilon\ge0$, and the minimum distortion under a rate budget $R\ge0$ are
\begin{equation}
\label{eq:rate_def_new}
R^\star(\varepsilon):=\inf_{q(s\mid y):\,\Delta_\KL(S)\le\varepsilon} I(Y;S), \qquad D^\star(R):=\inf_{q(s\mid y):\,I(Y;S)\le R}\Delta_\KL(S).
\end{equation}
\end{definition}

Zero distortion corresponds to exact posterior sufficiency \citep{lindley1956measure}, potentially requiring an unbounded information rate. A deterministic continuous statistic may supply that rate even in one dimension, however, quantization or stochastic encoding makes precision, and hence rate, an explicit modeling choice for approximate sufficiency \citep{joyce2008approximately, hayashi2018minimum}. 


\begin{lemma}[Posterior loss and compensation]
\label{lem:identity}
For every summary channel,
\begin{equation}
\label{eq:identity}
\Delta_\KL(S)=I(\theta;Y\mid S)=I(\theta;Y)-I(\theta;S),
\qquad \Pi_S:=\mathbb E[\Pi_Y\mid S].
\end{equation}
For any measurable reconstruction $Q_S$ with finite expected divergence,
\begin{equation}
\label{eq:compensation}
\mathbb E\,\KL(\Pi_Y\Vert Q_S)
=\Delta_\KL(S)+\mathbb E\,\KL(\Pi_S\Vert Q_S).
\end{equation}
Consequently $R^\star(\varepsilon)\ge[I(\theta;Y)-\varepsilon]_+$ and $D^\star(R)\ge[I(\theta;Y)-R]_+$.
\label{eq:universal_converse_new}
\end{lemma}

Equation~\eqref{eq:identity} identifies $R^\star$ with the information-bottleneck frontier \citep{tishby2000information,harremoes2007revisited} of the Bayesian experiment. When conditional entropies are well-defined and finite with respect to a common reference measure, the same identity reads $\Delta_\KL(S)=H(\theta\mid S)-H(\theta\mid Y)$, the expected-posterior-entropy selection criterion of \cite{hoffmann2026unifying}\footnote{\cite{hoffmann2026unifying} minimize expected-posterior-entropy within a pre-selected summary class. This class can restrict dimension, architecture, network size, or the number of selected statistics, but it places no direct constraint on the rate $I(Y;S)$. For example, if $Y=(U,V)$ with $U$ sufficient and $V$ carrying no information about $\theta$, then $U$ and $(U,V)$ induce identical posteriors and identical distortion, so no fidelity criterion separates them, while $I(Y;U)<I(Y;(U,V))$. Only a criterion that charges for information rate prefers the smaller summary. Our work is therefore complementary to theirs.}. The KL and mutual-information definitions above remain applicable even when such an entropy difference is unavailable. Equation~\eqref{eq:compensation} further separates information discarded by the summary from error introduced by its decoder. This separation also guides our interpretation of the empirical NPE rate-loss frontiers shown in Section \ref{sec:experiments}.


Lemma \ref{lem:identity} further motivates the reason to make information rate explicit. The Bayes-information expansion of \cite{clarkeBarron1994} for regular parametric models combined with Equation \eqref{eq:identity} makes it immediate that as the data accumulates, a summary must resolve increasingly narrow and concentrated posterior distributions. Therefore, even when a fixed-dimensional sufficient statistic exists, it must be transmitted with an unbounded increasing precision, hence the information rate, to achieve bounded distortion. The following lemma states the result formally.
\begin{lemma}[Summary-rate consequence in parametric models]
\label{lem:cb_benchmark}
For a regular $d$-dimensional i.i.d.\ model satisfying the prior-averaged Clarke--Barron conditions, every summary with $\Delta_\KL(S_n)\le\varepsilon$ obeys
\begin{equation}
\label{eq:cb_rate_bound}
I(Y^n;S_n)\ge\frac d2\log\frac{n}{2\pi e}
+h(\pi)+\frac12\mathbb E_\pi\log\det\mathcal I(\theta)-\varepsilon+o(1),
\end{equation}
where $\mathcal I$ is the per-observation Fisher information. Thus, for fixed $d$ and $\varepsilon$, $R_n^\star(\varepsilon)\ge\tfrac d2\log n-O(1)$.
\end{lemma}
Appendix~\ref{app:parametric} gives an alternative lower bound from Bayes mean-square error. We further show via a matching upper bound that the leading rate of $O(\tfrac d2 \log n)$ is indeed achievable under a quadratic posterior chart condition, which is also verified for an exponential-family class beyond Gaussians.

\section{Posterior geometry and the information rate}
\label{sec:posterior_geometry}
In this section, we develop upper and lower bounds on the intrinsic rate. We start by identifying that $R^\star$ depends on the \emph{law} of $\Pi_Y$, not just its support.
We state the posterior-source reduction \citep{harremoes2007revisited,armstrong2026source} in the following lemma and then make the connection to the relevant geometry precise; proof and endpoint properties are provided in Appendix~\ref{app:proofs-prelim}.

\begin{proposition}[Posterior-source reduction]
\label{prop:reduction}
Let $Z:=\Pi_Y$. Then $P(\theta\in\cdot\mid Z)=Z$ almost surely, and
$R^\star(\varepsilon)=\inf_{P_{\widehat Z\mid Z}:\,\mathbb E\KL(Z\Vert\widehat Z)\le\varepsilon} I(Z;\widehat Z),$
where $\widehat Z$ ranges over probability-measures.
\end{proposition}
We use posterior-source reduction to show that achievable rates and converses depend on two features of the random posterior: how accurately a finite codebook represents it on average, and how much predictive probability one accurate reconstruction captures. 

\subsection{Achievability by posterior quantization}
\label{sec:achievability}
We first ask how well the random posterior $\Pi_Y$ can be represented by a
small codebook of fixed reference distributions.
For $M\ge1$, write the average distortion of the quantized posterior as 
\begin{equation}
\label{eq:Dav}
D^{\mathrm{av}}_M:=\inf_{Q_1,\ldots,Q_M}\mathbb E\min_{j\le M}\KL(\Pi_Y\Vert Q_j),
\end{equation}
 where the expectation weights each posterior by its occurrence under the simulator and prior. Our next result identifies the optimal posterior-distortion incurred by any \(M\)-symbol summary with the average distortion of the quantized random posterior defined in \eqref{eq:Dav}. 

\begin{theorem}[Average posterior quantization]
\label{thm:avg_quant}
The optimal distortion with at most $M$ summary values satisfies
\begin{equation*}
D^{\mathrm{av}}_M
=\inf_{g:\mathcal Y\to[M]}\Delta_\KL(g(Y))
=\inf_{q(s\mid y):\,|\mathcal S|\le M}\Delta_\KL(S).
\end{equation*}
Moreover $D^{\mathrm{av}}_1=I(\theta;Y)$, and
$R^\star(\varepsilon)\le\log N^{\mathrm{av}}(\varepsilon),$ where $N^{\mathrm{av}}(\varepsilon):=\inf\{M:D^{\mathrm{av}}_M<\varepsilon\}.$
\end{theorem}
The above infimum is remarkably achieved by a deterministic summary and the result holds for general Bayesian experiments in Borel spaces. We prove it by directly replacing any stochastic summary with a deterministic nearest-center rule and then using the cell barycenter as the representative, a trick often used for $k$-means clustering problems; see Appendix \ref{app:proofs-geometry}. 

Since the information rate of an $M$-symbol deterministic summary satisfies $R(S)=H(S)\le\log M$, Theorem \ref{thm:avg_quant} also yields an upper bound on the optimal rate-distortion frontier. The theorem further distinguishes \emph{cardinality} from \emph{information}: for any fixed cardinality, every stochastic summary can be replaced with a  deterministic summary, incurring no greater posterior distortion. However, this replacement preserves the alphabet-size constraint, but need not preserve mutual information. As we show in the Section \ref{sec:randomization_essential}, when the comparison happens at matched information rates rather than cardinality, randomized summaries achieve strictly lower distortion.


\subsection{Converses that retain predictive mass}
\label{sec:converse}
A packing of possible posteriors alone cannot lower-bound the rate of the observed experiment. For example, let $\theta=Y$ have probability $1-\delta$ at $0$ and probability $\delta/M$ at each of $1,\ldots,M$. Its posterior family always contains $M+1$ mutually singular distributions, but the constant summary has distortion $h_2(\delta)+\delta\log M\to0$. A converse therefore must retain the masses of those posterior states. Our next theorem states the probability-weighted analogue of a packing lower bound that accounts for the mass represented by a single reconstruction.
\begin{theorem}[Posterior small-ball bound; specialization of {\citet{xu2016information}}]
\label{thm:smallball}
For $t>\varepsilon$, define the largest predictive mass captured by a forward-KL ball as
$\mathcal Q_Z(t):=\sup_Q P\{\KL(Z\Vert Q)\le t\}>0.$ Then, the intrinsic rate is lower bounded as
$R^\star(\varepsilon)\ge
\left[\left(1-\frac{\varepsilon}{t}\right)\log\frac1{\mathcal Q_Z(t)}-\log2\right]_+.$
\end{theorem}
Hence, applying general small-ball information bounds to the random posterior yields a converse governed by predictive mass. Take, for example, a uniform law on $M$ sufficiently separated posterior states, $\mathcal Q_Z(t)=1/M$ and the bound has the familiar $\log M$ scale.  Appendix~\ref{app:converses} gives a complementary cluster bound that retains the entropy of unequal cluster weights, and a KL-Laplace bound \citep{riegler2018rate}. Together, these results identify posterior approximation error and predictive concentration as complementary controls on the information budget, and clarify why cardinality and mutual information lead to different conclusions about encoder randomization.

We now discuss the Gaussian location model where the exact rate-distortion frontier is analytically available. Its sufficient summary statistic, the sample mean, has a continuous distribution, so transmitting it exactly costs infinite information. Under a positive finite rate budget, the classical Gaussian information-bottleneck solution \citep{chechik2005gaussianib,asoodeh2020bottleneck} instead adds noise to the mean to achieve the following optimal frontier.

\begin{proposition}[Gaussian-location frontier]
\label{thm:gaussian_exact}
Let $\theta\sim N(0,\tau^2I_d)$ and $Y_i\mid\theta\sim N(\theta,\sigma^2I_d)$ independently, and put $a=n\tau^2/\sigma^2$. Then, for $R\ge0$ and $0<\varepsilon<\tfrac d2\log(1+a)$
\begin{equation}
\label{eq:gaussian_exact}
D_n^\star(R)=\frac d2\log(1+ae^{-2R/d}),
\qquad R_n^\star(\varepsilon)=\frac d2\log\frac{a}{e^{2\varepsilon/d}-1}
\end{equation}
\end{proposition}

It is rather surprising that the optimal frontier in \eqref{eq:gaussian_exact} can only be achieved by adding independent Gaussian noise, with variance chosen to meet the rate budget, to the sample mean summary statistic; no deterministic summary passing the same amount of information can achieve it. We prove this using a conditional entropy-power argument against arbitrary channels in Appendix~\ref{app:parametric}. 


\section{When are stochastic summaries strictly optimal?}
\label{sec:randomization_essential}

The simple example in Proposition \ref{thm:gaussian_exact} raises an important question: does randomization provably yield \textit{strictly} optimal frontiers even beyond known the Gaussian-location case? To formalize the comparison, let $\mathcal G$ contain all measurable deterministic quantizers (commonly called codes in information theory) $g:\mathcal Y\to\mathbb N$ with $H(g(Y))<\infty$, including constants. Since the information cost of a deterministic quantizer is its entropy, restricting to countable outputs loses no finite-rate deterministic summaries. One can then write the corresponding distortion function and rate-distortion frontier as
\begin{equation}
\label{eq:det_frontier}
D^{\mathrm{det}}(R):=\inf_{g\in\mathcal G:\,H(g(Y))\le R}\Delta_\KL(g(Y)),
\qquad
\mathcal A_{\mathrm{det}}:=\{(H(g(Y)),\Delta_\KL(g(Y))):g\in\mathcal G\}.
\end{equation}
Further, draw a finite selector $J$ independently of the experiment with probabilities
$\lambda_j$ and output $S=(J,g_J(Y))$, revealing which quantizer was used. The revealed mixture then satisfies 
\begin{equation*}
 I(Y;S)=\sum_j\lambda_j I(Y;g_j(Y)) = \sum_j\lambda_j H(g_j(Y)),\qquad
 \Delta_{\KL}(S)=\sum_j\lambda_j\Delta_{\KL}(g_j(Y)).
\end{equation*}
This gives the deterministic time-sharing frontier
\begin{equation}
\label{eq:ts_frontier}
D^{\mathrm{ts}}(R):=\inf\{d:(r,d)\in\operatorname{conv}(\mathcal A_{\mathrm{det}}),\ r\le R\}.
\end{equation}
Here $\operatorname{conv}$ denotes finite convex combinations \citep{wagner2021sawbridge, witsenhausen1975conditional}. Further, define the zero-distortion rate as $R_0:=\inf\{R\ge0:D^{\mathrm{ts}}(R)=0\}$
with $\inf\varnothing=\infty$. Appendix~\ref{app:separation-foundations} establishes the time-sharing identities and frontier properties.  
\begin{remark}
    Time-sharing also introduces randomness via the selector, but unlike the noise added to sample mean in the Gaussian-location example, the selector is revealed to the posterior decoder. Our main separation result concerns the advantage of truly stochastic summaries with hidden randomness, going beyond revealed time-sharing mixtures of deterministic summaries.
\end{remark}
\subsection{The main separation result}

We now show that randomizing the summary is not an arbitrary choice made specially for the Gaussian model: for general Bayesian experiments, any summary that approaches the optimal frontier must be truly stochastic (with randomness hidden from the decoder) under the only assumption that sampling laws share a common support. 
\begin{assumption}[Common support]
\label{ass:common_support}
The sampling distributions $\{P_{Y\mid\theta}:\theta\in\Theta_0\}$ are mutually absolutely continuous on a measurable set $\Theta_0$ with $\pi(\Theta_0)=1$. 
\end{assumption}
Under the standing assumption that the Bayes-information $I(\theta; Y)$ is finite, Assumption \ref{ass:common_support} is equivalent to $\pi\otimes P_Y\sim P_{\theta Y}$, or 
 $\Pi_y\sim\pi$ for $P_Y$-almost every $y$,
where $\sim$ denotes mutual absolute continuity. Thus, intuitively, no observation rules out a prior plausible value of the parameter. This is the standard homogeneity condition of \citet[Section~3]{lecam_hajek_undated}; see Proposition~\ref{prop:support_equivalence}.
\begin{theorem}[Strict stochastic advantage]
\label{thm:separation}
Let $\theta\to Y$ be a Bayesian experiment on standard Borel spaces with $0<I(\theta;Y)<\infty$ satisfying Assumption~\ref{ass:common_support}. Then
\begin{equation}
\label{eq:separation}
D^\star(R)<D^{\mathrm{ts}}(R)\qquad\text{for every }0<R<R_0.
\end{equation}
For each such $R$, a finite-alphabet stochastic summary has information rate exactly $R$ and distortion strictly below $D^{\mathrm{ts}}(R)$. Both frontiers equal $I(\theta;Y)$ at $R=0$, and both vanish for $R\ge R_0$ if $R_0<\infty$.
\end{theorem}

The result establishes a striking limitation of the entire
deterministic summary class: at each interior information rate, any sequence of deterministic or time-sharing summaries cannot approach the optimal frontier; the optimal summary must therefore be truly stochastic. This conclusion applies to general Bayesian
experiments, including continuous observations, and requires no attainment assumption for either frontier. Moreover, the interval of separation need not
be bounded: whenever $R_0=\infty$, stochastic summaries
retain a strict advantage at every finite information budget.
The Gaussian experiment of Proposition~\ref{thm:gaussian_exact} provides an
explicit example: its noisy Gaussian summary outperforms the
deterministic time-sharing envelope on the entire positive half-line.

\begin{proof}[Proof sketch.]
    Consider a deterministic quantizer $T: \mathcal Y \to \mathbb N$ at any fixed interior budget. The proof rests on the argument that a careful random relabeling procedure applied to $T$ introduces uniformly lower posterior distortion than the amount of data information lost. Our construction id as follows: retain the original label of $T$ with probability $1-t$, otherwise replace it by an independent draw from its own marginal distribution, without revealing the replacement. The resulting  summary is stochastic with the same output distribution as $T$, but saves at least $\kappa(T)t\log(1/t)$ in information rate, where $\kappa(T):=1-\sum_c \mathbb P(T=c)^2$ is the collision gap. However, this random corruption also incurs additional posterior distortion, which can be exactly written as a weighted Jensen--Shannon divergence between
the quantizer-level joint and product laws. Under Assumption~\ref{ass:common_support}, data processing gives an $o(t\log(1/t))$ bound for this
distortion, uniformly over all $T$. 

To turn this mechanism into a strict gap between frontiers, we also control the rate saving uniformly.
This becomes possible because, at any interior rate, nearly optimal quantizers $T$ or their revealed mixtures must retain a positive amount of parameter information, but under the standing assumption that the Bayes
information is finite, any quantizer with $I(\theta;T)>0$
can not arbitrarily concentrate on one label. This gives a uniformly
positive lower bound on the mixture's average collision gap. Hence, relabeling \textit{uniformly}  saves a positive multiple of
$t\log(1/t)$ while only incuring $o(t\log(1/t))$ posterior distortion. One can then spend this saving by  revealed time-sharing
with a fixed finite quantizer of lower distortion
and restore the original rate while strictly reducing distortion.
Thus hidden mixing frees the budget and revealed mixing spends it on a better summary.
Appendix~\ref{app:general-separation} provides the full proof. We fix all comparison constants
before choosing the nearly optimal mixture, avoiding any assumption
of an attained deterministic optimum or a finite-rate sufficient quantizer.
\end{proof}

\subsection{Remarks about Theorem \ref{thm:separation}} \label{sec:remarks_separation}
\paragraph{Finite observation spaces.} This setting is a special case of Theorem \ref{thm:separation} in which both the
separation interval and the required output alphabet can be made
explicit. If $Y$ takes $m$ values with positive probabilities and its
posteriors are pairwise equivalent, then $R_0=H(\Pi_Y)$, and separation
holds throughout $0<R<H(\Pi_Y)$. Moreover, an optimum is attained with
at most $m+1$ output values, including when the parameter space is
general; this is the finite-observation IB cardinality bound
\citep{witsenhausen1975conditional,benger2023cardinality}.
In particular, Beta--Binomial experiments (Section~\ref{sec:experiments}) with positive prior
parameters satisfy these conditions, with $H(\Pi_C)=H(C)$.
They provide a controlled illustration of the theorem because the
deterministic count envelope can be enumerated independently of
neural optimization. 
Appendix~\ref{app:finite-separation} gives the complete statements,
the attainment argument, and an explicit stochastic count-channel
witness.

\paragraph{The common-support condition.} Assumption~\ref{ass:common_support} identifies a broad regime of strict
advantage, but is not a characterization of all such experiments.
We provide three simple examples that distinguish our result from weaker or stronger claims.
First, if the target itself is determined by the observation,
$\theta=f(Y)$, and has finite entropy, then the time sharing envelope is indeed optimal with 
\begin{equation}
\label{eq:det_target}
    D^{\mathrm{ts}}(R)=D^\star(R)=[H(\theta)-R]_+.
\end{equation}
Revealed time-sharing between the constant quantizer and $f(Y)$ attains
this line \citep{kolchinsky2019caveats}, making randomization beyond
time-sharing unnecessary. Second, posterior overlap alone neither guarantees nor rules out the stochastic advantage. For example, let $Y$ be a fair bit, and 
$\theta$ equal $Y$ with probability $\alpha\in(0,1)$ and an erasure
symbol $e$ otherwise, using an independent erasure coin. Although the posteriors overlap at $e$, we have
$I(\theta;S)=\alpha I(Y;S)$ for every summary $S$. Hence, both frontiers equal
$\alpha[\log 2-R]_+$. Conversely,
a fair-bit observation with posteriors $(0.1,0.9,0)$ and
$(0,0.1,0.9)$ has strict separation at every $0<R<\log 2$, even
though the sampling laws are not equivalent. Together, these
examples show that posterior overlap alone is insufficient and
common support is not necessary. See Appendix~\ref{app:support-examples} for
calculations and further boundary cases.

\paragraph{Quantifing the gap.} First, we clarify that the advantage is not a consequence of a particular frontier shape. The scalar Gaussian frontiers have the same initial slope (Proposition~\ref{prop:gaussian_same_slope}), while full-support
finite experiments can have an initial linear segment
\citep{benger2023cardinality}. Neither feature prevents strict
separation at every interior rate. What drives the proof is instead
the difference in scale between the information saved by hidden
relabeling and the additional posterior distortion incurred. Under the additional assumption
that the \textit{lautum information}
$L:=\KL(\pi\otimes P_Y\Vert P_{\theta Y})$ \citep{palomar2008lautum} is finite, one obtains a quantitative gap of $O(t \log (1/t))$; see Corollary~\ref{cor:lautum_gap}.

\paragraph{When does the stochastic advantage not hold?} 
The choice of information rate as the budget is essential for our result to hold. Relabeling
preserves the output distribution, so it reduces mutual information
without reducing output entropy. Together with
Theorem~\ref{thm:avg_quant}, the separation result establishes a
genuine contrast: while deterministic summaries suffice for the optimal
fixed-cardinality infimum, they are strictly suboptimal
under an (interior) information budget satisfying the
hypotheses of Theorem \ref{thm:separation}. Furthermore, the entropy-penalized deterministic information
bottleneck of \citet{strouse2017deterministic} studies a different
objective; neither its hard-assignment conclusions nor our
cardinality result imply deterministic optimality for every
output-entropy-constrained problem.

Finally, the result concerns a summary applied to one dataset at a
time. It does not compare against deterministic block codes that jointly
encode multiple independent Bayesian experiments. Such block codes can attain the
optimal tradeoff asymptotically in classical
finite-alphabet logarithmic-loss settings
\citep{courtade2014multiterminal}. The result,
however, does go beyond removing irrelevant variation within posterior
fibers by sufficient-statistic averaging
\citep{armstrong2026source}: strict improvement remains possible
even when there is no such variation to remove. The single-dataset setting is
precisely the one implemented by the NPE protocol below, where the
theorem motivates testing stochastic summaries at matched
information rates.
\section{Binary-bottleneck NPE: comparing summaries at measured rate}
\label{sec:npe_protocol}

The preceding sections distinguish nominal summary dimension, code cardinality, and information rate as different rate axes. In this section, we investigate the empirical gains of summary randomization, and motivate comparing  summary architectures at their measured rate. We examine two hypotheses:
\begin{description}[leftmargin=2.8em]
\item[(H1)] Stochastic bottlenecks can achieve lower evaluation posterior loss than deterministic bottlenecks at matched estimated rate. Where an exhaustive deterministic comparison is available, we also test improvement beyond deterministic time-sharing.
\item[(H2)] Comparisons between summary architectures depend on the information budget; one no-bottleneck score need not predict their ranking under compression.
\end{description}

To make information rates explicitly measurable, we insert a binary-bottleneck between the encoder and decoder networks of the standard NPE architecture. A summary encoder maps a dataset to $s=e_\psi(y)\in\mathbb R^{d_s}$. A binary encoder produces code probabilities $p_\eta(b_k=1\mid s)=\sigma(a_{\eta, k}(s))$ for $k \in [K]$, and a binary decoder maps a code $b\in\{0,1\}^K$ back to $\widetilde s=d_\eta(b)\in\mathbb R^{d_s}$. The posterior network receives only $\widetilde s$, with no skip connection from the original summary. The stochastic channel uses the conditionally independent Bernoulli channel and samples bits $b$ from $q_\eta(b\mid y)=\prod_{k=1}^{K} \operatorname{Bern}\!\left(b_k;p_{\eta,k}(e_\psi(y))\right)$. The deterministic channel uses a hard threshold of the same probabilities $b_k=\mathbf1\{p_{\eta,k}(e_\psi(y))>1/2\}$. 

\paragraph{Training and evaluation.}
The summary network, binary autoencoder, and posterior decoder are trained jointly with a sparse Bernoulli penalty which regularizes the information rate of the channel:
\begin{equation}
\begin{aligned}
 \mathcal L(\phi,\psi,\eta) =\mathbb E[-\log q_\phi(\theta\mid d_\eta(B))]+\beta\,\mathbb E_Y\sum_{k=1}^{K}
 \operatorname{KL}\!\left(q_\eta(b_k\mid y)
 \,\Vert\,\operatorname{Bern}(p_0)\right),
\end{aligned}
 \label{eq:training-objective}
\end{equation}
where $p_0=0.01$. Let $r_0(b):=\prod_k\operatorname{Bern}(b_k;p_0)$, then the variational identity \citep{alemi2016deep} $\mathbb E_Y\operatorname{KL}(q_\eta(B\mid Y)\Vert r_0) = I(Y;B)+\operatorname{KL}(q_\eta(B)\Vert r_0)$ shows that the penalty includes aggregate-marginal mismatch. We train both channels using forward codes and use straight-through gradients during the backward pass \citep{bengio2013estimating}. The deterministic channel's forward pass is annealed with a decreasing temperature: we start with the same stochastic training and slowly decrease the temperature of the channel followed by a hard threshold training phase to ensure convergence \citep{jang2017categorical, agustsson2017soft}. 

We vary $\beta$ to obtain operating points and measure the deployed rate after training. Let $q_i(b)=q_\eta(b\mid y_i)$ and $\bar q_M=M^{-1}\sum_iq_i$ for a reference set of $M$ simulator observations. We estimate $$R_M=\mathbb E_{B\sim\bar q_M}[-\log\bar q_M(B)]-M^{-1}\sum_iH(q_i)$$ by Monte Carlo. The exact empirical-mixture information obeys $R_M\leq\log M$ and has downward bias relative to population information after averaging over random reference sets. Hence, we also fit marginal  models $r_\omega$ on the marginal code distribution of the trained bottleneck NPE, giving a variational upper bound on the estimated rate. The two estimates are displayed as horizontal segments for experiments wherever exact enumeration is not feasible.

We replace the vertical axis in our experiments with evaluation NLL rather than a direct estimate of posterior KL. Write $q_\phi(\theta\mid B)$ for the decoder including its deterministic binary reconstruction. When the entropy and expected losses are finite, Lemma~\ref{lem:identity} gives
 $\mathrm{NLL}=H(\theta\mid Y)+\Delta_{\KL}(B)
 +\mathbb E_B\operatorname{KL}\!\left(p(\theta\mid B)\Vert q_\phi(\theta\mid B)\right).$
Since, within one simulator and parameterization, the first conditional entropy term is common, NLL differences combine information loss and decoder error. Although the resulting curves therefore describe the trained family, we verify on the analytic benchmarks that our trained decoders have multiple orders of magnitude less error than the measured separation in NLL.


\section{Experiments}
\label{sec:experiments}

We use a finite conjugate experiment and four simulator benchmarks to examine the hypotheses of Section~\ref{sec:npe_protocol}. The count experiment isolates information loss from decoder approximation and compares achievable stochastic channels with an exhaustive deterministic time-sharing hull. On the simulators, we compare evaluation NLL at estimated information rate within trained model families. The interpolation and comparator conventions are specified in Section~\ref{sec:npe_protocol}.

\begin{wrapfigure}[15]{r}{0.4\linewidth}
  \vspace{-\intextsep}          
  \centering
  \includegraphics[width=\linewidth]{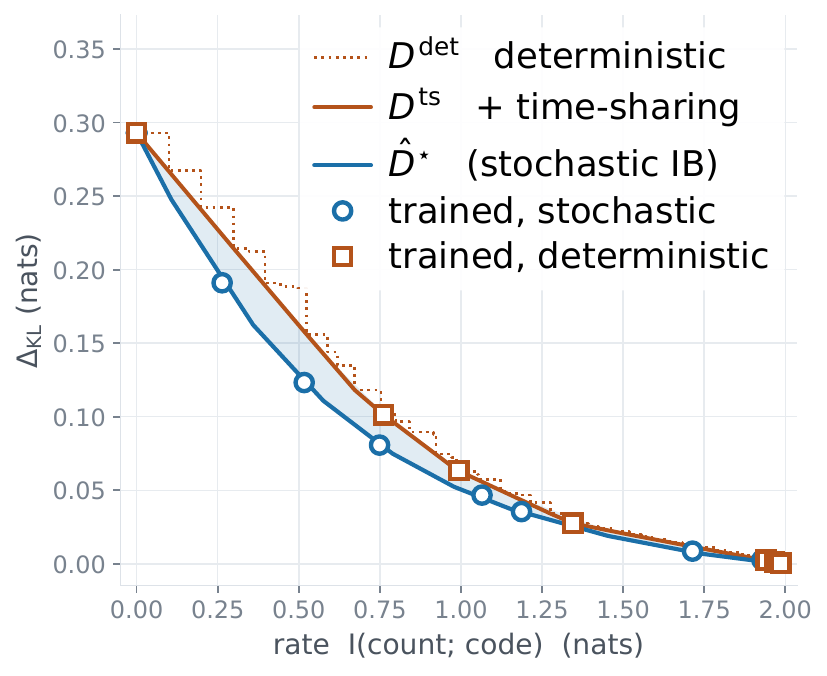}
  \caption{Randomization beyond deterministic time-sharing in the count
  benchmark.}
  \label{fig:beta-v6}
\end{wrapfigure}
\textbf{An enumerated posterior-compression example.}
Let $\theta\sim\operatorname{Beta}(5,5)$ and $Y_1,\ldots,Y_8\mid\theta$ be
independent Bernoulli observations. We encode their sufficient count
$C=\sum_iY_i$ using four bits. The nine count values have $21{,}147$
partitions, permitting enumeration of all deterministic count encoders up to
relabeling. Figure~\ref{fig:beta-v6} compares these pairs and their lower
convex hull with achievable stochastic channels from information-bottleneck
iterations. The largest observed reduction in information loss beyond this time-sharing hull is approximately $0.036$ nats.

All BNPE models use an exact finite-sum mutual-information penalty in this example and training lands at the two plotted frontiers. The maximum mean decoder approximation gap is approximately $4\times10^{-4}$ nats,
substantially smaller than the separation. 

\begin{figure}[t]
 \centering
 \includegraphics[width=\linewidth]{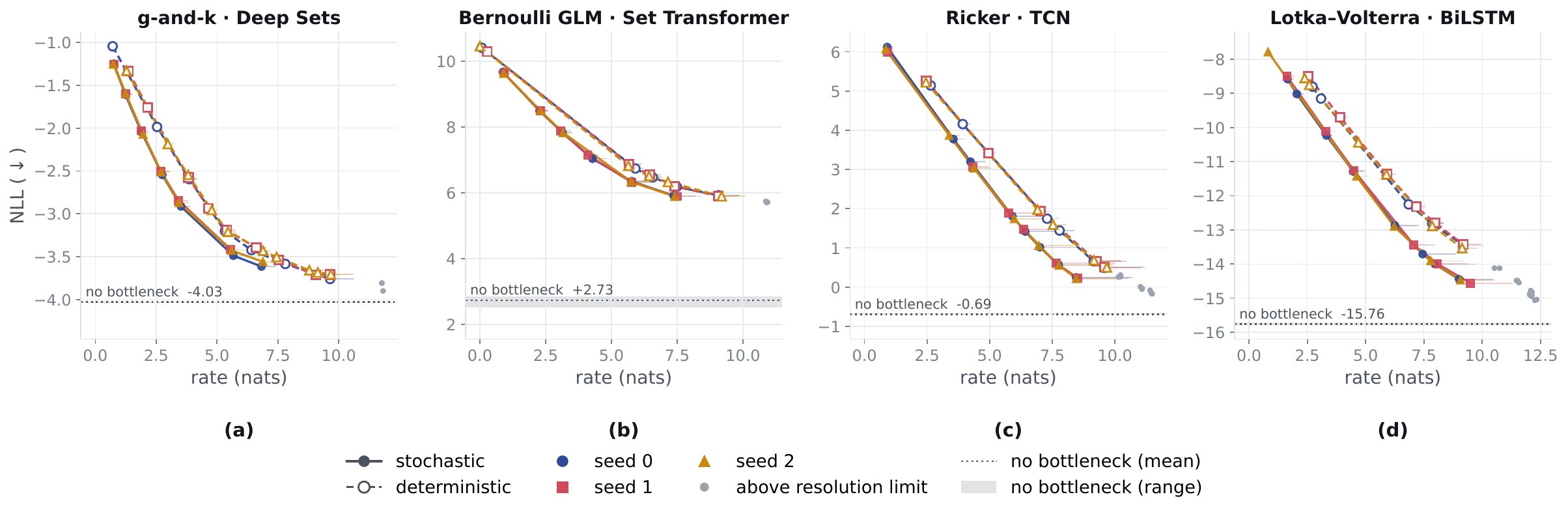}
 \caption{\textbf{Rate--NLL comparisons on four simulators.} All trained stochastic NPEs achieve lower NLL than the plotted lower convex hull of deterministic ones at matched information rate. Horizontal segments connect the MC and marginal-based rate estimates. Gray points exceed the 10-nat MC rate threshold and are omitted from averaged curves. Dotted lines are standard NPE references trained without a bottleneck.}
 \label{fig:arms-v6}
\end{figure}

\begin{figure}[t]
 \centering
 \includegraphics[width=\linewidth]{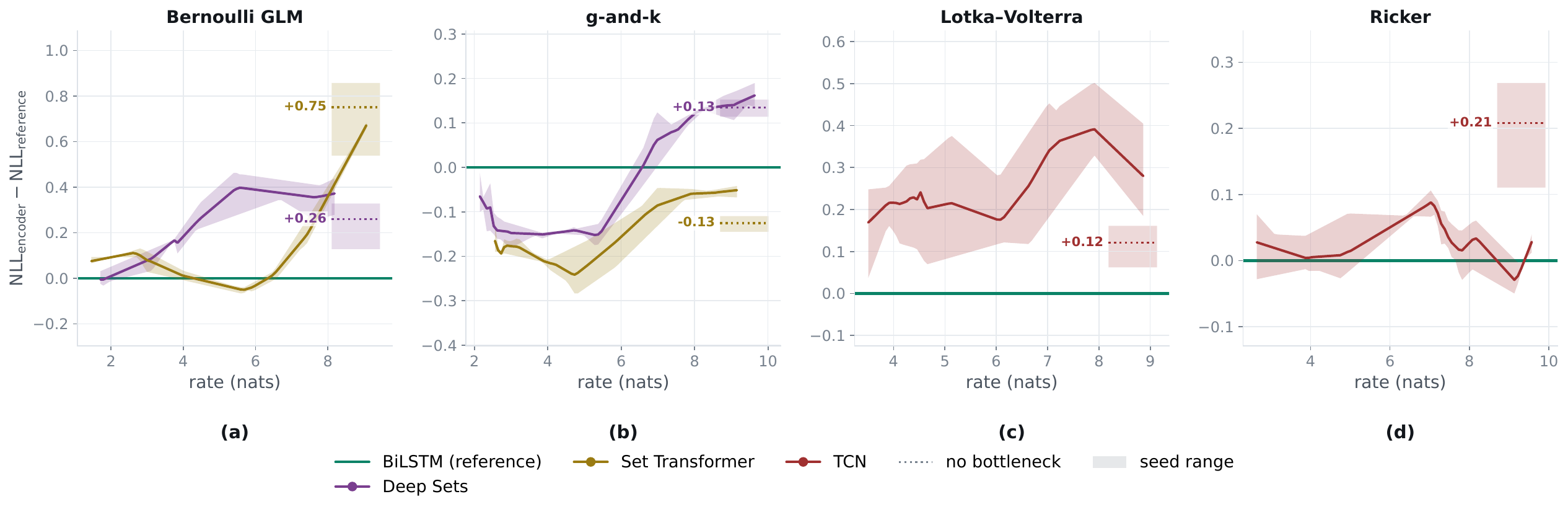}
 \caption{\textbf{Architecture differences at matched rate.} Each curve is encoder NLL minus BiLSTM NLL; negative values favor the named encoder. Shading shows the minimum-to-maximum range of the three paired differences. Dotted segments are corresponding mean no-bottleneck differences.}
 \vspace{-4mm}
 \label{fig:architectures-v6}
\end{figure}

\paragraph{Simulator protocol and estimands.}
We study a Bernoulli GLM, g-and-k, Lotka--Volterra, and Ricker, with Deep Sets, Set Transformer, BiLSTM, or TCN encoders as appropriate. Each supplies a 128-bit bottleneck followed by the same conditional-flow family. Appendix~\ref{app:experiments-v6} reports the configuration, estimators, and additional panels.

\textbf{(H1) Stochastic versus deterministic summaries.}
We show the comparison for one encoder architecture per simulator in Figure~\ref{fig:arms-v6} by plotting the time-sharing convex hull of deterministic summaries against the stochastic summary curve, separately for each training seed. All evaluated stochastic BNPE models exhibit better rate--NLL tradeoffs: the mean matched-rate NLL differences remain negative in all simulator-encoder combinations, ranging from $-0.168$ to $-0.617$ nats. The other six simulator-encoder combinations, deferred to Figure~\ref{fig:arms-rest-v6}, also show the same advantage. 


\textbf{(H2) Architecture comparisons depend on rate.}
Figure~\ref{fig:architectures-v6} shows rate-matched architecture comparisons for all four simulators. Within each training seed, we compute  differences from the BiLSTM reference, and then average them on a common rate grid. On g-and-k, the Deep Sets comparison changes sign, while Set Transformer remains favorable over the displayed shared range. On Bernoulli GLM, the Set Transformer gap is much smaller at intermediate rates than its no-bottleneck difference and the relative performance of DeepSets and Set Transformer changes with the rate. For Lotka--Volterra, TCN remains above BiLSTM throughout and for Ricker, the difference is smaller and varies with rate and seed. These comparisons demonstrate why one no-bottleneck score does not characterize an encoder under compression.

\section{Limitations and Conclusion}
\label{sec:conclusion}

\paragraph{Limitations.} Our experiments compare summaries at rates below roughly 12 nats, where the frontier is resolvable with our estimators. Whether the separation of Theorem \ref{thm:separation} has empirical consequences at the rates of unconstrained summary networks is a question we leave for future work.
Furthermore, the theorem concerns a summary channel applied to one observation or simulated dataset. It does not compare against deterministic block codes that summarize independent experiments or guarantee that a neural optimizer necessarily reaches the better channel. Furthermore, as we show via examples, the common support assumption is also only sufficient rather than necessary: characterizing the separation when sampling supports differ remains an open problem as well. 

\paragraph{Conclusion.} In this work, we show that posterior fidelity and retained information provide complementary coordinates for comparing summaries. A fixed representation dimension does not fix its information rate, and regular parametric concentration can require increasing rate even when a fixed-dimensional sufficient statistic exists. Posterior quantization and mass-sensitive converses connect this requirement to the law of the random posterior. At a prescribed cardinality, deterministic summaries achieve the same distortion infimum as stochastic ones. However, to the best of our knowledge, Theorem \ref{thm:separation} is the first formal result that establishes a strict advantage of stochastic summaries, over every deterministic and time-sharing mixture, when comparisons happen at matched information rate. 
Our binary-bottleneck NPE framework further turns this theoretical distinction into a measured comparison protocol. The count benchmark enumerates the deterministic envelope directly and shows strict gains. The scientific simulator comparisons also show favorable stochastic rate--NLL tradeoffs and architecture rankings that vary with rate. Together, the results establish an explicit information budget as a useful basis for studying both the statistical requirements and the learned performance of summaries in SBI.

\section*{AI use statement}
Generative AI tools assisted with literature searches,  code review, figure preparation, and manuscript editing. The authors are responsible for the scientific claims, source attribution, and final presentation.

\section*{Acknowledgments}
We thank Dr. Adam Klivans and the Institute for Foundations of Machine Learning (IFML) at UT Austin for providing the compute resources.

\clearpage
\bibliography{references}

@article{bengio2013estimating,
  title={Estimating or propagating gradients through stochastic neurons for conditional computation},
  author={Bengio, Yoshua and L{\'e}onard, Nicholas and Courville, Aaron},
  journal={arXiv preprint arXiv:1308.3432},
  year={2013}
}

@article{chen2024conditional,
  title={Conditional diffusions for amortized neural posterior estimation},
  author={Chen, Tianyu and Bansal, Vansh and Scott, James G},
  journal={arXiv preprint arXiv:2410.19105},
  year={2024}
}

@article{cranmer2020frontier,
  title={The frontier of simulation-based inference},
  author={Cranmer, Kyle and Brehmer, Johann and Louppe, Gilles},
  journal={Proceedings of the National Academy of Sciences},
  volume={117},
  number={48},
  pages={30055--30062},
  year={2020},
  publisher={National Academy of Sciences}
}

@article{papamakarios2016fast,
  title={Fast $\varepsilon$-free inference of simulation models with bayesian conditional density estimation},
  author={Papamakarios, George and Murray, Iain},
  journal={Advances in neural information processing systems},
  volume={29},
  year={2016}
}

@article{radev2020bayesflow,
  title={BayesFlow: Learning complex stochastic models with invertible neural networks},
  author={Radev, Stefan T and Mertens, Ulf K and Voss, Andreas and Ardizzone, Lynton and K{\"o}the, Ullrich},
  journal={IEEE transactions on neural networks and learning systems},
  volume={33},
  number={4},
  pages={1452--1466},
  year={2020},
  publisher={IEEE}
}

@article{chan2018likelihood,
  title={A likelihood-free inference framework for population genetic data using exchangeable neural networks},
  author={Chan, Jeffrey and Perrone, Valerio and Spence, Jeffrey and Jenkins, Paul and Mathieson, Sara and Song, Yun},
  journal={Advances in neural information processing systems},
  volume={31},
  year={2018}
}

@article{aakesson2021convolutional,
  title={Convolutional neural networks as summary statistics for approximate Bayesian computation},
  author={{\AA}kesson, Mattias and Singh, Prashant and Wrede, Fredrik and Hellander, Andreas},
  journal={IEEE/ACM Transactions on Computational Biology and Bioinformatics},
  volume={19},
  number={6},
  pages={3353--3365},
  year={2021},
  publisher={IEEE}
}

@article{blum2013comparative,
  title={A comparative review of dimension reduction methods in approximate Bayesian computation},
  author={Blum, Michael GB and Nunes, Maria Antonieta and Prangle, Dennis and Sisson, Scott A},
  journal={Statistical Science},
  pages={189--208},
  year={2013},
  publisher={JSTOR}
}

@article{nunes2010optimal,
  title={On optimal selection of summary statistics for approximate Bayesian computation.},
  author={Nunes, Matthew A and Balding, David J},
  journal={Statistical Applications in Genetics \& Molecular Biology},
  volume={9},
  number={1},
  year={2010}
}

@article{fearnhead2012constructing,
  title={Constructing summary statistics for approximate Bayesian computation: semi-automatic approximate Bayesian computation},
  author={Fearnhead, Paul and Prangle, Dennis},
  journal={Journal of the Royal Statistical Society Series B: Statistical Methodology},
  volume={74},
  number={3},
  pages={419--474},
  year={2012},
  publisher={Oxford University Press}
}

@article{charnock2018automatic,
  title={Automatic physical inference with information maximizing neural networks},
  author={Charnock, Tom and Lavaux, Guilhem and Wandelt, Benjamin D},
  journal={Physical Review D},
  volume={97},
  number={8},
  pages={083004},
  year={2018},
  publisher={APS}
}

@article{chen2020neural,
  title={Neural approximate sufficient statistics for implicit models},
  author={Chen, Yanzhi and Zhang, Dinghuai and Gutmann, Michael and Courville, Aaron and Zhu, Zhanxing},
  journal={arXiv preprint arXiv:2010.10079},
  year={2020}
}

@inproceedings{chen2023learning,
  title={Is learning summary statistics necessary for likelihood-free inference?},
  author={Chen, Yanzhi and Gutmann, Michael U and Weller, Adrian},
  booktitle={International Conference on Machine Learning},
  pages={4529--4544},
  year={2023},
  organization={PMLR}
}

@article{jiang2017learning,
  title={Learning summary statistic for approximate Bayesian computation via deep neural network},
  author={Jiang, Bai and Wu, Tung-yu and Zheng, Charles and Wong, Wing H},
  journal={Statistica Sinica},
  pages={1595--1618},
  year={2017},
  publisher={JSTOR}
}

@article{hoffmann2026unifying,
  title={Unifying Summary Statistic Selection for Approximate Bayesian Computation},
  author={Hoffmann, Till and Onnela, Jukka-Pekka},
  journal={Statistics and Computing},
  volume={36},
  number={2},
  pages={70},
  year={2026},
  publisher={Springer}
}

@article{lindley1956measure,
  title={On a measure of the information provided by an experiment},
  author={Lindley, Dennis V},
  journal={The Annals of Mathematical Statistics},
  volume={27},
  number={4},
  pages={986--1005},
  year={1956},
  publisher={Institute of Mathematical Statistics}
}

@article{tishby2000information,
  title={The information bottleneck method},
  author={Tishby, Naftali and Pereira, Fernando C and Bialek, William},
  journal={arXiv preprint physics/0004057},
  year={2000}
}

@article{alemi2016deep,
  title={Deep variational information bottleneck},
  author={Alemi, Alexander A and Fischer, Ian and Dillon, Joshua V and Murphy, Kevin},
  journal={arXiv preprint arXiv:1612.00410},
  year={2016}
}

@article{strouse2017deterministic,
  title={The deterministic information bottleneck},
  author={Strouse, Daniel J and Schwab, David J},
  journal={Neural computation},
  volume={29},
  number={6},
  pages={1611--1630},
  year={2017},
  publisher={MIT Press}
}

@inproceedings{harremoes2007revisited,
  title={The information bottleneck revisited or how to choose a good distortion measure},
  author={Harremo{\"e}s, Peter and Tishby, Naftali},
  booktitle={IEEE International Symposium on Information Theory (ISIT)},
  pages={566--570},
  year={2007}
}

@misc{armstrong2026source,
  title={Source-side sufficiency for the information bottleneck: exact reduction and finite-block equivalence},
  author={Armstrong, J.},
  note={arXiv:2604.26744},
  year={2026}
}

@article{banerjee2005bregman,
  title={Clustering with Bregman divergences},
  author={Banerjee, Arindam and Merugu, Srujana and Dhillon, Inderjit S. and Ghosh, Joydeep},
  journal={Journal of Machine Learning Research},
  volume={6},
  pages={1705--1749},
  year={2005}
}

@article{chou1989entropy,
  title={Entropy-constrained vector quantization},
  author={Chou, Philip A. and Lookabaugh, Tom and Gray, Robert M.},
  journal={IEEE Transactions on Acoustics, Speech, and Signal Processing},
  volume={37},
  number={1},
  pages={31--42},
  year={1989}
}

@misc{xu2016information,
  title={Information-theoretic lower bounds on Bayes risk in decentralized estimation},
  author={Xu, Aolin and Raginsky, Maxim},
  note={arXiv:1607.00550},
  year={2016}
}

@misc{riegler2018rate,
  title={Rate-distortion theory for general sets and measures},
  author={Riegler, Erwin and Koliander, G{\"u}nther and B{\"o}lcskei, Helmut},
  note={arXiv:1804.08980},
  year={2018}
}

@article{haussler1997mutual,
  title={Mutual information, metric entropy and cumulative relative entropy risk},
  author={Haussler, David and Opper, Manfred},
  journal={The Annals of Statistics},
  volume={25},
  number={6},
  pages={2451--2492},
  year={1997}
}

@article{clarkeBarron1994,
  title={Jeffreys' prior is asymptotically least favorable under entropy risk},
  author={Clarke, Bertrand S. and Barron, Andrew R.},
  journal={Journal of Statistical Planning and Inference},
  volume={41},
  number={1},
  pages={37--60},
  year={1994}
}

@misc{hayashi2018minimum,
  title={Minimum rates of approximate sufficient statistics},
  author={Hayashi, Masahito and Tan, Vincent Y. F.},
  note={arXiv:1612.02542; IEEE Transactions on Information Theory},
  year={2018}
}

@misc{elnmeir2020greedy,
  title={New approach to greedy vector quantization},
  author={El Nmeir, Rancy and Luschgy, Harald and Pag{\`e}s, Gilles},
  note={arXiv:2003.14145},
  year={2020}
}

@misc{saumard2014logconcavity,
  title={Log-concavity and strong log-concavity: a review},
  author={Saumard, Adrien and Wellner, Jon A.},
  note={arXiv:1404.5886},
  year={2014}
}

@article{chechik2005gaussianib,
  title={Information bottleneck for Gaussian variables},
  author={Chechik, Gal and Globerson, Amir and Tishby, Naftali and Weiss, Yair},
  journal={Journal of Machine Learning Research},
  volume={6},
  pages={165--188},
  year={2005}
}

@misc{asoodeh2020bottleneck,
  title={Bottleneck problems: information and estimation-theoretic view},
  author={Asoodeh, Shahab and Calmon, Flavio P.},
  note={arXiv:2011.06208},
  year={2020}
}

@inproceedings{kolchinsky2019caveats,
  title={Caveats for information bottleneck in deterministic scenarios},
  author={Kolchinsky, Artemy and Tracey, Brendan D. and Van Kuyk, Steven},
  booktitle={International Conference on Learning Representations (ICLR)},
  year={2019},
  note={arXiv:1808.07593}
}

@misc{benger2023cardinality,
  title={The cardinality bound on the information bottleneck representations is tight},
  author={Benger, Eyal and Asoodeh, Shahab and Chen, Jun},
  note={arXiv:2305.07000},
  year={2023}
}

@misc{sui2023evaluating,
  title={Evaluating summary statistics with mutual information for cosmological inference},
  author={Sui, Ce and Zhao, Xiaosheng and Jing, Tao and Mao, Yi},
  note={arXiv:2307.04994},
  year={2023}
}

@article{shao2022learning,
  title={Learning task-oriented communication for edge inference: an information bottleneck approach},
  author={Shao, Jiawei and Mao, Yuyi and Zhang, Jun},
  journal={IEEE Journal on Selected Areas in Communications},
  volume={40},
  number={1},
  pages={197--211},
  year={2022}
}

@article{joyce2008approximately,
  author  = {Joyce, Paul and Marjoram, Paul},
  title   = {Approximately Sufficient Statistics and {B}ayesian Computation},
  journal = {Statistical Applications in Genetics and Molecular Biology},
  volume  = {7},
  number  = {1},
  year    = {2008}
}

@article{alsing2018generalized,
  author = {Alsing, Justin and Wandelt, Benjamin},
  title = {Generalized Massive Optimal Data Compression},
  journal = {Monthly Notices of the Royal Astronomical Society: Letters},
  volume = {476},
  number = {1},
  pages = {L60--L64},
  year = {2018},
  doi = {10.1093/mnrasl/sly029},
  url = {https://academic.oup.com/mnrasl/article/476/1/L60/4909822}
}

@article{courtade2014multiterminal,
  author = {Courtade, Thomas A. and Weissman, Tsachy},
  title = {Multiterminal Source Coding Under Logarithmic Loss},
  journal = {IEEE Transactions on Information Theory},
  volume = {60},
  number = {1},
  pages = {740--761},
  year = {2014},
  doi = {10.1109/TIT.2013.2288257},
  url = {https://people.eecs.berkeley.edu/~courtade/pdfs/CourtadeWeissman_TransITJan2014.pdf}
}

@inproceedings{giladbachrach2003information,
  author = {Gilad-Bachrach, Ran and Navot, Amir and Tishby, Naftali},
  title = {An Information Theoretic Tradeoff between Complexity and Accuracy},
  booktitle = {Learning Theory and Kernel Machines},
  series = {Lecture Notes in Computer Science},
  volume = {2777},
  pages = {595--609},
  publisher = {Springer},
  year = {2003},
  doi = {10.1007/978-3-540-45167-9_43},
  url = {https://doi.org/10.1007/978-3-540-45167-9_43}
}

@inproceedings{hermans2020ratio,
  title = {Likelihood-free {MCMC} with Amortized Approximate Ratio Estimators},
  author = {Hermans, Joeri and Begy, Volodimir and Louppe, Gilles},
  booktitle = {Proceedings of the 37th International Conference on Machine Learning},
  volume = {119},
  pages = {4239--4248},
  year = {2020},
  publisher = {PMLR},
  url = {https://proceedings.mlr.press/v119/hermans20a.html}
}

@InProceedings{lueckmann19a,
  title = 	 {Likelihood-free inference with emulator networks},
  author =       {Lueckmann, Jan-Matthis and Bassetto, Giacomo and Karaletsos, Theofanis and Macke, Jakob H.},
  booktitle = 	 {Proceedings of The 1st Symposium on Advances in Approximate Bayesian Inference},
  pages = 	 {32--53},
  year = 	 {2019},
  editor = 	 {Ruiz, Francisco and Zhang, Cheng and Liang, Dawen and Bui, Thang},
  volume = 	 {96},
  series = 	 {Proceedings of Machine Learning Research},
  month = 	 {02 Dec},
  publisher =    {PMLR},
  url = 	 {https://proceedings.mlr.press/v96/lueckmann19a.html},
}

@inproceedings{lueckmann2021benchmarking,
  title={Benchmarking simulation-based inference},
  author={Lueckmann, Jan-Matthis and Boelts, Jan and Greenberg, David and Goncalves, Pedro and Macke, Jakob},
  booktitle={International conference on artificial intelligence and statistics},
  pages={343--351},
  year={2021},
  organization={PMLR}
}

@article{marin2012approximate,
  title={Approximate Bayesian computational methods},
  author={Marin, Jean-Michel and Pudlo, Pierre and Robert, Christian P and Ryder, Robin J},
  journal={Statistics and computing},
  volume={22},
  number={6},
  pages={1167--1180},
  year={2012},
  publisher={Springer}
}

@misc{lecam_hajek_undated,
  author = {Le Cam, Lucien},
  title = {On Some Results of {J. H\'{a}jek} Concerning Asymptotic Normality},
  year = {n.d.},
  howpublished = {Unpublished manuscript, University of California, Berkeley},
  note = {Undated manuscript in the Le Cam archive; Section 3, p. 14},
  url = {https://www.stat.berkeley.edu/~rice/LeCam/papers/Hajek2a.pdf}
}

@article{palomar2008lautum,
  author = {Palomar, Daniel P. and Verd{\'u}, Sergio},
  title = {Lautum information},
  journal = {IEEE Transactions on Information Theory},
  volume = {54},
  number = {3},
  pages = {964--975},
  year = {2008},
  doi = {10.1109/TIT.2007.915715},
  url = {https://palomar.home.ece.ust.hk/papers/2008/PalomarVerdu_TransIT2008_lautum.pdf}
}

@InProceedings{papamakarios19a,
  title = 	 {Sequential Neural Likelihood: Fast Likelihood-free Inference with Autoregressive Flows},
  author =       {Papamakarios, George and Sterratt, David and Murray, Iain},
  booktitle = 	 {Proceedings of the Twenty-Second International Conference on Artificial Intelligence and Statistics},
  pages = 	 {837--848},
  year = 	 {2019},
  editor = 	 {Chaudhuri, Kamalika and Sugiyama, Masashi},
  volume = 	 {89},
  series = 	 {Proceedings of Machine Learning Research},
  month = 	 {16--18 Apr},
  publisher =    {PMLR},
  url = 	 {https://proceedings.mlr.press/v89/papamakarios19a.html},
}

@inproceedings{slonim2000agglomerative,
  author = {Slonim, Noam and Tishby, Naftali},
  title = {Agglomerative Information Bottleneck},
  booktitle = {Advances in Neural Information Processing Systems},
  volume = {12},
  pages = {617--623},
  year = {2000},
  note = {Conference held in 1999},
  url = {https://proceedings.neurips.cc/paper_files/paper/1999/file/be3e9d3f7d70537357c67bb3f4086846-Paper.pdf}
}

@article{tavare1997inferring,
    author = {Tavaré, Simon and Balding, David J and Griffiths, R C and Donnelly, Peter},
    title = {Inferring Coalescence Times From DNA Sequence Data},
    journal = {Genetics},
    volume = {145},
    number = {2},
    pages = {505-518},
    year = {1997},
    month = {02},
    issn = {1943-2631},
    doi = {10.1093/genetics/145.2.505},
    url = {https://doi.org/10.1093/genetics/145.2.505},
    eprint = {https://academic.oup.com/genetics/article-pdf/145/2/505/35202057/genetics0505.pdf},
}

@inproceedings{theis2021advantages,
  author = {Theis, Lucas and Agustsson, Eirikur},
  title = {On the advantages of stochastic encoders},
  booktitle = {Neural Compression Workshop, International Conference on Learning Representations},
  year = {2021},
  url = {https://arxiv.org/abs/2102.09270}
}

@article{thomas2022likelihood,
  title={Likelihood-free inference by ratio estimation},
  author={Thomas, Owen and Dutta, Ritabrata and Corander, Jukka and Kaski, Samuel and Gutmann, Michael U},
  journal={Bayesian Analysis},
  volume={17},
  number={1},
  pages={1--31},
  year={2022},
  publisher={International Society for Bayesian Analysis}
}

@inproceedings{wagner2021sawbridge,
  author = {Wagner, Aaron B. and Ball{\'e}, Johannes},
  title = {Neural Networks Optimally Compress the {Sawbridge}},
  booktitle = {2021 Data Compression Conference (DCC)},
  pages = {143--152},
  year = {2021},
  publisher = {IEEE},
  doi = {10.1109/DCC50243.2021.00022},
  url = {https://arxiv.org/abs/2011.05065}
}

@article{witsenhausen1975conditional,
  author = {Witsenhausen, Hans S. and Wyner, Aaron D.},
  title = {A conditional entropy bound for a pair of discrete random variables},
  journal = {IEEE Transactions on Information Theory},
  volume = {21},
  number = {5},
  pages = {493--501},
  year = {1975},
  doi = {10.1109/TIT.1975.1055437},
  url = {https://doi.org/10.1109/TIT.1975.1055437}
}

@article{wood2010statistical,
  title={Statistical inference for noisy nonlinear ecological dynamic systems},
  author={Wood, Simon N},
  journal={Nature},
  volume={466},
  number={7310},
  pages={1102--1104},
  year={2010},
  publisher={Nature Publishing Group UK, London}
}

@article{wyner1973theorem,
  author = {Wyner, Aaron D. and Ziv, Jacob},
  title = {A theorem on the entropy of certain binary sequences and applications: {P}art {I}},
  journal = {IEEE Transactions on Information Theory},
  volume = {19},
  number = {6},
  pages = {769--772},
  year = {1973},
  doi = {10.1109/TIT.1973.1055107},
  url = {https://doi.org/10.1109/TIT.1973.1055107}
}

@inproceedings{germain2015made,
  title={Made: Masked autoencoder for distribution estimation},
  author={Germain, Mathieu and Gregor, Karol and Murray, Iain and Larochelle, Hugo},
  booktitle={International conference on machine learning},
  pages={881--889},
  year={2015},
  organization={Pmlr}
}

@article{zammit2025neural,
  title={Neural methods for amortized inference},
  author={Zammit-Mangion, Andrew and Sainsbury-Dale, Matthew and Huser, Rapha{\"e}l},
  journal={Annual Review of Statistics and Its Application},
  volume={12},
  number={1},
  pages={311--335},
  year={2025},
  publisher={Annual Reviews}
}

@inproceedings{agustsson2017soft,
  author    = {Agustsson, Eirikur and Mentzer, Fabian and Tschannen, Michael and
               Cavigelli, Lukas and Timofte, Radu and Benini, Luca and
               Van Gool, Luc},
  title     = {Soft-to-Hard Vector Quantization for End-to-End Learning
               Compressible Representations},
  booktitle = {Advances in Neural Information Processing Systems (NeurIPS)},
  year      = {2017}
}

@inproceedings{jang2017categorical,
  author    = {Jang, Eric and Gu, Shixiang and Poole, Ben},
  title     = {Categorical Reparameterization with {G}umbel-Softmax},
  booktitle = {International Conference on Learning Representations (ICLR)},
  year      = {2017}
}
\bibliographystyle{iclr2027_conference}
\clearpage
\appendix
\section*{Appendix guide}
The separation proof and its scope are developed in Appendices~\ref{app:separation-foundations}--\ref{app:support-examples}. The remaining appendices give the posterior-geometry and parametric arguments, experimental details, and consolidated related work.

\medskip
\noindent
\begin{tabularx}{\linewidth}{@{}lXr@{}}
\toprule
 & & Page \\
\midrule
\ref{app:separation-foundations} & \hyperref[app:separation-foundations]{Foundations for the separation theorem} & \pageref{app:separation-foundations} \\
\ref{app:general-separation} & \hyperref[app:general-separation]{Strict separation for general observation spaces} & \pageref{app:general-separation} \\
\ref{app:finite-separation} & \hyperref[app:finite-separation]{Finite observations: improvement and attainment} & \pageref{app:finite-separation} \\
\ref{app:support-examples} & \hyperref[app:support-examples]{Examples delimiting the support and rate assumptions} & \pageref{app:support-examples} \\
\ref{app:proofs-prelim} & \hyperref[app:proofs-prelim]{Posterior identities and source reduction} & \pageref{app:proofs-prelim} \\
\ref{app:proofs-geometry} & \hyperref[app:proofs-geometry]{Posterior quantization and uniform covering} & \pageref{app:proofs-geometry} \\
\ref{app:hellinger} & \hyperref[app:hellinger]{Hellinger covering and the local modulus} & \pageref{app:hellinger} \\
\ref{app:converses} & \hyperref[app:converses]{Further converses} & \pageref{app:converses} \\
\ref{app:parametric} & \hyperref[app:parametric]{Parametric rates: hypotheses and achievability} & \pageref{app:parametric} \\
\ref{app:parametric-rd-details} & \hyperref[app:parametric-rd-details]{Parametric rate--distortion calculations} & \pageref{app:parametric-rd-details} \\
\ref{app:experiments-v6} & \hyperref[app:experiments-v6]{Experimental protocol and additional results} & \pageref{app:experiments-v6} \\
\ref{app:related_works} & \hyperref[app:related_works]{Related work and connections} & \pageref{app:related_works} \\
\bottomrule
\end{tabularx}
\medskip

\section{Foundations for the separation theorem}
\label{app:separation-foundations}

Throughout these appendices, $\Theta$ and $\mathcal Y$ are standard Borel spaces, with finite Bayes information
$I:=I(\theta;Y)\in(0,\infty)$. Further, $\theta\to Y\to S$ is the summary Markov
chain. Auxiliary random variables are independent of the experiment unless
specified otherwise. 
\subsection{Notation}
\paragraph{Measures and quantizers.} In the general mixture statements below, $P$ and $Q$ denote probability
measures on the same measurable space $(\mathcal X,\mathcal A)$; $Q$ is
the reference measure in $\KL(P\Vert Q)$. We write $\mathcal G$ for the set of all countable-valued deterministic
quantizers (often called codes in the information-bottleneck literature), $T=g(Y)$ with $H(T)<\infty$. We include finite quantizers, and constants in $\mathcal G$ as well. 

For a quantizer $T=g(Y)$, write
\[
 p_c:=\mathbb P(T=c),\qquad
 P_T:=\sum_{c:\;p_c>0}p_c\delta_c,\qquad
 \Pi_c:=P_{\theta\mid T=c},\qquad
 \pi=\sum_{c:\;p_c>0}p_c\Pi_c.
\]
Thus $p_c$ is a cell probability, $P_T$ is the marginal law of the quantizer,
and $\Pi_c$ is the parameter posterior in cell $c$; $\delta_c$ is unit mass
at $c$. Thus $\Pi_T$ is the random posterior $\Pi_S$ of the main text
with $S=T$. For $p_c>0$, define its density relative to the prior by
\[
 f_c:=\frac{d\Pi_c}{d\pi},\qquad
 \Pi_c(d\theta)=f_c(\theta)\pi(d\theta).
\]
\paragraph{Relevance.} Following the information-bottleneck literature, we call $I(\theta;T)$ the
\emph{relevance} of the quantizer $T$: it is the information about the
parameter that the quantizer retains, and by Lemma~\ref{lem:identity} it
equals $I-\Delta_{\KL}(T)$, so a quantizer has high relevance exactly when
it has low posterior distortion. The same term is used below for the
relevance $I(\theta;S)$ of a general summary.


\paragraph{Singular mass and weighted Jensen-Shannon divergence.} For arbitrary $P,Q$, their Lebesgue decomposition is
$P=\tilde P_Q+P_Q^{\perp}$, where
$\tilde P_Q\ll Q$ and $P_Q^{\perp}\perp Q$.
The \emph{directed singular mass} is
\begin{equation}
\label{eq:singular_mass}
 \alpha(P\mid Q):=P_Q^{\perp}(\mathcal X)
 =\sup\{P(A):A\in\mathcal A,\ Q(A)=0\}.
\end{equation}
It measures the mass that $P$ places on events impossible under $Q$.
It lies in $[0,1]$, vanishes exactly when $P\ll Q$, and need not equal
$\alpha(Q\mid P)$. For $w\in(0,1)$, define the weighted Jensen-Shannon divergence as:
\begin{equation} \label{eq:weighted_js}
     M_w:=wP+(1-w)Q,\qquad
 \operatorname{JS}_w(P,Q)
 :=w\KL(P\Vert M_w)+(1-w)\KL(Q\Vert M_w).
\end{equation}
The first argument is the rare component as $w\downarrow0$.
We use $h_2(w)=-w\log w-(1-w)\log(1-w)$ for binary entropy.

\subsection{Results for deterministic summaries}
\begin{lemma}[Finite-rate deterministic summaries]
\label{lem:deterministic_atomic}
A deterministic summary with a standard Borel output space has finite
information rate if and only if its law is supported on countably many
atoms with finite entropy. In that case $I(Y;g(Y))=H(g(Y))$.
This restriction does not apply to stochastic summaries.
\end{lemma}
\begin{proof}
Let $T=g(Y)$ and let $\mu$ be the law of $T$ on its output space
$\mathcal X$. We first express the rate
of $T$ as a supremum of finite entropies, and then evaluate that supremum
separately for a law with a nonatomic part and for a purely atomic law.

Mutual information between general random variables is the supremum of
the mutual informations between their finite quantizations, and it is
enough to quantize one of the two arguments. Hence
\[
 I(Y;T)=\sup_{a}I(Y;a(T)),
\]
where $a$ ranges over measurable maps from the output space of $T$ to a
finite set. For a finite-valued $a(T)$, the mutual information is the
entropy difference $H(a(T))-H(a(T)\mid Y)$, and the conditional entropy
vanishes because $a(T)=a(g(Y))$ is a function of $Y$. Therefore
\begin{equation}
\label{eq:rate_as_sup_entropy}
 I(Y;T)=\sup_{a}H(a(T)).
\end{equation}
Since each $a(T)$ is a function of $T$, every term in the supremum is at
most $H(T)$, where $H(T)$ is the (possibly infinite) entropy of the
countable law when $\mu$ is atomic.

Suppose now that $\mu$ has a nonatomic component of total mass $c>0$.
Let $N$ be the countable set of atoms of $\mu$; the nonatomic part gives
$N$ zero mass, so $\mu$ restricted to $\mathcal X\setminus N$ is a
nonatomic measure of mass $c$. A nonatomic finite measure on a standard
Borel space takes every value between zero and its total mass on
measurable subsets, so $\mathcal X\setminus N$ can be partitioned into
measurable sets $A_1,\ldots,A_m$ with $\mu(A_i)=c/m$ for each $i$.
Define the quantizer $a$ to map $A_i$ to $i$ and $N$ to the symbol $0$.
Dropping the nonnegative term of the symbol $0$,
\[
 H(a(T))\ \ge\ \sum_{i=1}^m\frac cm\log\frac mc\ =\ c\log m-c\log c,
\]
which tends to infinity with $m$. By~\eqref{eq:rate_as_sup_entropy} the
rate is infinite.

Suppose instead that $\mu$ is purely atomic, $\mu=\sum_kp_k\delta_{x_k}$
with $p_1\ge p_2\ge\cdots$. Let $a_m$ retain the first $m$ atoms as
separate symbols and merge all remaining atoms into one symbol of mass
$r_m:=\sum_{k>m}p_k$. Then
\[
 H(a_m(T))=\sum_{k\le m}p_k\log\frac1{p_k}+r_m\log\frac1{r_m}.
\]
As $m\to\infty$, the partial sums increase to $H(T)$ and
$r_m\log(1/r_m)\to0$ because $r_m\to0$. Thus $\sup_aH(a(T))\ge H(T)$,
and the reverse inequality was noted above, so $I(Y;T)=H(T)$. The rate
is finite exactly when $H(T)<\infty$, which proves both assertions.
A stochastic summary can have a nonatomic law and finite rate, for
instance the observation plus independent Gaussian noise in
Proposition~\ref{thm:gaussian_exact}, so the restriction is specific to
deterministic quantizers.
\end{proof}

\begin{lemma}[Approximation by finite posterior quantizers]
\label{lem:finite_posterior_quantizers}
There are finite-valued deterministic quantizers $U_n$ such that
$I(\theta;U_n)\uparrow I$ and $\Delta_{\KL}(U_n)\downarrow0$.
Consequently, for every $\eta>0$ a finite quantizer has distortion less than
$\eta$, and $D^{\mathrm{det}}(R)\to0$ as $R\to\infty$.
\end{lemma}
\begin{proof}
The idea is to write $I(\theta;U_n)$ as the integral of a convex function
of a martingale that converges to the joint density of the experiment,
and to pass to the limit from both sides.

Because $\mathcal Y$ is standard Borel, its Borel sigma-field is generated
by a countable family of sets $C_1,C_2,\ldots$. Let $\mathcal P_n$ be the
finite partition of $\mathcal Y$ generated by $C_1,\ldots,C_n$, and let
$U_n$ be the label of the cell of $\mathcal P_n$ that contains $Y$. The
partitions refine one another, so $U_n$ is a function of $U_{n+1}$, each
$U_n$ takes finitely many values, and $H(U_n)<\infty$.

Put $Q_{\theta Y}:=\pi\otimes P_Y$. Since
$I=\KL(P_{\theta Y}\Vert Q_{\theta Y})<\infty$, the joint law is absolutely
continuous with respect to the product law, and we write
$f:=dP_{\theta Y}/dQ_{\theta Y}$. Consider the increasing sigma-fields
$\mathcal F_n:=\mathcal B(\Theta)\otimes\sigma(\mathcal P_n)$ on
$\Theta\times\mathcal Y$ and the conditional expectations
\[
 f_n:=\mathbb E_{Q_{\theta Y}}[f\mid\mathcal F_n].
\]
Each $f_n$ is $\mathcal F_n$-measurable, hence of the form
$f_n(\theta,y)=\sum_{C\in\mathcal P_n}f_{n,C}(\theta)\mathbf 1_C(y)$.
For $A\in\mathcal B(\Theta)$ and a cell $C$, the defining property of
conditional expectation gives
\[
 \mathbb P(\theta\in A,\,U_n=C)
 =\int_{A\times C}f\,dQ_{\theta Y}
 =\int_{A\times C}f_n\,dQ_{\theta Y}
 =P_Y(C)\int_Af_{n,C}\,d\pi .
\]
Thus $f_{n,C}$ is the density of the cell posterior $P_{\theta\mid U_n=C}$
relative to the prior whenever $P_Y(C)>0$, and
\begin{align*}
 \int f_n\log f_n\,dQ_{\theta Y}
 &=\sum_{C}P_Y(C)\int f_{n,C}\log f_{n,C}\,d\pi\\
 &=\sum_CP(U_n=C)\,\KL(P_{\theta\mid U_n=C}\Vert\pi)
 =I(\theta;U_n).
\end{align*}

We now show that these integrals converge to $I$. The sigma-fields
$\mathcal F_n$ increase to $\mathcal B(\Theta)\otimes\mathcal B(\mathcal Y)$,
so by L\'evy's upward theorem the martingale $(f_n)$ converges to $f$
almost surely under $Q_{\theta Y}$. Write $\phi(u):=u\log u$, a convex
function bounded below by $-1/e$. Conditional Jensen gives
$\phi(f_n)\le\mathbb E_{Q_{\theta Y}}[\phi(f)\mid\mathcal F_n]$, and
integrating,
\[
 I(\theta;U_n)=\int\phi(f_n)\,dQ_{\theta Y}
 \le\int\phi(f)\,dQ_{\theta Y}=I .
\]
For the reverse direction, apply Fatou's lemma to the nonnegative
functions $\phi(f_n)+1/e$, using continuity of $\phi$ and the almost sure
convergence $f_n\to f$:
\[
 I+\frac1e=\int\Big(\phi(f)+\frac1e\Big)dQ_{\theta Y}
 \le\liminf_n\int\Big(\phi(f_n)+\frac1e\Big)dQ_{\theta Y}
 =\liminf_nI(\theta;U_n)+\frac1e .
\]
Together, $I(\theta;U_n)\to I$. The sequence is nondecreasing because
$U_n$ is a function of $U_{n+1}$ and data processing gives
$I(\theta;U_n)\le I(\theta;U_{n+1})$. Lemma~\ref{lem:identity} then gives
$\Delta_{\KL}(U_n)=I-I(\theta;U_n)\downarrow0$.

For the consequences, fix $\eta>0$ and choose $n$ with
$\Delta_{\KL}(U_n)<\eta$; this is the required finite quantizer. For every
$R\ge H(U_n)$ the quantizer $U_n$ is feasible in the definition of
$D^{\mathrm{det}}(R)$, so $D^{\mathrm{det}}(R)\le\Delta_{\KL}(U_n)<\eta$.
Since $D^{\mathrm{det}}$ is nonincreasing and nonnegative, and $\eta$ was
arbitrary, $D^{\mathrm{det}}(R)\to0$ as $R\to\infty$.
\end{proof}

\subsection{Properties of time-sharing mixtures}
\begin{proposition}[Revealed time-sharing]
\label{prop:time_sharing}
\label{lem:revealed_selector}
Let $S_1,\ldots,S_k$ be summary channels and let $J$ be independent of
the experiment and their auxiliary randomness, with probabilities
$\lambda_j$. The revealed mixture $S=(J,S_J)$ satisfies
\begin{equation}
\label{eq:ts_affine}
 I(Y;S)=\sum_j\lambda_j I(Y;S_j),\qquad
 \Delta_{\KL}(S)=\sum_j\lambda_j\Delta_{\KL}(S_j).
\end{equation}
For deterministic components $S_j=g_j(Y)$, the component rate is
$H(g_j(Y))$, and their finite revealed mixtures achieve exactly
$\operatorname{conv}(\mathcal A_{\mathrm{det}})$.
The frontier $D^{\mathrm{ts}}$ in~\eqref{eq:ts_frontier} is convex,
nonincreasing, and continuous on $[0,\infty)$, with
$D^{\mathrm{ts}}(0)=I$ and $D^{\mathrm{ts}}(R)\to0$ as $R\to\infty$.
Moreover $0\le D^\star\le D^{\mathrm{ts}}\le D^{\mathrm{det}}\le I$.
Its zero-distortion threshold satisfies $R_0\ge I>0$;
$D^{\mathrm{ts}}>0$ on $[0,R_0)$, and, if $R_0<\infty$,
$D^{\mathrm{ts}}=0$ on $[R_0,\infty)$.
\end{proposition}
\begin{proof}
We prove the two affine identities first, then identify the achieved
pairs, and finally derive the shape properties of $D^{\mathrm{ts}}$ from
them.

The revealed mixture is a valid summary channel: given $Y$, the pair
$(J,S_J)$ is generated from $Y$ and randomness independent of $\theta$,
so $\theta\to Y\to(J,S_J)$ is a Markov chain. By the chain rule for
mutual information,
\[
 I(Y;J,S_J)=I(Y;J)+I(Y;S_J\mid J)=0+\sum_j\lambda_jI(Y;S_J\mid J=j).
\]
The first term vanishes because $J$ is independent of $Y$. On the event
$J=j$, the pair $(Y,S_J)$ equals $(Y,S_j)$, and the conditional law of
$(Y,S_j)$ given $J=j$ is its unconditional law because $J$ is independent
of the experiment and of the auxiliary randomness of $S_j$. Hence
$I(Y;S_J\mid J=j)=I(Y;S_j)$, which is the first identity
in~\eqref{eq:ts_affine}. The same conditioning gives
\[
 I(\theta;Y\mid J,S_J)=\sum_j\lambda_jI(\theta;Y\mid S_J,J=j)
 =\sum_j\lambda_jI(\theta;Y\mid S_j),
\]
and Lemma~\ref{lem:identity} identifies each conditional information with
the corresponding posterior distortion, proving the second identity.

For deterministic components $S_j=g_j(Y)$, Lemma~\ref{lem:deterministic_atomic}
gives $I(Y;S_j)=H(g_j(Y))$, so the mixture $(J,g_J(Y))$ realizes the pair
$\sum_j\lambda_j\big(H(g_j(Y)),\Delta_{\KL}(g_j(Y))\big)$. Conversely,
every finite convex combination of points of $\mathcal A_{\mathrm{det}}$
has this form for a suitable finite law of $J$. The achieved set is
therefore exactly $\operatorname{conv}(\mathcal A_{\mathrm{det}})$.

We turn to the shape of $D^{\mathrm{ts}}$. The constant quantizer gives
$(0,I)\in\mathcal A_{\mathrm{det}}$, so $D^{\mathrm{ts}}(R)\le I<\infty$
for every $R\ge0$, and $D^{\mathrm{ts}}$ is nonincreasing because the
feasible set $\{(r,d)\in\operatorname{conv}(\mathcal A_{\mathrm{det}}):r\le R\}$
grows with $R$. For convexity, fix $R_1,R_2\ge0$, $\mu\in(0,1)$, and
$\varepsilon>0$, and choose feasible $(r_i,d_i)$ with $r_i\le R_i$ and
$d_i\le D^{\mathrm{ts}}(R_i)+\varepsilon$. The convex combination
$\mu(r_1,d_1)+(1-\mu)(r_2,d_2)$ again lies in
$\operatorname{conv}(\mathcal A_{\mathrm{det}})$ and has rate at most
$\mu R_1+(1-\mu)R_2$, so
\[
 D^{\mathrm{ts}}(\mu R_1+(1-\mu)R_2)
 \le\mu d_1+(1-\mu)d_2
 \le\mu D^{\mathrm{ts}}(R_1)+(1-\mu)D^{\mathrm{ts}}(R_2)+\varepsilon,
\]
and $\varepsilon\downarrow0$ gives convexity.

The chain of inequalities follows from inclusions of feasible classes.
A single deterministic quantizer with $H(g(Y))\le R$ is a one-component mixture,
so $D^{\mathrm{ts}}\le D^{\mathrm{det}}$, and the constant quantizer gives
$D^{\mathrm{det}}\le I$. Every feasible point $(r,d)$ of
$\operatorname{conv}(\mathcal A_{\mathrm{det}})$ with $r\le R$ is realized
by a revealed mixture, which is a summary channel of rate $r\le R$ and
distortion $d$, so $D^\star(R)\le d$; taking the infimum over such points
gives $D^\star\le D^{\mathrm{ts}}$. Finally $D^\star\ge0$ because
divergences are nonnegative.

For continuity, a finite convex function on $[0,\infty)$ is continuous on
the open half-line $(0,\infty)$. At the origin, only rate-zero mixtures
are feasible, that is, mixtures of constant quantizers, each with distortion
$I(\theta;Y\mid\text{constant})=I$; thus $D^{\mathrm{ts}}(0)=I$. The
universal converse in Lemma~\ref{lem:identity} gives, for $0\le R<I$,
\[
 I-R\le D^\star(R)\le D^{\mathrm{ts}}(R)\le I,
\]
so $D^{\mathrm{ts}}(R)\to I$ as $R\downarrow0$, which is continuity from
the right at zero. The limit at infinity follows from
$0\le D^{\mathrm{ts}}\le D^{\mathrm{det}}$ and
Lemma~\ref{lem:finite_posterior_quantizers}.

For the threshold, the display above shows $D^{\mathrm{ts}}(R)\ge I-R>0$
for $R<I$, so $R_0\ge I>0$. On $[0,R_0)$ the frontier is positive by the
definition of $R_0$ as an infimum. If $R_0<\infty$, then
$D^{\mathrm{ts}}(R)=0$ for some $R$ arbitrarily close to $R_0$ from the
right, monotonicity gives $D^{\mathrm{ts}}=0$ on $(R_0,\infty)$, and
continuity gives $D^{\mathrm{ts}}(R_0)=0$ as well.
\end{proof}

\begin{lemma}[Truncating countable quantizers]
\label{lem:truncate_quantizers}
If $T\in\mathcal G$, define $T_n=T$ on $\{T\le n\}$ and $T_n=n+1$
otherwise. Then $H(T_n)\le H(T)$ and
$\Delta_{\KL}(T_n)\downarrow\Delta_{\KL}(T)$.
Thus $D^{\mathrm{ts}}$ is unchanged if its component quantizers must be finite.
\end{lemma}
\begin{proof}
Take the quantizer values in $\mathbb N=\{1,2,\ldots\}$. Since $T_n$ is a
function of $T$, both $H(T_n)\le H(T)$ and, by data processing,
$I(\theta;T_n)\le I(\theta;T)$. Moreover $T_n$ is a function of
$T_{n+1}$, so $I(\theta;T_n)$ is nondecreasing in $n$, and by
Lemma~\ref{lem:identity} the distortions $\Delta_{\KL}(T_n)=I-I(\theta;T_n)$
are nonincreasing in $n$ and bounded below by $\Delta_{\KL}(T)$.

To bound the gap, apply the chain rule to $I(\theta;T,T_n)$ in both
orders. Because $T_n$ is a function of $T$,
\[
 I(\theta;T)=I(\theta;T,T_n)=I(\theta;T_n)+I(\theta;T\mid T_n),
\]
so that
\[
 \Delta_{\KL}(T_n)-\Delta_{\KL}(T)=I(\theta;T\mid T_n)\le H(T\mid T_n).
\]
Conditioning on $T_n=c$ with $c\le n$ determines $T=c$ and contributes
no entropy. Conditioning on $T_n=n+1$ leaves the tail law
$(p_c/r_n)_{c>n}$ with $r_n:=\sum_{c>n}p_c$, so
\[
 H(T\mid T_n)=\sum_{c>n}p_c\log\frac{r_n}{p_c}
 \le\sum_{c>n}p_c\log\frac1{p_c},
\]
using $r_n\le1$. The right side is a tail of the convergent series
$H(T)=\sum_cp_c\log(1/p_c)$, hence tends to zero. This proves
$\Delta_{\KL}(T_n)\downarrow\Delta_{\KL}(T)$.

For the final claim, let $(\lambda_j,T_j)_{j\le k}$ be any finite mixture
of quantizers in $\mathcal G$ with average rate at most $R$ and average
distortion $d$. Truncating every component at level $n$ gives a mixture of
finite quantizers whose average rate is still at most $R$, by the first
inequality, and whose average distortion
$\sum_j\lambda_j\Delta_{\KL}(T_{j,n})$ decreases to $d$ as $n\to\infty$.
Hence the infimum over mixtures of finite quantizers is at most $d$, and
taking the infimum over all mixtures shows it is at most
$D^{\mathrm{ts}}(R)$. The reverse inequality holds because finite quantizers
belong to $\mathcal G$.
\end{proof}

\begin{lemma}[Cost of exact preservation]
\label{lem:zero_distortion_cost}
If $\Delta_{\KL}(S)=0$, then $\Pi_Y=\Pi_S$ almost surely, hence $\Pi_Y$
is measurable from $S$. If $\Pi_Y$ is countably valued, every such summary
has rate at least $H(\Pi_Y)$, allowing infinite entropy.
\end{lemma}
\begin{proof}
The random variable $\KL(\Pi_Y\Vert\Pi_S)$ is nonnegative, so
$\mathbb E\,\KL(\Pi_Y\Vert\Pi_S)=0$ forces $\KL(\Pi_Y\Vert\Pi_S)=0$ almost
surely, and a divergence vanishes only between equal measures. Thus
$\Pi_Y=\Pi_S$ almost surely, and $\Pi_S$ is a measurable function of $S$.

Suppose the posterior label $Z:=\Pi_Y$ takes countably many values. Since
$Z$ is a function of $Y$ and $\Pi_S$ is a function of $S$, two
applications of data processing give
\[
 I(Y;S)\ \ge\ I(Z;S)\ \ge\ I(Z;\Pi_S)\ =\ I(Z;Z),
\]
the last equality because $\Pi_S=Z$ almost surely. For a countably valued
$Z$, the identity $I(Z;Z)=H(Z)$ holds with both sides possibly infinite:
it is the purely atomic case of Lemma~\ref{lem:deterministic_atomic}, where
finite quantizations of $Z$ have entropies increasing to $H(Z)$.
\end{proof}

\begin{lemma}[Countable mixture compensation]
\label{lem:countable_compensation}
Let $\bar P=\sum_j\lambda_jP_j$, with $\lambda_j>0$, and suppose
$\KL(\bar P\Vert Q)<\infty$. Then every $\KL(P_j\Vert Q)$ is finite,
$\KL(P_j\Vert\bar P)\le\log(1/\lambda_j)$, and
\[
 \sum_j\lambda_j\KL(P_j\Vert Q)
 =\sum_j\lambda_j\KL(P_j\Vert\bar P)+\KL(\bar P\Vert Q),
\]
with the sums allowed to be infinite.
\end{lemma}
\begin{proof}
We first collect the integrability facts, then prove the identity for a
single component, and finally sum over components.

Since $\lambda_jP_j\le\bar P$, we have $P_j\ll\bar P$ with
$dP_j/d\bar P\le1/\lambda_j$, and therefore
\[
 \KL(P_j\Vert\bar P)=\int\log\frac{dP_j}{d\bar P}\,dP_j\le\log\frac1{\lambda_j}.
\]
The hypothesis $\KL(\bar P\Vert Q)<\infty$ gives $\bar P\ll Q$; write
$r:=d\bar P/dQ$. The negative part of $\log r$ is $\bar P$-integrable
without any assumption: since $u\log(1/u)\le1/e$ for $u\in(0,1]$,
\[
 \int(\log r)_-\,d\bar P=\int_{\{r<1\}}r\log\frac1r\,dQ\le\frac1e .
\]
The positive part is then integrable because
$\int(\log r)_+\,d\bar P=\KL(\bar P\Vert Q)+\int(\log r)_-\,d\bar P<\infty$.
Hence $\log r\in L^1(\bar P)$, and by the density bound also
$\log r\in L^1(P_j)$ with $\int|\log r|\,dP_j\le\lambda_j^{-1}\int|\log r|\,d\bar P$.

Since $P_j\ll\bar P\ll Q$, the chain rule for densities gives
$dP_j/dQ=(dP_j/d\bar P)\,r$ almost everywhere under $Q$, hence
\[
 \log\frac{dP_j}{dQ}=\log\frac{dP_j}{d\bar P}+\log r
 \qquad P_j\text{-almost surely.}
\]
Both terms on the right are $P_j$-integrable, so integrating against $P_j$
is legitimate and yields
\[
 \KL(P_j\Vert Q)=\KL(P_j\Vert\bar P)+\int\log r\,dP_j ,
\]
a finite quantity. This proves the first two claims.

Multiply the last display by $\lambda_j$ and sum over $j$. The series
$\sum_j\lambda_j\int|\log r|\,dP_j=\int|\log r|\,d\bar P$ converges, so
the sum and the integral in the last term can be exchanged:
\[
 \sum_j\lambda_j\int\log r\,dP_j=\int\log r\,d\bar P=\KL(\bar P\Vert Q).
\]
The remaining two series have nonnegative terms, and they differ by the
finite number $\KL(\bar P\Vert Q)$, so they are either both finite or both
infinite, and in the finite case the identity holds as stated.
\end{proof}
\subsection{Asymptotics of mixtures}
\begin{lemma}[Rare-mixture asymptotics]
\label{lem:rare_mixture}
 For the weighted Jensen-Shannon divergence $\operatorname{JS}_w(P,Q)$ defined in \eqref{eq:weighted_js} the following holds:
\begin{enumerate}[label=(\roman*),leftmargin=*,itemsep=3pt]
    \item $\operatorname{JS}_w(P,Q)$ is at most $h_2(w)$ and decreases under measurable maps.
    \item Let $M_w=wP + (1-w)Q$, and $\alpha(P\mid Q)$ denote the directed singular mass. Then, \[
 \lim_{w\downarrow0}\frac{\KL(P\Vert M_w)}{\log(1/w)}
 =\alpha(P\mid Q),\qquad
 \lim_{w\downarrow0}\frac{\operatorname{JS}_w(P,Q)}{w\log(1/w)}
 =\alpha(P\mid Q).
\]
In particular, $\operatorname{JS}_w(P,Q)=o(w\log(1/w))$ exactly when
$P\ll Q$.
\item  If $\KL(P\Vert Q)<\infty$, then
\[
 \operatorname{JS}_w(P,Q)
 =w\KL(P\Vert Q)-\KL(M_w\Vert Q)
 \le w\KL(P\Vert Q).
\]
\end{enumerate}
\end{lemma}
\begin{proof}
The first follows from an information-theoretic reading of
$\operatorname{JS}_w$. Let $\xi\sim\operatorname{Bernoulli}(w)$ and, given
$\xi$, draw $X$ from $P$ if $\xi=1$ and from $Q$ if $\xi=0$. Then
$X\sim M_w$, and the mutual information between the label and the draw is
\[
 I(\xi;X)=w\KL(P\Vert M_w)+(1-w)\KL(Q\Vert M_w)=\operatorname{JS}_w(P,Q).
\]
Hence $\operatorname{JS}_w(P,Q)\le H(\xi)=h_2(w)$, and for a measurable
map $\varphi$, data processing gives
$\operatorname{JS}_w(P\circ\varphi^{-1},Q\circ\varphi^{-1})=I(\xi;\varphi(X))\le I(\xi;X)$.

We now prove the limits. Take densities $p:=dP/d(P+Q)$ and
$q:=dQ/d(P+Q)$, so that $M_w$ has density $wp+(1-w)q$. Since
$M_w\ge wP$, the density $dP/dM_w=p/(wp+(1-w)q)$ is at most $1/w$ on
$\{p>0\}$, so $\KL(P\Vert M_w)\le\log(1/w)$ is finite, and
\[
 \KL(P\Vert M_w)=\int_{\{p>0\}}\log\frac{p}{wp+(1-w)q}\,dP .
\]
Split the domain into $S_0:=\{p>0,q=0\}$ and $S_1:=\{p>0,q>0\}$. On $S_0$
the integrand equals $\log(1/w)$ exactly. The set $S_0$ has $Q$-measure
zero, and $P$ restricted to $\{q=0\}$ is the singular part of $P$
relative to $Q$, so $P(S_0)=\alpha(P\mid Q)$ and
\[
 \int_{S_0}\log\frac{p}{wp+(1-w)q}\,dP=\alpha(P\mid Q)\log\frac1w .
\]
On $S_1$, write the integrand as $g_w=g_w^+-g_w^-$. As $w\downarrow0$ the
ratio $p/(wp+(1-w)q)$ converges pointwise to $p/q\in(0,\infty)$, so
$g_w$ converges to the finite number $\log(p/q)$ and
$g_w^+/\log(1/w)\to0$ pointwise on $S_1$. The bound $g_w^+\le\log(1/w)$
makes $g_w^+/\log(1/w)\le1$, which is $P$-integrable, so dominated
convergence gives $\int_{S_1}g_w^+\,dP=o(\log(1/w))$. For the negative
part, put $m:=wp+(1-w)q$ and note that on $\{m>p\}$,
\[
 p\log\frac mp=m\cdot\frac pm\log\frac mp\le\frac me,
\]
again by $u\log(1/u)\le1/e$. Integrating against $P+Q$ gives
$\int_{S_1}g_w^-\,dP\le\frac1e\int m\,d(P+Q)=\frac1e$, uniformly in $w$.
Dividing by $\log(1/w)$ and combining the three contributions,
\[
 \frac{\KL(P\Vert M_w)}{\log(1/w)}
 =\alpha(P\mid Q)+\frac{\int_{S_1}g_w^+\,dP}{\log(1/w)}
 -\frac{\int_{S_1}g_w^-\,dP}{\log(1/w)}
 \longrightarrow\alpha(P\mid Q).
\]

For the second limit, the other component of $\operatorname{JS}_w$ is
small: $dQ/dM_w\le1/(1-w)$ gives $\KL(Q\Vert M_w)\le\log\frac1{1-w}$, and
$-\log(1-w)/w\to1$ while $\log(1/w)\to\infty$, so
\[
 \frac{(1-w)\KL(Q\Vert M_w)}{w\log(1/w)}
 \le\frac{-\log(1-w)}{w\log(1/w)}\longrightarrow0 .
\]
Dividing $\operatorname{JS}_w(P,Q)=w\KL(P\Vert M_w)+(1-w)\KL(Q\Vert M_w)$
by $w\log(1/w)$ and using the first limit gives the second. Since
$\alpha(P\mid Q)=0$ exactly when $P\ll Q$, the divergence is
$o(w\log(1/w))$ exactly in that case.

Finally assume $\KL(P\Vert Q)<\infty$. Apply
Lemma~\ref{lem:countable_compensation} to the two-component mixture
$M_w=wP+(1-w)Q$ with reference $Q$. Its hypothesis holds because
convexity of the divergence in its first argument gives
$\KL(M_w\Vert Q)\le w\KL(P\Vert Q)+(1-w)\KL(Q\Vert Q)=w\KL(P\Vert Q)<\infty$.
The lemma's identity reads
\[
 w\KL(P\Vert Q)+(1-w)\KL(Q\Vert Q)
 =w\KL(P\Vert M_w)+(1-w)\KL(Q\Vert M_w)+\KL(M_w\Vert Q),
\]
that is, $w\KL(P\Vert Q)=\operatorname{JS}_w(P,Q)+\KL(M_w\Vert Q)$.
Rearranging gives the stated identity, and dropping the nonnegative term
$\KL(M_w\Vert Q)$ gives the inequality.
\end{proof}

\section{Strict separation for general observation spaces}
\label{app:general-separation}

\subsection{Notation for the experiment}
For $T=g(Y)$, write $P_{\theta Y}$ and $P_{\theta T}$ for the joint laws
of $(\theta,Y)$ and $(\theta,T)$. The corresponding product laws are
\[
 Q_{\theta Y}:=\pi\otimes P_Y,\qquad
 Q_{\theta T}:=\pi\otimes P_T,
 \qquad I=\KL(P_{\theta Y}\Vert Q_{\theta Y}).
\]
Here $P_Y$ and $P_T$ are the marginal laws of the observation and quantizer.
Use the cell probabilities $p_c=P_T(\{c\})$ and cell posteriors $\Pi_c$
defined in Appendix~\ref{app:separation-foundations}.
For a measurable $A\subseteq\Theta$ and a positive-probability cell $c$,
\[
 P_{\theta T}(A\times\{c\})=p_c\Pi_c(A),\qquad
 Q_{\theta T}(A\times\{c\})=p_c\pi(A).
\]
The laws $P_{\theta T},Q_{\theta T}$ are the images of
$P_{\theta Y},Q_{\theta Y}$ under $(\theta,y)\mapsto(\theta,g(y))$:
$P_{\theta T}$ retains the parameter--quantizer dependence, while
$Q_{\theta T}$ makes them independent with the same marginals.

Since $\KL(P_{\theta Y}\Vert Q_{\theta Y})=I<\infty$, a joint density
$f=dP_{\theta Y}/dQ_{\theta Y}$ exists. Versions of the conditional distributions satisfy
\[
 \Pi_y(d\theta)=f(\theta,y)\pi(d\theta),\qquad
 P_{Y\mid\theta}(dy)=f(\theta,y)P_Y(dy)
\]
for almost every conditioning value. Indeed, Tonelli's theorem and the
two marginal identities give $\int f(\theta,y)\pi(d\theta)=1$
for $P_Y$-almost every $y$, and $\int f(\theta,y)P_Y(dy)=1$
for $\pi$-almost every $\theta$. Defining the kernels arbitrarily on
the exceptional null sets gives regular conditional versions.
\subsection{The common support assumption}
\begin{proposition}[Equivalent support formulations]
\label{prop:support_equivalence}
Under $I(\theta;Y)<\infty$, Assumption~\ref{ass:common_support} is
equivalent to each of
\begin{enumerate}[label=(\roman*),leftmargin=*]
\item $P_{\theta Y}\sim\pi\otimes P_Y$;
\item $\Pi_y\sim\pi$ for $P_Y$-almost every $y$;
\item $P_{Y\mid\theta}\sim P_Y$ for $\pi$-almost every $\theta$.
\end{enumerate}
For a dominated simulator with density $p(y\mid\theta)$, these conditions
are equivalent to $p(y\mid\theta)>0$ for
$\pi\otimes P_Y$-almost every $(\theta,y)$.
\end{proposition}
\begin{proof}
All three conditions concern the zero set of the joint density $f$. We
first show that each of (i), (ii), (iii) is equivalent to
$Q_{\theta Y}\{f=0\}=0$, and then relate (iii) to the assumption as stated.

For (i): we already know $P_{\theta Y}\ll Q_{\theta Y}$, so (i) is the
reverse relation $Q_{\theta Y}\ll P_{\theta Y}$. If
$Q_{\theta Y}\{f=0\}=0$ and $P_{\theta Y}(A)=\int_Af\,dQ_{\theta Y}=0$,
then $f=0$ almost everywhere on $A$ under $Q_{\theta Y}$, so
$Q_{\theta Y}(A)=Q_{\theta Y}(A\cap\{f=0\})=0$. Conversely, the set
$\{f=0\}$ has $P_{\theta Y}$-measure $\int_{\{f=0\}}f\,dQ_{\theta Y}=0$,
so $Q_{\theta Y}\ll P_{\theta Y}$ forces $Q_{\theta Y}\{f=0\}=0$. Thus
\[
 \text{(i)}\iff Q_{\theta Y}\{f=0\}=0 .
\]
By Tonelli's theorem this common quantity can be computed in either order:
\[
 Q_{\theta Y}\{f=0\}
 =\int\pi\{\theta:f(\theta,y)=0\}\,P_Y(dy)
 =\int P_Y\{y:f(\theta,y)=0\}\,\pi(d\theta).
\]

For (ii): by the displayed version of the posterior, $\Pi_y\ll\pi$ holds
for $P_Y$-almost every $y$ with density $f(\cdot,y)$. The same one-variable
argument as above shows that $\pi\ll\Pi_y$ holds exactly when
$\pi\{\theta:f(\theta,y)=0\}=0$. Hence (ii) holds exactly when the
integrand of the first integral vanishes for $P_Y$-almost every $y$,
that is, exactly when $Q_{\theta Y}\{f=0\}=0$.

For (iii): likewise $P_{Y\mid\theta}\ll P_Y$ with density $f(\theta,\cdot)$
for $\pi$-almost every $\theta$, and $P_Y\ll P_{Y\mid\theta}$ holds exactly
when $P_Y\{y:f(\theta,y)=0\}=0$. Hence (iii) holds exactly when the
integrand of the second integral vanishes for $\pi$-almost every $\theta$,
again exactly when $Q_{\theta Y}\{f=0\}=0$.

It remains to connect (iii) with Assumption~\ref{ass:common_support},
which asks for a measurable set $\Theta_0$ with $\pi(\Theta_0)=1$ on
which all sampling laws are mutually absolutely continuous.
Suppose first that the assumption holds and fix $\theta_0\in\Theta_0$.
Since $\pi(\Theta_0)=1$, the predictive law is the mixture
$P_Y=\int_{\Theta_0}P_{Y\mid\theta}\,\pi(d\theta)$. If
$P_{Y\mid\theta_0}(A)=0$, then $P_{Y\mid\theta}(A)=0$ for every
$\theta\in\Theta_0$ by equivalence, so $P_Y(A)=0$. If $P_Y(A)=0$, then
$P_{Y\mid\theta}(A)=0$ for $\pi$-almost every $\theta\in\Theta_0$, hence
for at least one such $\theta$, and equivalence gives
$P_{Y\mid\theta_0}(A)=0$. Thus $P_{Y\mid\theta_0}\sim P_Y$, and therefore
$P_{Y\mid\theta}\sim P_{Y\mid\theta_0}\sim P_Y$ for every
$\theta\in\Theta_0$, which is (iii).
Conversely, suppose (iii) holds and let
$\Theta_0:=\{\theta:P_Y\{y:f(\theta,y)=0\}=0\}$, a measurable set by
Tonelli's theorem, intersected with the full-measure set on which the
displayed version of $P_{Y\mid\theta}$ is a probability measure. Then
$\pi(\Theta_0)=1$ by (iii), and for $\theta,\theta'\in\Theta_0$ we have
$P_{Y\mid\theta}\sim P_Y\sim P_{Y\mid\theta'}$, so the assumption holds
with this $\Theta_0$.

Finally, let the simulator be dominated by a sigma-finite measure $\nu$
with density $p(y\mid\theta)$ and predictive density
$m(y)=\int p(y\mid\theta)\pi(d\theta)$, so that $P_Y(dy)=m(y)\nu(dy)$.
Then $P_{\theta Y}(d\theta,dy)=p(y\mid\theta)\pi(d\theta)\nu(dy)$ and
$Q_{\theta Y}(d\theta,dy)=m(y)\pi(d\theta)\nu(dy)$, so on the set
$\{m>0\}$, which carries all of $P_Y$, the joint density is
$f(\theta,y)=p(y\mid\theta)/m(y)$. The condition $Q_{\theta Y}\{f=0\}=0$
is therefore exactly $p(y\mid\theta)>0$ for $\pi\otimes P_Y$-almost every
$(\theta,y)$.
\end{proof}

\paragraph{Scope of the support condition.}
The condition is mutual absolute continuity, not equality of topological
supports. A sufficient statistic preserves it: conditional distributions
of the raw observation given the statistic are parameter-independent,
so equivalence of the statistic laws lifts through their shared kernel.
It is also preserved by finite i.i.d. replication. A parameter-independent
observation kernel having strictly positive densities on a common support
gives equivalent sampling laws after integrating any latent dynamics.
Finiteness of $I$ must still be checked separately; equivalence alone
does not bound mutual information.

\subsection{Supporting lemmas for Theorem \ref{thm:separation}}
For a quantizer $T=g(Y)$ and a cell with $p_c>0$, the cell posterior is the
average of the observation posteriors over the cell, so its density
relative to the prior is
\[
 f_c(\theta)=p_c^{-1}\int_{\{g=c\}}f(\theta,y)P_Y(dy)
 \le p_c^{-1}\int f(\theta,y)P_Y(dy)
 =p_c^{-1}
\]
for $\pi$-almost every $\theta$. Hence
$\KL(\Pi_c\Vert\pi)=\int\log f_c\,d\Pi_c\le\log(1/p_c)$, and
$I(\theta;T)=\sum_c p_c\KL(\Pi_c\Vert\pi)\le I$ by data processing.
Define the integrable posterior-information function
$\varphi(y)=\KL(\Pi_y\Vert\pi)$, whose integral is
$\int\varphi\,dP_Y=I<\infty$, and, for $0\le\varepsilon<1$, set the moduli
\[
 \omega(\delta)=
 \sup_{A:\; P_Y(A)\le\delta}\int_A\varphi\,dP_Y,
 \qquad
 \psi(\delta)=\log\frac1{1-\delta}+\omega(\delta).
\]
Absolute continuity of the integral gives $\omega(\delta)\to0$ as
$\delta\downarrow0$. For any $\Lambda>0$, we have $\omega(\delta)\le\Lambda\delta+\int(\varphi-\Lambda)_+\,dP_Y$. The last
integral tends to zero as $\Lambda\to\infty$ by dominated convergence,
because $0\le(\varphi-\Lambda)_+\le\varphi$ and $\varphi$ is integrable.
Then, choosing $\Lambda$ large and then $\delta$ small makes $\omega(\delta)$
arbitrarily small, and both moduli are nondecreasing in $\delta$.

\begin{lemma}[Informative quantizers cannot be nearly constant]
\label{lem:informative_quantizer_collision}
For $T\in\mathcal G$, let
$p_{\max}=\max_c p_c$ and $\kappa(T)=1-\sum_c p_c^2$.
Then
\[
 1-p_{\max}\le\kappa(T)\le2(1-p_{\max}),\qquad
 I(\theta;T)\le\psi(1-p_{\max}).
\]
Consequently, if $\psi(\delta)<u$, every quantizer with relevance at least
$u$ has $\kappa(T)>\delta$.
\end{lemma}
\begin{proof}
The collision bounds come from comparing $\sum_cp_c^2$ with $p_{\max}$
in both directions. Since every $p_c\le p_{\max}$ and the $p_c$ sum to
one, $\sum_cp_c^2\le p_{\max}\sum_cp_c=p_{\max}$, so
$\kappa(T)\ge1-p_{\max}$. Since the largest cell alone contributes
$p_{\max}^2$,
\[
 \kappa(T)\le1-p_{\max}^2=(1-p_{\max})(1+p_{\max})\le2(1-p_{\max}).
\]

For the relevance bound, let $b$ be a cell with $p_b=p_{\max}$ and split
$I(\theta;T)=\sum_cp_c\KL(\Pi_c\Vert\pi)$ into the largest cell and the
rest. For the largest cell, the density bound $f_b\le1/p_b$ gives
\[
 p_b\KL(\Pi_b\Vert\pi)\le p_b\log\frac1{p_b}\le\log\frac1{p_b}.
\]
For every other cell, the cell posterior is the mixture
$\Pi_c=p_c^{-1}\int_{\{g=c\}}\Pi_y\,P_Y(dy)$. The divergence
$P\mapsto\KL(P\Vert\pi)$ is convex, being a supremum of affine functionals
of $P$ by the Donsker--Varadhan representation, so the integral form of
Jensen's inequality gives
\[
 p_c\KL(\Pi_c\Vert\pi)
 \le\int_{\{g=c\}}\KL(\Pi_y\Vert\pi)\,P_Y(dy)
 =\int_{\{g=c\}}\varphi\,dP_Y .
\]
Summing over $c\ne b$, the cells $\{g=c\}$ with $c\ne b$ are disjoint with
union $\{g\ne b\}$, a set of $P_Y$-probability $1-p_b$. By the definition
of $\omega$,
\[
 \sum_{c\ne b}p_c\KL(\Pi_c\Vert\pi)\le\int_{\{g\ne b\}}\varphi\,dP_Y
 \le\omega(1-p_b).
\]
Adding the two parts,
$I(\theta;T)\le\log(1/p_b)+\omega(1-p_b)=\psi(1-p_{\max})$.

For the last assertion, suppose $\psi(\delta)<u\le I(\theta;T)$. Then
$\psi(1-p_{\max})\ge I(\theta;T)>\psi(\delta)$, and since $\psi$ is
nondecreasing this forces $1-p_{\max}>\delta$. The lower collision bound
gives $\kappa(T)\ge1-p_{\max}>\delta$.
\end{proof}

\begin{definition}[Lautum information {\citep{palomar2008lautum}}]
\label{def:lautum_information}
The lautum information of the experiment is the reverse product/joint
divergence
\[
 L:=\KL(Q_{\theta Y}\Vert P_{\theta Y})
   =\int\KL(\pi\Vert\Pi_y)\,P_Y(dy),
\]
and its quantizer-level version is
\[
 L_T:=\KL(Q_{\theta T}\Vert P_{\theta T})
     =\sum_{c:p_c>0}p_c\KL(\pi\Vert\Pi_c).
\]
Either quantity may be infinite. Mutual information uses the opposite
direction, $I=\KL(P_{\theta Y}\Vert Q_{\theta Y})$; its finiteness does not imply $L<\infty$.
\end{definition}
Since $P_{\theta T},Q_{\theta T}$ are images of $P_{\theta Y},Q_{\theta Y}$ under the same map, data processing
gives $L_T\le L$. Finite lautum information will yield a quantitative
bound; the qualitative separation needs only the support condition.

\begin{lemma}[Random relabeling: channel law and rate saving]
\label{lem:random_relabeling}
For $T\in\mathcal G$ and $t\in(0,1)$, independently draw
$\xi\sim\operatorname{Bernoulli}(t)$ and $T'\sim P_T$, and set
$V_t=T$ if $\xi=0$ and $V_t=T'$ otherwise.
\begin{enumerate}[label=(\roman*),leftmargin=*,itemsep=3pt]
\item \emph{Relabeling preserves the quantizer marginal and mixes the joint law.}
\[
 V_t\sim P_T,\qquad
 P_{\theta V_t}=(1-t)P_{\theta T}+tQ_{\theta T}.
\]
\item \emph{Relabeling gives a $t\log(1/t)$ lower bound on the rate saving.}
We have $I(Y;V_t)=I(T;V_t)$, and
\begin{equation}
\label{eq:relabel_saving}
\begin{aligned}
 B_t(T)&:=H(T)-I(Y;V_t)\\
       &=H(T\mid V_t)=H(V_t\mid T)
         \ge\kappa(T)t\log(1/t).
\end{aligned}
\end{equation}
\item \emph{A nonconstant quantizer retains positive information after relabeling.}
\[
 0<B_t(T)<H(T)\qquad\text{if }H(T)>0.
\]
\end{enumerate}
\end{lemma}
\begin{proof}
The pair $(\xi,T')$ is independent of $(\theta,Y)$, so $V_t$ is obtained
from $T$ by a channel that does not look at $\theta$ or $Y$; in
particular $\theta\to Y\to T\to V_t$ is a Markov chain.

For (i), condition on the coin. On $\{\xi=0\}$ the pair $(\theta,V_t)$
equals $(\theta,T)$, whose law is $P_{\theta T}$. On $\{\xi=1\}$ it equals
$(\theta,T')$, and $T'$ is independent of $\theta$ with law $P_T$, so its
law is $\pi\otimes P_T=Q_{\theta T}$. Averaging over the coin,
\[
 P_{\theta V_t}=(1-t)P_{\theta T}+tQ_{\theta T}.
\]
Both components have $\theta$-marginal $\pi$ and quantizer marginal $P_T$, so
$V_t\sim P_T$.

For (ii), the Markov chain $Y\to T\to V_t$ gives $I(Y;V_t)\le I(T;V_t)$,
and $T$ being a function of $Y$ gives $I(T;V_t)\le I(Y;V_t)$; hence
$I(Y;V_t)=I(T;V_t)$. All entropies below are finite because
$H(T)<\infty$ and $H(V_t)=H(T)$. Writing $I(T;V_t)=H(T)-H(T\mid V_t)$,
\[
 B_t(T)=H(T)-I(T;V_t)=H(T\mid V_t).
\]
The joint law of $(T,V_t)$ is
\[
 \mathbb P(T=c,V_t=c')=(1-t)p_c\mathbf 1\{c=c'\}+tp_cp_{c'},
\]
which is symmetric in $(c,c')$ and has both marginals equal to $P_T$.
Therefore $H(T\mid V_t)=H(T,V_t)-H(V_t)=H(T,V_t)-H(T)=H(V_t\mid T)$.
To bound $H(V_t\mid T)$ from below, condition on $T=c$: the label $V_t$
equals $c$ with probability $(1-t)+tp_c$ and equals each $c'\ne c$ with
probability $tp_{c'}$. Dropping the nonnegative diagonal term and using
$p_{c'}\le1$,
\[
 H(V_t\mid T=c)
 \ \ge\ \sum_{c'\ne c}tp_{c'}\log\frac1{tp_{c'}}
 \ \ge\ t\log\frac1t\sum_{c'\ne c}p_{c'}
 \ =\ t\log\frac1t\,(1-p_c).
\]
Averaging over $c$ with weights $p_c$,
\[
 B_t(T)=\sum_cp_cH(V_t\mid T=c)\ \ge\ t\log\frac1t\sum_cp_c(1-p_c)
 =\kappa(T)\,t\log\frac1t .
\]

For (iii), let $T$ be nonconstant, so that $0<p_c<1$ for at least one
cell and $\kappa(T)=\sum_cp_c(1-p_c)>0$. The lower bound in (ii) then gives
$B_t(T)>0$. For the upper bound, $B_t(T)<H(T)$ is the same as
$I(T;V_t)>0$, that is, $T$ and $V_t$ are not independent. If they were,
the diagonal of the joint law would satisfy
$(1-t)p_c+tp_c^2=p_c^2$, that is, $(1-t)p_c(1-p_c)=0$ for every $c$,
which forces every $p_c\in\{0,1\}$ and contradicts nonconstancy.
\end{proof}

\begin{lemma}[Posterior cost of relabeling and uniform bounds]
\label{lem:relabeling_cost}
For the channel in Lemma~\ref{lem:random_relabeling}, let
$L_t(T):=\Delta_{\KL}(V_t)-\Delta_{\KL}(T)$ be its added posterior
distortion. The following bounds separate the exact cost from the
conditions used to control it.
\begin{enumerate}[label=(\roman*),leftmargin=*,itemsep=3pt]
\item \emph{The exact cost is a relevance term plus a mixture divergence.}
\begin{equation}
\label{eq:relabel_cost}
 L_t(T)=tI(\theta;T)+\operatorname{JS}_t(Q_{\theta T},P_{\theta T}).
\end{equation}
\item \emph{One experiment-level quantity bounds the cost for every quantizer.}
\begin{equation}
\label{eq:relabel_uniform}
 0\le L_t(T)\le A(t):=tI+\operatorname{JS}_t(Q_{\theta Y},P_{\theta Y})\le tI+h_2(t).
\end{equation}
\item \emph{Common support makes this uniform cost smaller than the saved-rate scale.}
Under Assumption~\ref{ass:common_support},
\[
 A(t)=o(t\log(1/t))\qquad(t\downarrow0).
\]
\item \emph{Finite experiment-level lautum information gives a linear bound.}
\[
 L<\infty\quad\Longrightarrow\quad A(t)\le t(I+L).
\]
\item \emph{Finite quantizer-level lautum information gives a second exact identity.}
If $L_T<\infty$, then
\[
 L_t(T)=t\{I(\theta;T)+L_T\}
       -\KL((1-t)P_{\theta T}+tQ_{\theta T}\Vert P_{\theta T}).
\]
\end{enumerate}
\end{lemma}
\begin{proof}
By Lemma~\ref{lem:identity}, $\Delta_{\KL}(V_t)-\Delta_{\KL}(T)
=I(\theta;T)-I(\theta;V_t)$, and this is nonnegative by data processing
along $\theta\to T\to V_t$. The exact identity in (i) comes from
expanding $I(\theta;V_t)$ with the compensation lemma.

For (i), Lemma~\ref{lem:random_relabeling}(i) gives
$P_{\theta V_t}=\bar P:=(1-t)P_{\theta T}+tQ_{\theta T}$ with marginals
$\pi$ and $P_T$, so
\[
 I(\theta;V_t)=\KL(\bar P\Vert\pi\otimes P_{V_t})=\KL(\bar P\Vert Q_{\theta T}).
\]
Apply Lemma~\ref{lem:countable_compensation} to the two-component mixture
$\bar P$ with reference $Q_{\theta T}$. Its finite-divergence hypothesis
holds because convexity of the divergence gives
$\KL(\bar P\Vert Q_{\theta T})\le(1-t)\KL(P_{\theta T}\Vert Q_{\theta T})+t\cdot0
=(1-t)I(\theta;T)\le I<\infty$. The compensation identity reads
\[
 (1-t)\KL(P_{\theta T}\Vert Q_{\theta T})+t\KL(Q_{\theta T}\Vert Q_{\theta T})
 =\big[(1-t)\KL(P_{\theta T}\Vert\bar P)+t\KL(Q_{\theta T}\Vert\bar P)\big]
  +\KL(\bar P\Vert Q_{\theta T}).
\]
The left side is $(1-t)I(\theta;T)$. The bracket on the right is
$\operatorname{JS}_t(Q_{\theta T},P_{\theta T})$, since
$\bar P=tQ_{\theta T}+(1-t)P_{\theta T}$ is the weighted mixture with the
rare component $Q_{\theta T}$ carrying weight $t$. The last term is
$I(\theta;V_t)$. Hence
\[
 (1-t)I(\theta;T)=\operatorname{JS}_t(Q_{\theta T},P_{\theta T})+I(\theta;V_t),
\]
and therefore
\begin{align*}
 L_t(T)&=I(\theta;T)-I(\theta;V_t)\\
 &=I(\theta;T)-(1-t)I(\theta;T)+\operatorname{JS}_t(Q_{\theta T},P_{\theta T})\\
 &=tI(\theta;T)+\operatorname{JS}_t(Q_{\theta T},P_{\theta T}).
\end{align*}

For (ii), bound the two terms of~\eqref{eq:relabel_cost} separately.
First, $I(\theta;T)\le I$. Second, the map
$(\theta,y)\mapsto(\theta,g(y))$ sends $Q_{\theta Y}$ to $Q_{\theta T}$
and $P_{\theta Y}$ to $P_{\theta T}$, so the data-processing property of
Lemma~\ref{lem:rare_mixture} gives
$\operatorname{JS}_t(Q_{\theta T},P_{\theta T})\le\operatorname{JS}_t(Q_{\theta Y},P_{\theta Y})$.
Together, $L_t(T)\le tI+\operatorname{JS}_t(Q_{\theta Y},P_{\theta Y})=A(t)$,
and the entropy bound $\operatorname{JS}_t\le h_2(t)$ from the same lemma
gives the last inequality.

For (iii), Proposition~\ref{prop:support_equivalence}(i) turns the
support assumption into $Q_{\theta Y}\ll P_{\theta Y}$, so the directed
singular mass $\alpha(Q_{\theta Y}\mid P_{\theta Y})$ is zero. The second
limit in Lemma~\ref{lem:rare_mixture}, applied with the rare component
$Q_{\theta Y}$, gives $\operatorname{JS}_t(Q_{\theta Y},P_{\theta Y})=o(t\log(1/t))$.
The other term satisfies $tI/(t\log(1/t))=I/\log(1/t)\to0$. Hence
$A(t)=o(t\log(1/t))$.

For (iv), if $L=\KL(Q_{\theta Y}\Vert P_{\theta Y})<\infty$, the final
inequality of Lemma~\ref{lem:rare_mixture} with $P=Q_{\theta Y}$ and
$Q=P_{\theta Y}$ gives $\operatorname{JS}_t(Q_{\theta Y},P_{\theta Y})\le tL$,
so $A(t)\le tI+tL$.

For (v), if $L_T=\KL(Q_{\theta T}\Vert P_{\theta T})<\infty$, the final
identity of Lemma~\ref{lem:rare_mixture} with $P=Q_{\theta T}$ and
$Q=P_{\theta T}$ reads
\[
 \operatorname{JS}_t(Q_{\theta T},P_{\theta T})
 =tL_T-\KL\big(tQ_{\theta T}+(1-t)P_{\theta T}\,\big\Vert\,P_{\theta T}\big).
\]
Substituting this into~\eqref{eq:relabel_cost} gives the stated formula.
\end{proof}
\subsection{Proof of Theorem~\ref{thm:separation}}
\begin{proof}
Fix $R\in(0,R_0)$ and write $D=D^{\mathrm{ts}}(R)$. The argument has
three stages. We first fix constants that depend only on the experiment
and on $R$; then we take a deterministic mixture that is nearly optimal at
rate $R$ and relabel each of its components, which saves rate at a
$t\log(1/t)$ scale while costing distortion at a smaller scale; finally
we spend the saved rate on a fixed finite quantizer with distortion below $D$.
Fixing the constants before choosing the mixture is what makes the
argument independent of attainment.

We begin by showing $0<D<I$. Positivity holds because $R<R_0$ and
$D^{\mathrm{ts}}>0$ on $[0,R_0)$ by Proposition~\ref{prop:time_sharing}.
For the upper bound, Lemma~\ref{lem:finite_posterior_quantizers} supplies a
finite quantizer $U'$ with $I(\theta;U')>0$; such a quantizer is nonconstant, so
$H(U')>0$. Time-sharing $U'$ with the constant quantizer at weight
$\lambda:=\min\{1,R/H(U')\}$ on $U'$ has rate $\lambda H(U')\le R$ and,
by~\eqref{eq:ts_affine} and Lemma~\ref{lem:identity}, distortion
\[
 \lambda\Delta_{\KL}(U')+(1-\lambda)I=I-\lambda I(\theta;U')<I .
\]
Hence $D<I$, and we set $a:=I-D>0$.

Next, Lemma~\ref{lem:finite_posterior_quantizers} also supplies a finite quantizer
$U$ with $d_U:=\Delta_{\KL}(U)<D$. Its rate $r_U:=H(U)$ must exceed $R$:
otherwise $U$ alone would be a feasible one-component mixture at rate
$R$ with distortion below $D^{\mathrm{ts}}(R)$.

We now fix the remaining constants. Since $\psi(\delta)\to0$ as
$\delta\downarrow0$, choose $\delta\in(0,1)$ with $\psi(\delta)<a/4$,
and set
\[
 K:=\frac{a\delta}{4I}>0,\qquad
 c:=\frac{K(D-d_U)}{r_U}>0 .
\]
By Lemma~\ref{lem:relabeling_cost}(iii), $A(t)/(t\log(1/t))\to0$, so we
can choose $t\in(0,1/2]$ with
\[
 A(t)\le\frac c2\,t\log\frac1t .
\]
Finally choose $\eta$ with
\[
 0<\eta<\min\Big\{\frac a2,\ \frac c4\,t\log\frac1t\Big\}.
\]
Hence, every constant so far depends only on the experiment, $R$, and the choices
of $U$ and $\delta$.

We now take a nearly optimal mixture. By the definition of
$D^{\mathrm{ts}}(R)$ as an infimum and Lemma~\ref{lem:truncate_quantizers},
there is a finite revealed mixture of finite quantizers $(\lambda_j,T_j)_{j\le k}$
with average rate $r:=\sum_j\lambda_jH(T_j)\le R$ and average distortion
$d:=\sum_j\lambda_j\Delta_{\KL}(T_j)$ satisfying $D\le d\le D+\eta$.
By Lemma~\ref{lem:identity}, its average relevance is
\[
 \sum_j\lambda_jI(\theta;T_j)=I-d\ \ge\ I-D-\eta=a-\eta>\frac a2 .
\]
Let $v$ be the total weight of the components with
$I(\theta;T_j)\ge a/4$. Bounding the relevance of the other components by
$a/4$ and that of all components by $I$,
\[
 \frac a2<\sum_j\lambda_jI(\theta;T_j)\le\frac a4(1-v)+Iv\le\frac a4+Iv,
\]
so $v>a/(4I)$. Each component counted in $v$ has relevance at least
$a/4>\psi(\delta)$, so Lemma~\ref{lem:informative_quantizer_collision} gives
$\kappa(T_j)>\delta$ for it. Therefore
\[
 \sum_j\lambda_j\kappa(T_j)\ \ge\ \delta v\ >\ \frac{a\delta}{4I}=K .
\]

We now relabel. Apply the channel of Lemma~\ref{lem:random_relabeling}
to every component $T_j$ with the same $t$, using independent coins and
draws for different components, and keep the original selector $J$
visible. The result is the revealed mixture $S':=(J,V_{t,J})$ of the
relabeled quantizers $V_{t,j}$, which is a summary channel because all added
randomness is independent of the experiment. By~\eqref{eq:ts_affine} and
Lemma~\ref{lem:random_relabeling}(ii), its rate is
\[
 r'=\sum_j\lambda_jI(Y;V_{t,j})=\sum_j\lambda_j\{H(T_j)-B_t(T_j)\}=r-B,
\]
where the total saving satisfies
\[
 B:=\sum_j\lambda_jB_t(T_j)\ \ge\ \sum_j\lambda_j\kappa(T_j)\,t\log\frac1t\ \ge\ Kt\log\frac1t>0 .
\]
By~\eqref{eq:ts_affine} and Lemma~\ref{lem:relabeling_cost}(ii), its
distortion is
\[
 \Delta_{\KL}(S')=\sum_j\lambda_j\{\Delta_{\KL}(T_j)+L_t(T_j)\}=d+L',
 \qquad
 0\le L':=\sum_j\lambda_jL_t(T_j)\le A(t).
\]
In particular $0\le r'<r\le R<r_U$.

We now spend the saved rate. Use a fresh revealed selector, independent of
everything else, to transmit $U$ with probability
\[
 u:=\frac{R-r'}{r_U-r'}\in(0,1)
\]
and $S'$ otherwise. By~\eqref{eq:ts_affine} the final summary $S$ has rate
$(1-u)r'+ur_U=r'+u(r_U-r')=R$ exactly, and distortion
\[
 \Delta_{\KL}(S)=(1-u)(d+L')+ud_U=(d+L')-u\{(d+L')-d_U\}.
\]
Two facts control the subtracted term. First, $d\ge D>d_U$ and $L'\ge0$
give $(d+L')-d_U\ge D-d_U>0$. Second, using $R\ge r$ and $r_U-r'\le r_U$,
\[
 u=\frac{R-r'}{r_U-r'}\ \ge\ \frac{r-r'}{r_U}=\frac{B}{r_U}\ \ge\ \frac{Kt\log(1/t)}{r_U}.
\]
Therefore
\begin{align*}
 \Delta_{\KL}(S)
 &\le d+L'-\frac{Kt\log(1/t)}{r_U}(D-d_U)\\
 &\le D+\eta+A(t)-c\,t\log\frac1t\\
 &\le D+\frac c4\,t\log\frac1t+\frac c2\,t\log\frac1t-c\,t\log\frac1t\\
 &=D-\frac c4\,t\log\frac1t\ <\ D,
\end{align*}
where the second line uses $d\le D+\eta$, $L'\le A(t)$, and the definition
of $c$, and the third line uses the choices of $\eta$ and $t$.

Every component quantizer and every selector takes finitely many values, so the
alphabet of $S$ is finite, and all randomization is generated independently
of the experiment, so $S$ is a valid summary channel of rate exactly $R$.
Hence $D^\star(R)\le\Delta_{\KL}(S)<D^{\mathrm{ts}}(R)$, which
is~\eqref{eq:separation}; no attainment of either frontier was used.

For the endpoints, a summary with $I(Y;S)=0$ is independent of $Y$, hence
of $\theta$, so $\Delta_{\KL}(S)=I-I(\theta;S)=I$ and
$D^\star(0)=I=D^{\mathrm{ts}}(0)$. If $R_0<\infty$,
Proposition~\ref{prop:time_sharing} gives $D^{\mathrm{ts}}=0$ on
$[R_0,\infty)$, and $0\le D^\star\le D^{\mathrm{ts}}$ gives the same for
$D^\star$.
\end{proof}

\begin{corollary}[Pointwise exclusion of deterministic near-optimizers]
\label{cor:nearoptimal_stochastic}
For each fixed $R\in(0,R_0)$, put
$\gamma_R=D^{\mathrm{ts}}(R)-D^\star(R)>0$.
A channel with $I(Y;S)\le R$ and
$\Delta_{\KL}(S)<D^\star(R)+\gamma_R$ cannot be a finite revealed
mixture of deterministic quantizers, including a deterministic quantizer itself.
Every rate-feasible minimizing sequence eventually lies outside this class.
The gap is pointwise in $R$, not asserted uniform over rates, experiments,
or neural model classes.
\end{corollary}
\begin{proof}
Theorem~\ref{thm:separation} gives $\gamma_R>0$. A finite revealed mixture
of deterministic quantizers with average rate at most $R$ realizes a point
$(r,d)\in\operatorname{conv}(\mathcal A_{\mathrm{det}})$ with $r\le R$, so
by the definition of the time-sharing frontier
$d\ge D^{\mathrm{ts}}(R)=D^\star(R)+\gamma_R$; a single deterministic
quantizer is the one-component case. If $S_n$ are channels with
$I(Y;S_n)\le R$ and $\Delta_{\KL}(S_n)\to D^\star(R)$, then
$\Delta_{\KL}(S_n)<D^\star(R)+\gamma_R$ for all large $n$, so these $S_n$
are not deterministic mixtures.
\end{proof}

\begin{corollary}[A quantitative gap under finite lautum information]
\label{cor:lautum_gap}
If $L<\infty$, retain $a,U,\delta,K,c$ from the proof above.
For every $t\in(0,1/2]$ with $\log(1/t)\ge2(I+L)/c$,
\[
 D^\star(R)\le D^{\mathrm{ts}}(R)-\frac c2t\log(1/t).
\]
\end{corollary}
\begin{proof}
By Lemma~\ref{lem:relabeling_cost}(iv), $A(t)\le t(I+L)$, and the
condition $\log(1/t)\ge2(I+L)/c$ gives
\[
 A(t)\le t(I+L)\le\frac c2\,t\log\frac1t ,
\]
so this $t$ satisfies the requirement imposed on $t$ in the proof of
Theorem~\ref{thm:separation}. Run the construction of that proof with this
$t$ and with any $\eta\in(0,a/2)$; the restriction $\eta<a/2$ is all that
was used to obtain $v>a/(4I)$, and the final chain of inequalities becomes
\[
 \Delta_{\KL}(S)\le D+\eta+A(t)-c\,t\log\frac1t
 \le D+\eta-\frac c2\,t\log\frac1t .
\]
Since $S$ has rate exactly $R$, $D^\star(R)\le D+\eta-\frac c2t\log(1/t)$
for every such $\eta$, and letting $\eta\downarrow0$ gives the claim.
\end{proof}

\subsection{Equal initial slopes do not rule out a strict gap}
The separation theorem concerns every positive interior rate, whereas an
initial-slope comparison describes only the first-order behavior at zero.
In information-bottleneck theory, the best relevance-to-rate ratio is the
sharp constant in a mutual-information strong data-processing inequality;
it also determines the frontier's initial slope
\citep{asoodeh2020bottleneck}. For the Gaussian experiment this constant is
$\rho^2$, consistently with the classical Gaussian IB solution
\citep{chechik2005gaussianib}. Proposition~\ref{prop:gaussian_same_slope}
shows that deterministic threshold quantizers approach the same constant as
stochastic channels. Nevertheless, the deterministic time-sharing
frontier remains strictly above the stochastic frontier at every positive
rate. Thus a strict inequality between initial slopes is unnecessary for
our theorem: the Gaussian example has the same first-order information
efficiency in both classes while retaining a strict gap at each fixed
positive rate.

\begin{proposition}[Gaussian initial slopes]
\label{prop:gaussian_same_slope}
Let $(\theta,X)$ be standard jointly Gaussian with correlation
$\rho\in(0,1)$, with $X$ the observation. Define
\[
 \beta_{\mathrm{det}}=
 \sup_{g:\,0<H(g(X))<\infty}\frac{I(\theta;g(X))}{H(g(X))},
 \qquad
 s^\star=\sup_{S:\,0<I(X;S)<\infty}
                 \frac{I(\theta;S)}{I(X;S)}.
\]
Then $\beta_{\mathrm{det}}=s^\star=\rho^2$, the deterministic
supremum is not attained, and the right derivatives of
$D^{\mathrm{ts}}$ and $D^\star$ at zero both equal $-\rho^2$.
Nevertheless $D^\star(R)<D^{\mathrm{ts}}(R)$ for every $R>0$.
\end{proposition}
\begin{proof}
We first identify the stochastic frontier and read off $s^\star$ from its
tangent at the origin; then we show that threshold quantizers approach the same
ratio, which gives $\beta_{\mathrm{det}}$; then we convert the two ratios
into initial slopes; and finally we verify the hypotheses of
Theorem~\ref{thm:separation}.

Write $X=\rho\theta+\sqrt{1-\rho^2}\,W$ with $W$ standard normal and
independent of $\theta$. Rescaling the observation to $X/\rho$ changes no
mutual information and puts the model in the form of
Proposition~\ref{thm:gaussian_exact} with $d=1$, $n=1$, $\tau^2=1$, and
$\sigma^2=(1-\rho^2)/\rho^2$, so that $a:=n\tau^2/\sigma^2=\rho^2/(1-\rho^2)$.
The theorem gives
\[
 D^\star(R)=\tfrac12\log(1+ae^{-2R}),\qquad
 I=D^\star(0)=\tfrac12\log(1+a)=\tfrac12\log\frac1{1-\rho^2}.
\]
Differentiating,
\[
 \frac{d}{dR}D^\star(R)=-\frac{ae^{-2R}}{1+ae^{-2R}},\qquad
 \frac{d^2}{dR^2}D^\star(R)=\frac{2ae^{-2R}}{(1+ae^{-2R})^2}>0,
\]
so $D^\star$ is strictly convex, positive at every finite $R$, and has
slope $-a/(1+a)=-\rho^2$ at zero. Strict convexity places the graph
above its tangent at the origin,
\[
 D^\star(R)\ \ge\ I-\rho^2R\qquad\text{with strict inequality for }R>0 .
\]
For any channel with $0<I(X;S)<\infty$, Lemma~\ref{lem:identity} gives
$I(\theta;S)=I-\Delta_{\KL}(S)\le I-D^\star(I(X;S))\le\rho^2I(X;S)$, so
$s^\star\le\rho^2$. The Gaussian channels in the achievability part of
Proposition~\ref{thm:gaussian_exact} attain $\Delta_{\KL}=D^\star(R)$ at
every rate $R>0$, and for them
$I(\theta;S)/I(X;S)=\{I-D^\star(R)\}/R\to-\frac{d}{dR}D^\star(0)=\rho^2$
as $R\downarrow0$. Hence $s^\star=\rho^2$.

We now show that threshold quantizers approach the same ratio. For $b>0$ let
$T_b=\mathbf 1\{X>b\}$, $p_b=\mathbb P(X>b)$, and $\mu_b=\mathbb E[X\mid X>b]$.
The rate of $T_b$ is $H(T_b)=h_2(p_b)$. For its relevance, keep only the
cell $\{X>b\}$ in $I(\theta;T_b)=\sum_cp_c\KL(\Pi_c\Vert\pi)$:
\[
 I(\theta;T_b)\ \ge\ p_b\,\KL\big(P_{\theta\mid X>b}\,\big\Vert\,N(0,1)\big).
\]
The conditional mean of $\theta$ in this cell is
$\mathbb E[\theta\mid X>b]=\mathbb E[\mathbb E[\theta\mid X]\mid X>b]=\rho\mu_b$.
A law $P$ on $\mathbb R$ with finite mean $m$ satisfies
$\KL(P\Vert N(0,1))\ge m^2/2$: the Donsker--Varadhan inequality
$\KL(P\Vert N(0,1))\ge\int g\,dP-\log\int e^g\,dN(0,1)$ with $g(x)=sx$ gives
$sm-s^2/2$, and the choice $s=m$ gives $m^2/2$. Therefore
\[
 I(\theta;T_b)\ \ge\ \frac{p_b\rho^2\mu_b^2}{2},\qquad H(T_b)=h_2(p_b).
\]
As $b\to\infty$, the Gaussian tail estimates give $p_b\sim\phi(b)/b$ and
$\mu_b=\phi(b)/p_b\sim b$, where $\phi$ is the standard normal density,
hence $\log(1/p_b)\sim b^2/2$; also
$h_2(p_b)=p_b\log(1/p_b)+(1-p_b)\log\frac1{1-p_b}$, whose second term is
of order $p_b$ and negligible against the first, so
$h_2(p_b)\sim p_b\log(1/p_b)\sim p_bb^2/2$. Consequently
\[
 \liminf_{b\to\infty}\frac{I(\theta;T_b)}{H(T_b)}
 \ \ge\ \lim_{b\to\infty}\frac{p_b\rho^2b^2/2}{p_bb^2/2}=\rho^2 .
\]
A deterministic quantizer is a channel with $I(X;g(X))=H(g(X))$, so
$\beta_{\mathrm{det}}\le s^\star=\rho^2$, and the threshold quantizers show
$\beta_{\mathrm{det}}\ge\rho^2$. Thus $\beta_{\mathrm{det}}=s^\star=\rho^2$.
The deterministic supremum is not attained: a quantizer with
$0<H(g(X))=R<\infty$ has
$I(\theta;g(X))=I-\Delta_{\KL}(g(X))\le I-D^\star(R)<\rho^2R$ by the
strict tangent inequality at $R>0$.

Next we convert the ratios into initial slopes. For the time-sharing
frontier, every deterministic component satisfies
$I(\theta;g_j(X))\le\beta_{\mathrm{det}}H(g_j(X))$, so a mixture with
average rate at most $R$ has distortion
$I-\sum_j\lambda_jI(\theta;g_j(X))\ge I-\beta_{\mathrm{det}}R$; hence
$D^{\mathrm{ts}}(R)\ge I-\beta_{\mathrm{det}}R$ and the right derivative
at zero, which exists by convexity, is at least $-\beta_{\mathrm{det}}$.
Conversely, fix a quantizer $g$ with $0<H(g(X))=:r<\infty$ and relevance
$\iota:=I(\theta;g(X))$. Time-sharing it with the constant quantizer at weight
$R/r$ for $R\le r$ gives rate $R$ and distortion $I-(\iota/r)R$, so
$D^{\mathrm{ts}}(R)\le I-(\iota/r)R$ near zero and the right derivative is
at most $-\iota/r$. Taking the supremum over $g$ gives at most
$-\beta_{\mathrm{det}}$. The right derivative of $D^{\mathrm{ts}}$ at
zero is therefore $-\beta_{\mathrm{det}}=-\rho^2$. For the unrestricted
frontier the same two-sided argument with $s^\star$ in place of
$\beta_{\mathrm{det}}$ applies, and agrees with the direct computation
$\frac{d}{dR}D^\star(0)=-\rho^2$ above.

Finally we check the hypotheses of Theorem~\ref{thm:separation}. The
joint law of $(\theta,X)$ is a nondegenerate bivariate Gaussian and the
product law $\pi\otimes P_X$ is a standard bivariate Gaussian; both have
strictly positive densities on $\mathbb R^2$, so they are mutually
absolutely continuous, which is condition (i) of
Proposition~\ref{prop:support_equivalence}. The Bayes information
$I=\tfrac12\log(1+a)$ lies in $(0,\infty)$, and the lautum information is
also finite: for the two zero-mean Gaussians with covariances
$\Sigma_1=I_2$ and $\Sigma_0=\begin{psmallmatrix}1&\rho\\\rho&1\end{psmallmatrix}$,
\begin{align*}
 L=\KL(N(0,\Sigma_1)\Vert N(0,\Sigma_0))
 &=\tfrac12\big\{\operatorname{tr}(\Sigma_0^{-1}\Sigma_1)-2-\log\det(\Sigma_0^{-1}\Sigma_1)\big\}\\
 &=\frac{\rho^2}{1-\rho^2}-\tfrac12\log\frac1{1-\rho^2}
 =a-\tfrac12\log(1+a),
\end{align*}
using $\operatorname{tr}\Sigma_0^{-1}=2/(1-\rho^2)$ and
$\det\Sigma_0^{-1}=1/(1-\rho^2)$. Since $D^{\mathrm{ts}}(R)\ge D^\star(R)>0$
at every finite $R$, the threshold is $R_0=\infty$, and
Theorem~\ref{thm:separation} gives $D^\star(R)<D^{\mathrm{ts}}(R)$ for
every $R>0$.
\end{proof}

\section{Finite observation spaces: strict improvement and attainment}
\label{app:finite-separation}
Finite observation spaces reveal both additional conclusions and a concrete
mechanism behind the general theorem. The finite set of deterministic
partitions makes its time-sharing envelope attainable and identifies the
zero-distortion threshold exactly. A perturbation that confuses just two
cells then expresses the distortion cost per saved nat through the
directed singular mass of their posteriors. This gives a criterion that
can prove strict improvement even when common support fails, as in the
unequal-support example of Appendix~\ref{app:support-examples}.

Finite observations also permit an existence result beyond the general
separation statement: the stochastic optimum is attained with at most
$|\mathcal Y|+1$ outputs, including when the parameter space is infinite.
The final construction connects these results to the count benchmark by
specifying a feasible one-bit channel and the exhaustive deterministic
comparison. Thus the appendix supplies a support-sensitive extension, an
attainment result, and an explicit example, while also giving an elementary
finite-cell explanation of the separation mechanism.

In this section, we assume that $Y$ has finite support with $m$ positive-probability atoms.
The parameter space may be any standard Borel space. Every deterministic
quantizer is a partition of the observation support, so there are finitely many
distinct deterministic rate--distortion pairs and their convex hull is compact.

\begin{lemma}[Finite-experiment envelope]
\label{lem:finite_envelope}
Every $D^{\mathrm{ts}}(R)$ is attained by a finite mixture. Its
zero-distortion threshold is $R_0=H(\Pi_Y)$, which is positive when $I>0$.
At each $0<R<R_0$, every optimal mixture has positive distortion and
uses rate exactly $R$.
\end{lemma}
\begin{proof}
Two deterministic quantizers that induce the same partition of the $m$ atoms
have the same entropy and the same cell posteriors, hence the same
rate--distortion pair. So $\mathcal A_{\mathrm{det}}$ is a finite set and
$\operatorname{conv}(\mathcal A_{\mathrm{det}})$ is a compact convex
polygon in $\mathbb R^2$.

For attainment, fix $R\ge0$. The feasible set
$\{(r,d)\in\operatorname{conv}(\mathcal A_{\mathrm{det}}):r\le R\}$ is a
closed subset of a compact set, and it is nonempty because it contains
the constant quantizer's point $(0,I)$. The continuous coordinate
$(r,d)\mapsto d$ therefore attains its minimum $D^{\mathrm{ts}}(R)$ at
some feasible point, and every point of the convex hull is a finite
convex combination of points of $\mathcal A_{\mathrm{det}}$, which by
Proposition~\ref{prop:time_sharing} is realized by a finite revealed
mixture.

For the threshold, consider the posterior label quantizer $g(y):=\Pi_y$,
which takes at most $m$ values. Since $g(Y)$ determines $\Pi_Y$, the
barycenter formula gives $\Pi_{g(Y)}=\mathbb E[\Pi_Y\mid g(Y)]=\Pi_Y$, so
$\Delta_{\KL}(g(Y))=0$, and its rate is $H(g(Y))=H(\Pi_Y)$. Thus
$D^{\mathrm{ts}}(H(\Pi_Y))=0$ and $R_0\le H(\Pi_Y)$. Conversely, suppose
$D^{\mathrm{ts}}(R)=0$ for some $R$, and let $(\lambda_j,T_j)$ be an
optimal mixture at rate $R$, which exists by the first part. Since
$\sum_j\lambda_j\Delta_{\KL}(T_j)=0$ with all $\lambda_j>0$, every
component has $\Delta_{\KL}(T_j)=0$, and
Lemma~\ref{lem:zero_distortion_cost} gives $H(T_j)\ge H(\Pi_Y)$ for each
$j$. The average rate is then at least $H(\Pi_Y)$, so feasibility forces
$R\ge H(\Pi_Y)$. Hence $R_0\ge H(\Pi_Y)$, and the two bounds give
$R_0=H(\Pi_Y)$.

The threshold is positive when $I>0$: $H(\Pi_Y)=0$ means that $\Pi_Y$
is almost surely constant, and since $\mathbb E\,\Pi_Y=\pi$ this constant
is $\pi$, which gives $I=\mathbb E\,\KL(\Pi_Y\Vert\pi)=0$.

Finally fix $0<R<R_0$. Every optimal mixture has distortion
$D^{\mathrm{ts}}(R)>0$ by Proposition~\ref{prop:time_sharing}. Suppose an
optimal mixture had average rate $r<R$. Mix it, through a fresh revealed
selector, with the posterior label quantizer, which has rate $R_0$ and
distortion zero, giving the label quantizer weight $u\in(0,1)$. The new
average rate is $(1-u)r+uR_0$, which is at most $R$ when
$u\le(R-r)/(R_0-r)$, a positive number because $r<R<R_0$. The new
distortion is $(1-u)D^{\mathrm{ts}}(R)<D^{\mathrm{ts}}(R)$. This feasible
mixture is strictly better, contradicting optimality; so every optimal
mixture uses rate exactly $R$.
\end{proof}

\begin{lemma}[Rare confusion between two cells]
\label{lem:rare_confusion}
Choose two cells $a\ne b$ of a deterministic quantizer $T$, with masses
$p_a,p_b>0$ and posteriors $\Pi_a,\Pi_b$. Change label $a$ to $b$ with
probability $t\in(0,1)$, using a hidden independent coin, and leave all
other labels unchanged. Put
$w_t=p_at/(p_b+p_at)$ and $M_t=w_t\Pi_a+(1-w_t)\Pi_b$.
The resulting summary $V_t$ has rate saving and added distortion
\begin{align}
 B_t&=H(T)-I(Y;V_t)=(p_b+p_at)h_2(w_t)>0,\\
 L_t&=\Delta_{\KL}(V_t)-\Delta_{\KL}(T)
     =p_at\KL(\Pi_a\Vert M_t)+p_b\KL(\Pi_b\Vert M_t).
\end{align}
Moreover $B_t\sim p_at\log(1/t)$ and
$L_t/B_t\to\alpha(\Pi_a\mid\Pi_b)$. If $\Pi_a\ll\Pi_b$, then $L_t=o(B_t)$;
if $\KL(\Pi_a\Vert\Pi_b)<\infty$, then
$L_t\le p_at\KL(\Pi_a\Vert\Pi_b)$.
\end{lemma}
\begin{proof}
The channel from $T$ to $V_t$ uses only $T$ and an independent coin, so
$\theta\to Y\to T\to V_t$ is a Markov chain, and as in
Lemma~\ref{lem:random_relabeling} we have $I(Y;V_t)=I(T;V_t)$. Both the
rate saving and the added distortion are conditional quantities given
$V_t$, and only one output value of $V_t$ leaves anything uncertain.

For the rate saving, $B_t=H(T)-I(T;V_t)=H(T\mid V_t)$. If $V_t=c$ with
$c\ne b$, then $T=c$ with certainty, because a label is changed only into
$b$. If $V_t=b$, then either $T=b$, or $T=a$ and the coin came up. The
output $b$ has probability $p_b+p_at$, and conditionally on it,
\[
 \mathbb P(T=a\mid V_t=b)=\frac{p_at}{p_b+p_at}=w_t,\qquad
 \mathbb P(T=b\mid V_t=b)=1-w_t .
\]
Hence $H(T\mid V_t)=(p_b+p_at)h_2(w_t)$, which is positive because
$w_t\in(0,1)$.

For the added distortion, Lemma~\ref{lem:identity} and the chain rule
give, as in Lemma~\ref{lem:truncate_quantizers},
\[
 L_t=I(\theta;T)-I(\theta;V_t)=I(\theta;T\mid V_t)
 =\sum_c\mathbb P(V_t=c)\,I(\theta;T\mid V_t=c).
\]
For $c\ne b$ the label $T$ is determined given $V_t=c$, and the
corresponding term vanishes. Given $V_t=b$, the label is $a$ with
probability $w_t$ and $b$ with probability $1-w_t$, and the conditional
law of $\theta$ given $(T,V_t)$ is $\Pi_T$ because the coin is
independent of $\theta$ given $T$. So the conditional law of $\theta$
given $V_t=b$ is the mixture $M_t$, and
\[
 I(\theta;T\mid V_t=b)=w_t\KL(\Pi_a\Vert M_t)+(1-w_t)\KL(\Pi_b\Vert M_t)
 =\operatorname{JS}_{w_t}(\Pi_a,\Pi_b).
\]
Multiplying by $\mathbb P(V_t=b)=p_b+p_at$ and using
$(p_b+p_at)w_t=p_at$ and $(p_b+p_at)(1-w_t)=p_b$ gives the stated formula
for $L_t$.

For the asymptotics as $t\downarrow0$, we have $w_t\sim(p_a/p_b)t$, so
$\log(1/w_t)=\log(1/t)+O(1)\sim\log(1/t)$, and $h_2(w)\sim w\log(1/w)$
as $w\downarrow0$. Therefore
\[
 B_t=(p_b+p_at)h_2(w_t)\sim p_b\cdot\frac{p_a}{p_b}t\log\frac1t=p_at\log\frac1t .
\]
For $L_t=(p_b+p_at)\operatorname{JS}_{w_t}(\Pi_a,\Pi_b)$, the second limit
in Lemma~\ref{lem:rare_mixture} gives
$\operatorname{JS}_{w_t}(\Pi_a,\Pi_b)\sim\alpha(\Pi_a\mid\Pi_b)\,w_t\log(1/w_t)$
when the singular mass is positive, and
$\operatorname{JS}_{w_t}(\Pi_a,\Pi_b)=o(w_t\log(1/w_t))$ when it is zero.
In either case
\[
 \frac{L_t}{B_t}=\frac{\operatorname{JS}_{w_t}(\Pi_a,\Pi_b)}{h_2(w_t)}
 =\frac{\operatorname{JS}_{w_t}(\Pi_a,\Pi_b)}{w_t\log(1/w_t)}\cdot
   \frac{w_t\log(1/w_t)}{h_2(w_t)}
 \longrightarrow\alpha(\Pi_a\mid\Pi_b)\cdot1 .
\]
If $\Pi_a\ll\Pi_b$ the limit is zero, so $L_t=o(B_t)$. If
$\KL(\Pi_a\Vert\Pi_b)<\infty$, the last inequality of
Lemma~\ref{lem:rare_mixture} gives
$\operatorname{JS}_{w_t}(\Pi_a,\Pi_b)\le w_t\KL(\Pi_a\Vert\Pi_b)$, and
multiplying by $p_b+p_at$ gives $L_t\le p_at\KL(\Pi_a\Vert\Pi_b)$.
\end{proof}

\begin{proposition}[A directional support criterion]
\label{prop:finite_refined_criterion}
Fix $0<R<R_0$ and $D=D^{\mathrm{ts}}(R)$. Suppose an optimal
deterministic mixture contains a positive-weight component having two
positive-mass cells with
\[
 \alpha(\Pi_a\mid\Pi_b)<\frac{D}{R_0-R}.
\]
Then a finite stochastic channel of rate exactly $R$ has distortion
strictly below $D$. The comparison threshold satisfies
$D/(R_0-R)\le I/R_0\le1$.
\end{proposition}
\begin{proof}
The construction perturbs one component of the optimal mixture with the
two-cell confusion of Lemma~\ref{lem:rare_confusion}, and then restores
the rate by mixing in the exact posterior label. The perturbation trades
rate for distortion at the ratio $\alpha(\Pi_a\mid\Pi_b)$, while the label
quantizer buys distortion back at the ratio $D/(R_0-R)$; the hypothesis says the
first ratio is the cheaper one.

Let $(\lambda_j,T_j)$ be an optimal mixture with the stated component
$T=T_{j_0}$ of weight $v:=\lambda_{j_0}>0$. By
Lemma~\ref{lem:finite_envelope} it has rate exactly $R$ and distortion
$D>0$. Replace $T_{j_0}$ by the channel $V_t$ of
Lemma~\ref{lem:rare_confusion}, keeping the visible selector. By the
affine identities~\eqref{eq:ts_affine}, the rate and distortion become
\[
 R-\delta r,\qquad D+\delta d,\qquad
 \delta r:=vB_t>0,\quad \delta d:=vL_t\ge0,
\]
and Lemma~\ref{lem:rare_confusion} gives
$\delta d/\delta r=L_t/B_t\to\alpha(\Pi_a\mid\Pi_b)$ as $t\downarrow0$.

Now use a fresh revealed selector to transmit the exact posterior label
quantizer, of rate $R_0=H(\Pi_Y)$ and distortion zero, with probability
\[
 u:=\frac{\delta r}{R_0-R+\delta r},
\]
and the perturbed mixture otherwise. The rate returns to
\[
 (1-u)(R-\delta r)+uR_0
 =R-\delta r+u\,(R_0-R+\delta r)=R,
\]
and the distortion becomes
\[
 (1-u)(D+\delta d)=\frac{R_0-R}{R_0-R+\delta r}\,(D+\delta d).
\]
This is strictly less than $D$ exactly when
$(R_0-R)(D+\delta d)<D(R_0-R+\delta r)$, that is, when
\[
 \frac{\delta d}{\delta r}<\frac{D}{R_0-R}.
\]
Since $\delta d/\delta r\to\alpha(\Pi_a\mid\Pi_b)<D/(R_0-R)$, this holds
for all sufficiently small $t$. The resulting channel has finitely many
outputs and rate exactly $R$.

For the threshold, $D^{\mathrm{ts}}$ is convex with
$D^{\mathrm{ts}}(0)=I$ and $D^{\mathrm{ts}}(R_0)=0$, so its graph lies
below the chord between these points and
$D\le I(1-R/R_0)$, which rearranges to $D/(R_0-R)\le I/R_0$. Finally, the
posterior label is a sufficient statistic, so
$I=I(\theta;\Pi_Y)\le H(\Pi_Y)=R_0$, giving $I/R_0\le1$.
\end{proof}
\begin{corollary}[Finite observations]
\label{cor:finite_separation}
If $Y$ has finite support with positive probabilities, $I(\theta;Y)>0$, and the posteriors $\Pi_y$ are pairwise equivalent, then $R_0=H(\Pi_Y)$ and \eqref{eq:separation} holds for every $0<R<H(\Pi_Y)$.
\end{corollary}
\begin{proof}
We verify the hypotheses of Theorem~\ref{thm:separation}, using the
finite structure to identify $R_0$.

For Corollary~\ref{cor:finite_separation}, the prior is the finite mixture
$\pi=\sum_y\mathbb P(Y=y)\Pi_y$ with positive weights. Each $\Pi_y\ll\pi$
because $\mathbb P(Y=y)\Pi_y\le\pi$. Conversely, if $\Pi_y(A)=0$ for some
$y$, pairwise equivalence gives $\Pi_{y'}(A)=0$ for every $y'$, hence
$\pi(A)=0$. Thus $\Pi_y\sim\pi$ for every $y$, which is condition (ii) of
Proposition~\ref{prop:support_equivalence}, so
Assumption~\ref{ass:common_support} holds. The Bayes information satisfies
$0<I\le H(Y)<\infty$, the spaces are standard Borel, and
Theorem~\ref{thm:separation} applies on $(0,R_0)$, where
$R_0=H(\Pi_Y)$ by Lemma~\ref{lem:finite_envelope}. Alternatively, one can
use Proposition~\ref{prop:finite_refined_criterion}: every cell posterior
of every deterministic quantizer is a positive mixture of observation
posteriors, hence equivalent to $\pi$, so any two cells have
$\alpha(\Pi_a\mid\Pi_b)=0<D/(R_0-R)$; and an optimal mixture at
$0<R<R_0$ has rate exactly $R>0$, so it contains a nonconstant component,
which has at least two cells.
\end{proof}

\begin{corollary}[Beta--Binomial experiments]
\label{cor:beta_binomial}
For $\theta\sim\operatorname{Beta}(\alpha,\beta)$ and $C\mid\theta\sim\operatorname{Binomial}(n,\theta)$ with $\alpha,\beta>0$ and $n\ge1$, separation holds throughout $0<R<H(C)$. It also holds when the observation is the full ordered Bernoulli sample.
\end{corollary}
\begin{proof}
For Corollary~\ref{cor:beta_binomial}, every count $c\in\{0,\ldots,n\}$
has positive predictive probability, and the posterior given $C=c$ is
$\operatorname{Beta}(\alpha+c,\beta+n-c)$. These $n+1$ laws have strictly
positive densities on $(0,1)$, so they are pairwise equivalent, and they
are distinct because their parameters differ. Distinctness makes
$c\mapsto\Pi_c$ injective, so $H(\Pi_C)=H(C)$, which is positive because
$C$ is nonconstant for $n\ge1$; and $I>0$ because no posterior equals the
prior. Corollary~\ref{cor:finite_separation} gives separation on
$0<R<H(C)$. For the ordered Bernoulli sample $Y\in\{0,1\}^n$, all $2^n$
outcomes have positive probability and the posterior given $Y=y$ is
$\operatorname{Beta}(\alpha+\sum_iy_i,\beta+n-\sum_iy_i)$, which depends
on $y$ only through the count. These posteriors are again pairwise
equivalent, and $H(\Pi_Y)=H(C)$. Corollary~\ref{cor:finite_separation}
applies directly to this finite observation space with the same
threshold. This is a separate application, not a transfer of a
deterministic restriction through a sufficient statistic.
\end{proof}

\begin{proposition}[Finite-alphabet attainment with a general parameter]
\label{prop:finite_attainment}
If $Y$ has $m$ positive-probability atoms, $D^\star(R)$ is attained,
for every $R\ge0$, by a channel with at most $m+1$ outputs, even when
$\theta$ is not finite-valued.
\end{proposition}
\begin{proof}
The proof rewrites the optimization over channels as an optimization
over probability laws on the simplex, shows that this second problem has
a compact feasible set and a continuous objective, and then reduces the
support of an optimizer by Carath\'eodory's theorem.

Write $p_i>0$ for the observation probabilities and $\Pi_i$ for the
posteriors, $i=1,\ldots,m$, so that $\pi=\sum_ip_i\Pi_i$. For a summary
$S$, define the conditional observation-probability vector
$w(S):=(\mathbb P(Y=i\mid S))_{i=1}^m$, which takes values in the compact
simplex $\mathcal D_m$, and let $\nu$ be its law. Taking expectations,
the barycenter of $\nu$ is $\int w\,\nu(dw)=(\mathbb P(Y=i))_i=p$.
Conversely, let $\nu$ be any probability law on $\mathcal D_m$ with
barycenter $p$, and define the channel
\[
 q(dw\mid Y=i):=\frac{w_i}{p_i}\,\nu(dw).
\]
This is a probability kernel because $\int(w_i/p_i)\nu(dw)=p_i/p_i=1$.
Under it, $\mathbb P(Y=i,S\in dw)=p_i(w_i/p_i)\nu(dw)=w_i\nu(dw)$, so $S$
has law $\nu$ and $\mathbb P(Y=i\mid S=w)=w_i$; that is, $w(S)=S$ and the
law of $w(S)$ is $\nu$. The laws on the simplex with barycenter $p$ are
therefore exactly the laws of $w(S)$ over all summaries.

We express rate and distortion through $\nu$. With
$h(w):=-\sum_iw_i\log w_i$,
\[
 I(Y;S)=H(Y)-H(Y\mid S)=H(p)-\mathbb E\,h(w(S))=H(p)-\int h\,d\nu .
\]
By the barycenter formula, $\Pi_S=\sum_i\mathbb P(Y=i\mid S)\Pi_i=\sum_iw_i(S)\Pi_i$,
so with $F(w):=\KL(\sum_iw_i\Pi_i\Vert\pi)$,
\[
 I(\theta;S)=\mathbb E\,\KL(\Pi_S\Vert\pi)=\int F\,d\nu,
 \qquad
 \Delta_{\KL}(S)=I-\int F\,d\nu .
\]
Since a summary's distortion depends on $S$ only through the joint law of
$(Y,S)$, and any $\nu$ is realized by the explicit channel above,
\[
 D^\star(R)=I-\sup\Big\{\int F\,d\nu:\ \nu\in\mathcal P(\mathcal D_m),\
 \int w\,\nu(dw)=p,\ \int h\,d\nu\ge H(p)-R\Big\}.
\]

Both integrands are continuous on $\mathcal D_m$. The entropy $h$ is
continuous with the convention $0\log0=0$. For $F$, note that
$p_i\Pi_i\le\pi$ gives $d\Pi_i/d\pi\le1/p_i$, so the density
$f_w:=d(\sum_iw_i\Pi_i)/d\pi=\sum_iw_i\,d\Pi_i/d\pi$ takes values in
$[0,1/\min_ip_i]$ and depends linearly, hence continuously, on $w$ at
every point of $\Theta$. The function $u\mapsto u\log u$ is bounded on
this interval, so dominated convergence under $\pi$ gives
$F(w)=\int f_w\log f_w\,d\pi\to F(w')$ as $w\to w'$.

The feasible set is nonempty and compact. It contains $\delta_p$, the law
of the constant summary, since $\int h\,d\delta_p=H(p)\ge H(p)-R$. The
space $\mathcal P(\mathcal D_m)$ is compact in the weak topology because
$\mathcal D_m$ is a compact metric space. The barycenter constraints
$\int w_i\,d\nu=p_i$ and the rate constraint $\int h\,d\nu\ge H(p)-R$ are
closed, since $w_i$ and $h$ are continuous and bounded. The objective
$\nu\mapsto\int F\,d\nu$ is weakly continuous for the same reason.
Therefore the supremum is attained by some $\nu^\star$.

We reduce the support. Consider the continuous map
\[
 \Phi(w):=\big(w_1,\ldots,w_{m-1},h(w),F(w)\big)\in\mathbb R^{m+1}.
\]
The moment vector $\int\Phi\,d\nu^\star$ lies in the convex hull of the
compact set $\Phi(\mathcal D_m)$. By Carath\'eodory's theorem in
$\mathbb R^{m+1}$, it is a convex combination of at most $m+2$ points
$\Phi(w^{(1)}),\ldots,\Phi(w^{(m+2)})$ with weights $\mu_k$. The discrete
law $\nu':=\sum_k\mu_k\delta_{w^{(k)}}$ has the same first $m-1$
barycenter coordinates as $\nu^\star$, hence the same barycenter since
$w_m=1-\sum_{i<m}w_i$, the same entropy moment, and the same $F$-moment.
So $\nu'$ is feasible and optimal.

Suppose $\nu'$ has exactly $m+2$ atoms with positive weights. The linear
map $\varepsilon\mapsto\big(\sum_k\varepsilon_k,\ \sum_k\varepsilon_kw^{(k)}_i\ (i<m),\
\sum_k\varepsilon_kh(w^{(k)})\big)$ from $\mathbb R^{m+2}$ to
$\mathbb R^{m+1}$ has a nonzero kernel element $\varepsilon$. For
$|s|$ small, the weights $\mu+s\varepsilon$ are positive, sum to one, and
give a law with the same barycenter and the same entropy moment, hence a
feasible law, whose $F$-moment is
$\int F\,d\nu'+s\sum_k\varepsilon_kF(w^{(k)})$. Optimality of $\nu'$ for
both signs of $s$ forces $\sum_k\varepsilon_kF(w^{(k)})=0$, so the
$F$-moment is constant along the whole segment. Since
$\sum_k\varepsilon_k=0$ and $\varepsilon\ne0$, some coordinate of
$\varepsilon$ is negative; increase $s$ until the first weight reaches
zero. The resulting law is still feasible and optimal and has at most
$m+1$ atoms.

Finally, the explicit channel $q(dw\mid Y=i)=(w_i/p_i)\nu(dw)$ built from
this finitely supported optimal $\nu$ has at most $m+1$ outputs, rate at
most $R$, and distortion $I-\int F\,d\nu=D^\star(R)$.
\end{proof}

Under the separation hypotheses, a finite-observation optimum at an
interior rate cannot be deterministic or a revealed deterministic mixture.
The attainment proposition does not supply a global numerical optimizer;
it establishes existence and a cardinality bound.

\paragraph{The explicit count witness and enumeration.}
For $\alpha=\beta=5$, $n=8$, take the observation to be $C$, write
$w_c=P(C=c)$ and $\Pi_c=P_{\theta\mid C=c}=\operatorname{Beta}(5+c,13-c)$.
For any finite channel $q(b\mid c)$, its induced posterior is
\[
 \Pi_b=\frac{\sum_cw_cq(b\mid c)\Pi_c}{\sum_cw_cq(b\mid c)},
\]
on positive-probability outputs, with
\[
 R=\sum_{c,b}w_cq(b\mid c)\log\frac{q(b\mid c)}{\sum_jw_jq(b\mid j)},
 \qquad
 \Delta_{\KL}=\sum_{c,b}w_cq(b\mid c)\KL(\Pi_c\Vert\Pi_b).
\]
All $\operatorname{Bell}(9)=21{,}147$ count partitions are enumerated;
posterior mixtures are integrated by Gauss--Legendre quadrature and the
deterministic envelope is their lower convex hull.
The displayed quadrature values are numerical approximations,
not certified integration enclosures.

Put $w_A=P(C\le3)$ and define $\varrho_\star$ as the unique solution in
$(0,1)$ of
\[
 h_2(w_A\varrho)-w_Ah_2(\varrho)=0.5.
\]
The left side is the information rate of
$q(1\mid c)=\varrho\mathbf1\{c\le3\}$: the output $B$ has
$\mathbb P(B=1)=w_A\varrho$, so $H(B)=h_2(w_A\varrho)$, while
$H(B\mid C)=w_Ah_2(\varrho)$ because $B$ is random only on the counts
$c\le3$. It is strictly increasing in $\varrho$: its derivative is
$w_A\log\{(1-w_A\varrho)/(w_A(1-\varrho))\}$, which is positive because
$1-w_A\varrho>w_A(1-\varrho)$ reduces to $w_A<1$.
The solution is approximately $0.8782923057$.
Its distortion is approximately $0.154529$, below the time-sharing
envelope $0.162253$ and the best individual deterministic count quantizer
$0.187462$ at rate $0.5$.
The stochastic point is feasible, not claimed globally optimal.
Neither the nine-count enumeration nor its numerical envelope enumerates
all deterministic functions of the ordered raw sample.

\section{Examples delimiting the support and rate assumptions}
\label{app:support-examples}

\paragraph{A deterministic target.}
Suppose $\theta=f(Y)$ is discrete with $H(\theta)<\infty$.
For every summary,
$\Delta_{\KL}(S)=H(\theta)-I(\theta;S)
\ge H(\theta)-I(Y;S)$.
Time-sharing the constant summary and $f(Y)$ attains the bound up to
rate $H(\theta)$; transmitting $f(Y)$ handles larger budgets, proving
\eqref{eq:det_target}. Not every deterministic quantizer is efficient:
if $Y=(\theta,U)$ with independent fair bits, the quantizer
$T=U$ has positive rate and zero relevance, lying above the frontier.
Channels driven by $\theta$ alone satisfy
$I(Y;S)=I(\theta;S)$ and attain the line at their achieved rates.

\paragraph{Overlapping posteriors without an advantage.}
Let $Y$ be a fair bit and independently draw $E\sim\operatorname{Bernoulli}(\alpha)$,
where $0<\alpha<1$. Set $\theta=Y$ when $E=1$ and $\theta=e$
otherwise, with $e$ a distinct erasure symbol. For every channel from $Y$,
the indicator $E$ is independent of $(Y,S)$ and is determined by $\theta$.
The chain rule therefore gives
$I(\theta;S)=\alpha I(Y;S)$ and $I(\theta;Y)=\alpha\log2$.
It follows that
\[
 D^\star(R)=D^{\mathrm{ts}}(R)=\alpha[\log2-R]_+,
\]
attained by visible time-sharing of the constant and identity quantizers.
The target is not a function of $Y$, and the posteriors overlap at $e$
with mass $1-\alpha$: neither property ensures strict separation.

\paragraph{Strict advantage despite unequal supports.}
Let $Y$ be a fair bit with posteriors
$\Pi_0=(0.1,0.9,0)$ and $\Pi_1=(0,0.1,0.9)$.
There are only two deterministic partition types, so
$R_0=\log2$ and
$D^{\mathrm{ts}}(R)=I(1-R/\log2)$ on $[0,\log2]$.
The prior is $(0.05,0.5,0.45)$; only its middle atom leaves $Y$
uncertain, with conditional binary entropy $h_2(0.1)$. Hence
$I=\log2-\tfrac12h_2(0.1)$ and
$I/\log2\ge1/2>0.1=\alpha(\Pi_0\mid\Pi_1)$.
At every interior rate, the optimal mixture uses the identity quantizer with
positive weight. Proposition~\ref{prop:finite_refined_criterion} therefore
gives strict separation at every such rate, despite failure of equivalence.

\paragraph{Topological support is insufficient.}
For a probability law on a metric observation space, its topological
support is
\[
 \operatorname{supp}_{\mathrm{top}}(P)
 :=\{y:P(U)>0\text{ for every open neighborhood }U\text{ of }y\}.
\]
This definition tests neighborhoods. Mutual absolute continuity instead
requires agreement on all measurable null events:
$P(A)=0$ if and only if $Q(A)=0$ for every measurable $A$.
Equal topological supports need not ensure this stronger property.
Let $\theta$ be a fair bit. Given $\theta=0$, let $Y$ be uniform on
$[0,1]$; given $\theta=1$, give positive masses to an enumeration of
the rationals in $[0,1]$. Both sampling laws have topological support
$[0,1]$, but they are mutually singular and
$\theta=\mathbf1\{Y\in\mathbb Q\}$ almost surely.
Thus \eqref{eq:det_target} gives no advantage despite equal topological
supports, unlike the measure equivalence in Assumption~\ref{ass:common_support}.
The separation proof uses the latter property to make the directed
singular mass $\alpha(Q_{\theta Y}\mid P_{\theta Y})$ vanish for its joint and product laws.
No choice of topology is needed for that argument. Common support in this
measure-theoretic sense is sufficient for the theorem, but the preceding
unequal-support example shows that it is not necessary for separation.

\paragraph{An output-entropy constraint is a different problem.}
For $\theta=Y$ a fair bit and $0<R<\log2$, every deterministic quantizer with entropy at most $R$
is constant. Choose $t\in(0,1)$ with $h_2(t/2)=R$ and use
$q(B=1\mid Y=1)=t$, $q(B=1\mid Y=0)=0$.
Then $H(B)=R$ but
\[
 I(\theta;B)=h_2(t/2)-\tfrac12h_2(t)>0.
\]
This improves on every feasible deterministic quantizer under $H(S)\le R$;
the cardinality result and hard-assignment claims for an entropy-penalized
Lagrangian do not extend to this constraint. Relabeling preserves output
entropy and saves no cost under that measure.

\paragraph{Infinite mutual information need not mean a noiseless component.}
The finite-information consequences in Section~\ref{sec:preliminaries}
do not characterize which noisy models satisfy that assumption, nor do
they imply finite exact-preservation rate (Proposition~\ref{prop:shape}).
For a counterexample, let $\theta=N$ have a countable
full-support distribution $p$ with $H(N)=\infty$, for example a
normalized $p_n\propto1/[n(\log n)^2]$ on $n\ge2$.
Let $Y=N+G$ with $G\sim N(0,1)$ independent. All sampling laws have
strictly positive densities on $\mathbb R$, so they are equivalent,
and no nonconstant function of $N$ is determined by $Y$.
For $Z=\operatorname{round}(Y)$ and
$E=\operatorname{round}(G)$, translation by the integer $N$ gives
$Z=N+E$ almost surely. The discretized Gaussian $E$ has finite entropy,
and hence
\[
 H(N\mid Y)\le H(N\mid Z)=H(E\mid Z)\le H(E)<\infty.
\]
For finite quantizations $U_m$ of $N$ with $H(U_m)\to\infty$,
$I(N;Y)\ge I(U_m;Y)\ge H(U_m)-H(E)$, proving
$I(N;Y)=\infty$ despite additive Gaussian observation noise.
Thus an exactly observed continuous parameter coordinate is only one
possible cause of infinite mutual information.

\section{Proofs for Section \ref{sec:preliminaries}}
\label{app:proofs-prelim}

\begin{proof}[Proof of Lemma \ref{lem:identity}]
We prove the three parts of \eqref{eq:identity} in turn, then the
compensation identity \eqref{eq:compensation} by applying the chain rule
for divergences in two different orders, and finally the two converses.

Since $\theta\to Y\to S$ is a Markov chain, $\theta$ is conditionally
independent of $S$ given $Y$, so the conditional law of $\theta$ given
$(Y,S)$ is the conditional law given $Y$ alone:
$P_{\theta\mid Y,S}=\Pi_Y$ almost surely. On standard Borel spaces the
conditional mutual information is the expected divergence between the
two conditional laws of $\theta$,
\[
 I(\theta;Y\mid S)=\mathbb E\,\KL(P_{\theta\mid Y,S}\Vert P_{\theta\mid S})
 =\mathbb E\,\KL(\Pi_Y\Vert\Pi_S)=\Delta_\KL(S).
\]
For the second equality in \eqref{eq:identity}, expand $I(\theta;Y,S)$
with the chain rule in both orders:
\[
 I(\theta;Y,S)=I(\theta;Y)+I(\theta;S\mid Y)=I(\theta;Y),
 \qquad
 I(\theta;Y,S)=I(\theta;S)+I(\theta;Y\mid S),
\]
where the first line uses $I(\theta;S\mid Y)=0$ from the Markov property.
Equating the two expansions gives $I(\theta;Y\mid S)=I(\theta;Y)-I(\theta;S)$,
and all terms are finite because $0\le I(\theta;S)\le I(\theta;Y)<\infty$.
For the barycenter formula, fix a measurable $A\subseteq\Theta$ and apply
the tower property together with the Markov property:
\[
 \mathbb P(\theta\in A\mid S)
 =\mathbb E\big[\mathbb P(\theta\in A\mid Y,S)\,\big|\,S\big]
 =\mathbb E\big[\Pi_Y(A)\,\big|\,S\big].
\]
Since $\Theta$ is standard Borel, a countable generating algebra of sets
$A$ suffices to identify the two kernels, giving $\Pi_S=\mathbb E[\Pi_Y\mid S]$
almost surely.

For \eqref{eq:compensation}, let $Q_S$ be a measurable reconstruction with
$\mathbb E\,\KL(\Pi_Y\Vert Q_S)<\infty$, and fix a value $s$ of $S$. On
$\Theta\times\mathcal Y$ compare the conditional joint law
$P_{\theta Y\mid S=s}$ with the product $Q_s\otimes P_{Y\mid S=s}$, whose
$\mathcal Y$-marginal agrees with that of the joint. The chain rule for
divergences holds in $[0,\infty]$ without integrability assumptions.
Disintegrating along $Y$ first,
\begin{align*}
 \KL\big(P_{\theta Y\mid S=s}\,\big\Vert\,Q_s\otimes P_{Y\mid S=s}\big)
 &=\KL(P_{Y\mid S=s}\Vert P_{Y\mid S=s})+\int\KL(P_{\theta\mid Y=y,S=s}\Vert Q_s)\,P_{Y\mid S=s}(dy)\\
 &=\int\KL(\Pi_y\Vert Q_s)\,P_{Y\mid S=s}(dy).
\end{align*}
Disintegrating along $\theta$ first,
\begin{align*}
 \KL\big(P_{\theta Y\mid S=s}\,\big\Vert\,Q_s\otimes P_{Y\mid S=s}\big)
 &=\KL(P_{\theta\mid S=s}\Vert Q_s)\\
 &\quad+\int\KL(P_{Y\mid\theta,S=s}\Vert P_{Y\mid S=s})\,P_{\theta\mid S=s}(d\theta).
\end{align*}
The last integral is the mutual information between $\theta$ and $Y$ under
the conditional law given $S=s$, which by symmetry of mutual information
equals $\int\KL(\Pi_y\Vert\Pi_s)\,P_{Y\mid S=s}(dy)$. Equating the two
expansions and averaging over $s$ with respect to the law of $S$,
\[
 \mathbb E\,\KL(\Pi_Y\Vert Q_S)=\mathbb E\,\KL(\Pi_Y\Vert\Pi_S)+\mathbb E\,\KL(\Pi_S\Vert Q_S)
 =\Delta_\KL(S)+\mathbb E\,\KL(\Pi_S\Vert Q_S).
\]
The left side is finite by hypothesis, so both terms on the right are
finite and no difference of infinities occurs. When all measures have
densities $p_Y,p_S,q_S$ with respect to a common reference measure, the
same identity is the expectation of the pointwise decomposition
$\log(p_Y/q_S)=\log(p_Y/p_S)+\log(p_S/q_S)$, using that $\log(p_S/q_S)$ is
$\sigma(S)$-measurable and $\mathbb E[p_Y\mid S]=p_S$; the argument above
avoids the choice of a reference measure.

For the converses, data processing along $\theta\to Y\to S$ gives
$I(\theta;S)\le I(Y;S)$, and with the identity just proved,
\[
 \Delta_\KL(S)=I(\theta;Y)-I(\theta;S)\ \ge\ I(\theta;Y)-I(Y;S).
\]
If $\Delta_\KL(S)\le\varepsilon$ then $I(Y;S)\ge I(\theta;Y)-\varepsilon$,
and rates are nonnegative, so $R^\star(\varepsilon)\ge[I(\theta;Y)-\varepsilon]_+$.
If $I(Y;S)\le R$ then $\Delta_\KL(S)\ge I(\theta;Y)-R$, and distortions are
nonnegative, so $D^\star(R)\ge[I(\theta;Y)-R]_+$.
\end{proof}

\begin{proof}[Proof of Proposition \ref{prop:reduction}]
We first verify that $Z=\Pi_Y$ is its own posterior, then prove the two
inequalities behind the reduction identity $R^\star(\varepsilon)=\inf_{P_{\widehat Z\mid Z}:\,\mathbb E\KL(Z\Vert\widehat Z)\le\varepsilon} I(Z;\widehat Z)$, and finally give the explicit
reduction of an arbitrary kernel, which also proves Remark~\ref{rem:nondet}.

For a measurable $A\subseteq\Theta$, the tower property gives
\[
 \mathbb P(\theta\in A\mid Z)=\mathbb E\big[\mathbb P(\theta\in A\mid Y)\,\big|\,Z\big]
 =\mathbb E[Z(A)\mid Z]=Z(A),
\]
since $Z(A)$ is $\sigma(Z)$-measurable and $Z$ is a function of $Y$.
Consequently the law of $Z$ determines the joint law of $(\theta,Z)$:
for measurable $A$ and $C$,
$\mathbb P(\theta\in A,Z\in C)=\mathbb E[Z(A)\mathbf 1\{Z\in C\}]$.

We first show that the right side of the reduction identity is at most the left.
Let $S$ be a summary channel with $\Delta_\KL(S)\le\varepsilon$ and set
$\widehat Z:=\Pi_S$, a probability-measure-valued function of $S$. The
joint law of $(Z,\widehat Z)$ defines a channel from $Z$ to reconstructions,
with distortion
$\mathbb E\,\KL(Z\Vert\widehat Z)=\mathbb E\,\KL(\Pi_Y\Vert\Pi_S)=\Delta_\KL(S)\le\varepsilon$.
Since $Z$ is a function of $Y$ and $\widehat Z$ a function of $S$, data
processing gives $I(Z;\widehat Z)\le I(Y;S)$. Thus every feasible summary
yields a feasible reconstruction channel of no larger rate, and taking the
infimum over summaries shows that the right side is at most $R^\star(\varepsilon)$.

Now, we show the reverse inequality. Let $P_{\widehat Z\mid Z}$ be a channel
with $\mathbb E\,\KL(Z\Vert\widehat Z)\le\varepsilon$. Build a summary from
$Y$ in two stages: compute $Z=\Pi_Y$, then draw $S$ from
$P_{\widehat Z\mid Z}(\cdot\mid Z)$ using fresh randomness. Then
$\theta\to Y\to S$ is a Markov chain, and the joint law of $(Z,S)$ equals
that of $(Z,\widehat Z)$. Because $S$ depends on $Y$ only through $Z$, and
$Z$ is a function of $Y$, the chain rule gives
\[
 I(Y;S)=I(Y,Z;S)=I(Z;S)+I(Y;S\mid Z)=I(Z;S)=I(Z;\widehat Z).
\]
For the distortion, use the value of $S$ itself as a reconstruction,
$Q_S:=S$. Its expected divergence is
$\mathbb E\,\KL(\Pi_Y\Vert Q_S)=\mathbb E\,\KL(Z\Vert\widehat Z)\le\varepsilon<\infty$,
so \eqref{eq:compensation} applies and gives
\[
 \Delta_\KL(S)=\mathbb E\,\KL(\Pi_Y\Vert Q_S)-\mathbb E\,\KL(\Pi_S\Vert Q_S)
 \le\mathbb E\,\KL(Z\Vert\widehat Z)\le\varepsilon .
\]
Thus $S$ is feasible for $R^\star(\varepsilon)$ with rate $I(Z;\widehat Z)$,
and taking the infimum over reconstruction channels gives
$R^\star(\varepsilon)\le$ the right side of the posterior source reduction.

Finally we reduce an arbitrary kernel explicitly. Let $S$ be generated
from $Y$ by a kernel $q(ds\mid y)$, define
$q'(ds\mid z):=\mathbb P(S\in ds\mid Z=z)$, and generate $S'$ from $Z$
through $q'$ with fresh randomness. For measurable $A$ and $C$, using
$\mathbb P(\theta\in A\mid Y,S)=Z(A)$ and $\mathbb P(S\in C\mid Y)=q(C\mid Y)$,
\begin{align*}
 \mathbb P(\theta\in A,S\in C)
 &=\mathbb E\big[Z(A)\,q(C\mid Y)\big]
 =\mathbb E\big[Z(A)\,\mathbb E[q(C\mid Y)\mid Z]\big]\\
 &=\mathbb E\big[Z(A)\,q'(C\mid Z)\big]
 =\mathbb P(\theta\in A,S'\in C),
\end{align*}
where the second equality is the tower property with the
$\sigma(Z)$-measurable factor $Z(A)$, the third uses
$\mathbb E[q(C\mid Y)\mid Z]=\mathbb P(S\in C\mid Z)=q'(C\mid Z)$, and the
last uses that $S'$ is conditionally independent of $\theta$ given $Y$
with $\mathbb P(S'\in C\mid Y)=q'(C\mid Z)$. Hence $(\theta,S)$ and
$(\theta,S')$ have the same joint law, so they induce the same
reconstruction $\Pi_S$ and, by \eqref{eq:identity}, the same distortion.
Their rates compare as
$I(Y;S')=I(Z;S')=I(Z;S)\le I(Y;S)$, by the same chain-rule computation
as above and data processing. This proves Remark~\ref{rem:nondet}: even if
$q$ is degenerate, so that $S=g(Y)$ is deterministic, the averaged kernel
$q'(\cdot\mid z)=\mathbb P(g(Y)\in\cdot\mid Z=z)$ is degenerate only when
$g(Y)$ is a function of $Z$.
\end{proof}

\begin{proposition}[Shape and endpoints]
\label{prop:shape}
$R^\star$ is nonincreasing and convex, $R^\star(\varepsilon)=0$ for $\varepsilon\ge I(\theta;Y)$, and $R^\star(0)=\infty$ whenever the law of $Z=\Pi_Y$ is non-atomic.
\end{proposition}

\begin{remark}[Deterministic restrictions and sufficient statistics]
\label{rem:nondet}
Replacing $q(\cdot\mid Y)$ by its conditional average given $Z$ preserves the joint law of $(\theta,S)$ and can only reduce rate. Even if $S=g(Y)$, this averaged kernel need not be deterministic. A deterministic restriction on a sufficient statistic is therefore distinct from a deterministic restriction on the raw observation. The deterministic comparison in Section~\ref{sec:randomization_essential} refers to functions of the original observation. 
\end{remark}

\begin{proof}[Proof of Proposition \ref{prop:shape}]
Monotonicity holds because the feasible set
$\{q:\Delta_\KL(S)\le\varepsilon\}$ grows with $\varepsilon$.

For convexity, fix $\varepsilon_1,\varepsilon_2\ge0$, $\mu\in(0,1)$, and
$\eta>0$, and choose channels $S_1,S_2$ with $\Delta_\KL(S_i)\le\varepsilon_i$
and $I(Y;S_i)\le R^\star(\varepsilon_i)+\eta$; if some
$R^\star(\varepsilon_i)$ is infinite there is nothing to prove. Let $S=(J,S_J)$
be the revealed mixture with $\mathbb P(J=1)=\mu$. By \eqref{eq:ts_affine},
\[
 \Delta_\KL(S)=\mu\Delta_\KL(S_1)+(1-\mu)\Delta_\KL(S_2)\le\mu\varepsilon_1+(1-\mu)\varepsilon_2,
\]
\[
 I(Y;S)=\mu I(Y;S_1)+(1-\mu)I(Y;S_2)\le\mu R^\star(\varepsilon_1)+(1-\mu)R^\star(\varepsilon_2)+\eta .
\]
Hence $R^\star(\mu\varepsilon_1+(1-\mu)\varepsilon_2)\le\mu R^\star(\varepsilon_1)+(1-\mu)R^\star(\varepsilon_2)+\eta$,
and $\eta\downarrow0$ gives convexity.

For the zero-rate endpoint, the constant summary has $I(Y;S)=0$ and, by
\eqref{eq:identity}, distortion $I(\theta;Y)-I(\theta;S)=I(\theta;Y)$. It is
feasible whenever $\varepsilon\ge I(\theta;Y)$, so $R^\star(\varepsilon)=0$ there.

For exact preservation, let $\Delta_\KL(S)=0$. Lemma~\ref{lem:zero_distortion_cost}
gives $Z=\Pi_S$ almost surely, so $Z$ coincides with a measurable function
of $S$. Data processing then gives
\[
 I(Y;S)\ \ge\ I(Z;S)\ \ge\ I(Z;\Pi_S)\ =\ I(Z;Z).
\]
The space of probability measures on $\Theta$ is standard Borel, so
Lemma~\ref{lem:deterministic_atomic} applies to the deterministic map
$Z\mapsto Z$: when the law of $Z$ is non-atomic, $I(Z;Z)=\infty$. Hence
every summary with zero distortion has infinite rate, and $R^\star(0)=\infty$.
\end{proof}

\section{Proofs for Section \ref{sec:posterior_geometry}}
\label{app:proofs-geometry}

\begin{proof}[Proof of Theorem \ref{thm:avg_quant}]
Write
\[
 D^{\mathrm{card}}_M:=\inf_{g:\mathcal Y\to[M]}\Delta_\KL(g(Y)),
 \qquad
 D^{\mathrm{stoch}}_M:=\inf_{q(s\mid y):\,|\mathcal S|\le M}\Delta_\KL(S)
\]
for the deterministic and stochastic $M$-symbol infima. We show
$D^{\mathrm{av}}_M\le D^{\mathrm{stoch}}_M\le D^{\mathrm{card}}_M\le D^{\mathrm{av}}_M$,
then evaluate the case $M=1$, and finally derive the rate bound.

We first show $D^{\mathrm{av}}_M\le D^{\mathrm{stoch}}_M$. Let $S$ be any
summary with values in $[M]$, deterministic or not. For each output $j$
with $\mathbb P(S=j)>0$ put $Q_j:=\Pi_{S=j}=P_{\theta\mid S=j}$, and choose
the remaining $Q_j$ arbitrarily. Pointwise, the nearest center is at least
as close as the center actually selected by $S$:
\[
 \min_{j\le M}\KL(\Pi_Y\Vert Q_j)\ \le\ \KL(\Pi_Y\Vert Q_S)=\KL(\Pi_Y\Vert\Pi_S).
\]
Taking expectations gives $\mathbb E\min_j\KL(\Pi_Y\Vert Q_j)\le\Delta_\KL(S)$,
so $D^{\mathrm{av}}_M\le\Delta_\KL(S)$; the infimum over $S$ gives the claim.
The inequality $D^{\mathrm{stoch}}_M\le D^{\mathrm{card}}_M$ holds because a
deterministic $g:\mathcal Y\to[M]$ is a particular $M$-symbol channel.

We now show $D^{\mathrm{card}}_M\le D^{\mathrm{av}}_M$ by the
nearest-center construction. Let $Q_1,\ldots,Q_M$ be a codebook with
finite average distortion $\mathbb E\min_j\KL(\Pi_Y\Vert Q_j)<\infty$; such
codebooks exist because $Q_1=\cdots=Q_M=\pi$ gives $\mathbb E\,\KL(\Pi_Y\Vert\pi)=I(\theta;Y)<\infty$.
Define the nearest-center quantizer with the smallest index breaking ties,
\[
 g(y):=\min\argmin_{j\le M}\KL(\Pi_y\Vert Q_j).
\]
It is measurable because each $y\mapsto\KL(\Pi_y\Vert Q_j)$ is measurable
by the standing assumption and
$\{g=j\}=\bigcap_{i<j}\{\KL(\Pi_y\Vert Q_j)<\KL(\Pi_y\Vert Q_i)\}\cap\bigcap_{i>j}\{\KL(\Pi_y\Vert Q_j)\le\KL(\Pi_y\Vert Q_i)\}$.
Put $S:=g(Y)$ and use $Q_S$ as a reconstruction. By construction
$\KL(\Pi_Y\Vert Q_S)=\min_j\KL(\Pi_Y\Vert Q_j)$, so
$\mathbb E\,\KL(\Pi_Y\Vert Q_S)$ is finite and the compensation identity
\eqref{eq:compensation} applies:
\[
 \mathbb E\min_j\KL(\Pi_Y\Vert Q_j)
 =\mathbb E\,\KL(\Pi_Y\Vert Q_S)
 =\Delta_\KL(S)+\mathbb E\,\KL(\Pi_S\Vert Q_S)
 \ \ge\ \Delta_\KL(S).
\]
Conditioning on each positive-mass cell, this is the statement that
replacing the center $Q_j$ by the cell barycenter $\Pi_{g=j}$ can only
reduce the cell's expected divergence. Hence
$D^{\mathrm{card}}_M\le\Delta_\KL(g(Y))\le\mathbb E\min_j\KL(\Pi_Y\Vert Q_j)$
for every codebook, and taking the infimum over codebooks gives
$D^{\mathrm{card}}_M\le D^{\mathrm{av}}_M$. The three infima are therefore
equal, which proves the result.

For $M=1$ the only deterministic quantizer is the constant, whose distortion is
$I(\theta;Y\mid\text{constant})=I(\theta;Y)$, so $D^{\mathrm{av}}_1=I(\theta;Y)$.
The same value is visible from the compensation identity with a constant
$S$ and a single center $Q$, which reads
$\mathbb E\,\KL(\Pi_Y\Vert Q)=I(\theta;Y)+\KL(\pi\Vert Q)$ and is minimized
at $Q=\pi$.

For the rate bound, suppose $D^{\mathrm{av}}_M<\varepsilon$. By the
definition of the infimum there is a codebook with
$\mathbb E\min_j\KL(\Pi_Y\Vert Q_j)<\varepsilon$, and by the display above
its nearest-center quantizer $S=g(Y)$ has $\Delta_\KL(S)<\varepsilon$. Its
rate is $I(Y;S)=H(S)\le\log M$ because $S$ is a deterministic quantizer with at
most $M$ values. Hence $R^\star(\varepsilon)\le\log M$ whenever
$D^{\mathrm{av}}_M<\varepsilon$, and taking $M=N^{\mathrm{av}}(\varepsilon)$
gives the bound. The strict inequality in the definition
of $N^{\mathrm{av}}$ is what guarantees a feasible codebook without an
attainment assumption; if only $D^{\mathrm{av}}_M\le\varepsilon$ is known,
the same argument gives $R^\star(\varepsilon+\eta)\le\log M$ for every $\eta>0$.
\end{proof}

\begin{corollary}[Uniform posterior covering]
\label{cor:unif_cover}
Let $N^{\mathrm{unif}}_\KL(\varepsilon)$ be the least number of probability-measure centers $Q_j$ with $\min_j\KL(\Pi_y\Vert Q_j)\le\varepsilon$ for $P_Y$-almost every $y$. If this number is finite, a deterministic summary achieves distortion at most $\varepsilon$ and rate at most $\log N^{\mathrm{unif}}_\KL(\varepsilon)$.
\end{corollary}

\begin{proof}[Proof of Corollary \ref{cor:unif_cover}]
Let $Q_1,\ldots,Q_N$ be a uniform cover with $N=N^{\mathrm{unif}}_\KL(\varepsilon)$,
and let $S=g(Y)$ be its nearest-center quantizer as in the proof above.
The pointwise bound $\min_j\KL(\Pi_y\Vert Q_j)\le\varepsilon$ for
$P_Y$-almost every $y$ gives $\mathbb E\min_j\KL(\Pi_Y\Vert Q_j)\le\varepsilon$,
and the compensation step gives $\Delta_\KL(S)\le\mathbb E\min_j\KL(\Pi_Y\Vert Q_j)\le\varepsilon$.
The rate is $H(S)\le\log N$. Since the codebook is given, no infimum has
to be approximated.
\end{proof}

\begin{proposition}[Unbounded Gaussian posterior means]
\label{prop:no_unif_cover}
For $\mathcal P=\{N(m,vI_d):m\in\mathbb R^d\}$, fixed $v>0$, and a full-support predictive law for $m$, every finite codebook has infinite essential-uniform forward-KL radius. With normalized squared Hellinger distance $\Hel^2(P,Q)=\tfrac12\int(\sqrt p-\sqrt q)^2$, no finite uniform cover of radius $r<1$ exists.
\end{proposition}

\begin{proof}[Proof of Proposition \ref{prop:no_unif_cover}]
Both claims rest on the same picture: a finite codebook is tight, so it
places little mass far from the origin, while the Gaussian $N(m,vI_d)$
places half its mass near $m$; for large $\Vert m\Vert$ these two regions
are disjoint. Write $P_m:=N(m,vI_d)$.

Fix centers $Q_1,\ldots,Q_M$ and $\delta\in(0,1)$. By translation
invariance, choose $A>0$ with $P_m(B(m,A))=N(0,vI_d)(B(0,A))\ge1/2$ for
every $m$. Each $Q_j$ is tight, and a finite union of compact sets is
compact, so there is a compact $K$ with $Q_j(K^{\mathrm c})<\delta$ for
all $j$. Let $\rho_K:=\sup_{x\in K}\Vert x\Vert$. If
$\Vert m\Vert>A+\rho_K$, the ball $B(m,A)$ does not meet $K$, so
$B(m,A)\subseteq K^{\mathrm c}$ and
\[
 P_m(B(m,A))\ge\tfrac12,\qquad Q_j(B(m,A))\le Q_j(K^{\mathrm c})<\delta .
\]
Data processing for the two-cell partition $\{B(m,A),B(m,A)^{\mathrm c}\}$
bounds the divergence by the binary divergence of the two cell
probabilities, and with $p:=P_m(B(m,A))\ge1/2$, $q:=Q_j(B(m,A))<\delta$,
\[
 \KL(P_m\Vert Q_j)\ \ge\ \kld(p,q)
 =p\log\frac pq+(1-p)\log\frac{1-p}{1-q}
 \ \ge\ p\log\frac1q-h_2(p)
 \ \ge\ \tfrac12\log\frac1\delta-\log2,
\]
where the middle step drops the nonnegative term $(1-p)\log\frac1{1-q}$
and the last uses $p\ge1/2$, $q<\delta$, and $h_2\le\log2$. This bound
holds for every $j$ and every $m$ with $\Vert m\Vert>A+\rho_K$. That set
of means is open and nonempty, hence has positive probability under the
full-support predictive law, so the essential supremum of
$\min_j\KL(P_m\Vert Q_j)$ is at least $\tfrac12\log(1/\delta)-\log2$.
Letting $\delta\downarrow0$ shows the essential-uniform radius is infinite.

For the Hellinger claim, write the affinity as
$\mathcal A(P,Q):=\int\sqrt{pq}$, so that $\Hel^2(P,Q)=1-\mathcal A(P,Q)$.
Fix $\delta\in(0,1/4)$ and again a compact $K$ with $Q_j(K^{\mathrm c})<\delta$
for all $j$. Since $P_m(K)=N(0,vI_d)(K-m)$ and $K-m$ moves away from the
origin as $\Vert m\Vert\to\infty$, there is $A'$ with $P_m(K)<\delta$
whenever $\Vert m\Vert>A'+\rho_K$. Splitting the affinity over $K$ and
$K^{\mathrm c}$ and applying the Cauchy--Schwarz inequality on each piece,
\[
 \mathcal A(P_m,Q_j)
 \le\sqrt{P_m(K)\,Q_j(K)}+\sqrt{P_m(K^{\mathrm c})\,Q_j(K^{\mathrm c})}
 \le\sqrt\delta+\sqrt\delta=2\sqrt\delta .
\]
Hence $\Hel^2(P_m,Q_j)\ge1-2\sqrt\delta$ for all $j$ on the
positive-probability set $\{\Vert m\Vert>A'+\rho_K\}$. Given any radius
$r<1$, choose $\delta$ so small that $1-2\sqrt\delta>r^2$; then
$\min_j\Hel(P_m,Q_j)>r$ on a set of positive predictive probability, so
the codebook is not a uniform cover of radius $r$.
\end{proof}

\section{Hellinger covering and the local modulus}
\label{app:hellinger}

This extension makes the quantization argument usable when posterior geometry is controlled in Hellinger distance. An explicit comparison between the two divergences then converts the geometric bound into a posterior KL guarantee.

We use normalized squared Hellinger distance $\Hel^2(P,Q)=\tfrac12\int(\sqrt p-\sqrt q)^2\,d\mu$. Let $\mathcal P$ be the essential posterior family, $\Mix(\mathcal P)$ its measurable mixtures, and
\[
\omega(r)=\sup\big\{\KL(P\Vert Q):P\in\mathcal P,\ Q\in\Mix(\mathcal P),\ \Hel(P,Q)\le r\big\}.
\]

\begin{proposition}
\label{prop:hellinger_cover}
A uniform Hellinger cover of radius $r$ with $M$ centres yields a deterministic summary with
$I(Y;S)\le\log M$, $\Delta_{\Hel}(S):=\mathbb{E}\,\Hel^2(\Pi_Y,\Pi_S)\le4r^2$ and $\Delta_\KL(S)\le\omega(2r)$. For an arbitrary $M$-centre codebook with expected squared nearest-centre error $e$, the nearest-centre quantizer satisfies $\Delta_{\Hel}(S)\le4e$, and if $\KL(P\Vert Q)\le C\,\Hel^2(P,Q)$ holds for every relevant posterior--mixture pair then $\Delta_\KL(S)\le4Ce$.
\end{proposition}

\begin{proof}
Let $Q_1,\ldots,Q_M$ be the codebook and let $S=g(Y)$ be its
nearest-centre quantizer, $g(y)=\min\argmin_j\Hel(\Pi_y,Q_j)$, which is
measurable as in the proof of Theorem~\ref{thm:avg_quant}. Its rate is
$H(S)\le\log M$. Write $\Pi_j:=\Pi_{S=j}=\mathbb E[\Pi_Y\mid S=j]$ for the
cell posterior, which is a mixture of posteriors in cell $j$ and hence
an element of $\Mix(\mathcal P)$. The argument compares $\Pi_Y$ with
$\Pi_S$ through the centre $Q_S$, using two facts about the Hellinger
distance: it is a metric, and its square is convex in each argument.

The convexity is seen from the affinity: $\Hel^2(P,Q)=1-\int\sqrt{pq}\,d\mu$,
and for fixed $q$ the map $p\mapsto\sqrt{pq}$ is concave pointwise, so
$P\mapsto\Hel^2(P,Q)$ is convex. The integral form of Jensen's inequality
applied to the mixture $\Pi_j$ gives, for each positive-mass cell,
\begin{equation}
\label{eq:hellinger_cell_jensen}
 \Hel^2(Q_j,\Pi_j)\ \le\ \mathbb E\big[\Hel^2(Q_j,\Pi_Y)\,\big|\,S=j\big].
\end{equation}

We first treat the uniform cover. Its defining property is
$\Hel(\Pi_y,Q_{g(y)})\le r$ for $P_Y$-almost every $y$, so the right side
of~\eqref{eq:hellinger_cell_jensen} is at most $r^2$ and
$\Hel(Q_j,\Pi_j)\le r$ for every cell. The triangle inequality then gives,
almost surely,
\[
 \Hel(\Pi_Y,\Pi_S)\le\Hel(\Pi_Y,Q_S)+\Hel(Q_S,\Pi_S)\le r+r=2r .
\]
Squaring and taking expectations gives $\Delta_{\Hel}(S)\le4r^2$. For the
divergence, $\Pi_Y\in\mathcal P$ and $\Pi_S\in\Mix(\mathcal P)$ are at
Hellinger distance at most $2r$, so by the definition of the modulus
$\KL(\Pi_Y\Vert\Pi_S)\le\omega(2r)$ almost surely, and
$\Delta_\KL(S)\le\omega(2r)$.

We now treat an arbitrary codebook with expected squared nearest-centre
error $e=\mathbb E\min_j\Hel^2(\Pi_Y,Q_j)=\mathbb E\,\Hel^2(\Pi_Y,Q_S)$.
The triangle inequality and $(a+b)^2\le2a^2+2b^2$ give, pointwise,
\[
 \Hel^2(\Pi_Y,\Pi_S)\le2\Hel^2(\Pi_Y,Q_S)+2\Hel^2(Q_S,\Pi_S).
\]
Take conditional expectations given $S=j$ and
use~\eqref{eq:hellinger_cell_jensen} on the second term:
\[
 \mathbb E\big[\Hel^2(\Pi_Y,\Pi_S)\,\big|\,S=j\big]
 \le2\,\mathbb E\big[\Hel^2(\Pi_Y,Q_j)\,\big|\,S=j\big]+2\Hel^2(Q_j,\Pi_j)
 \le4\,\mathbb E\big[\Hel^2(\Pi_Y,Q_j)\,\big|\,S=j\big].
\]
Averaging over the cells gives $\Delta_{\Hel}(S)\le4e$. If
$\KL(P\Vert Q)\le C\,\Hel^2(P,Q)$ for every $P\in\mathcal P$ and
$Q\in\Mix(\mathcal P)$, apply it pointwise to the pair $(\Pi_Y,\Pi_S)$ and
take expectations: $\Delta_\KL(S)\le C\,\Delta_{\Hel}(S)\le4Ce$.
\end{proof}

A local quadratic modulus is only usable when individual distances are controlled on its range of validity, which a uniform cover supplies and an average Hellinger bound does not. A checkable sufficient condition for the global comparison is that the density ratios $t=p/q$ of every relevant pair lie in a compact subinterval of $(0,\infty)$: the function $t\mapsto(t\log t-t+1)/(\sqrt t-1)^2$, extended by the value $2$ at $t=1$, is then bounded, and integrating gives $\KL\le C\,\Hel^2$. Conversely, unbounded likelihood ratios alone do not establish that every local comparison fails. We record this route because Hellinger entropy is the standard currency in nonparametric Bayesian theory, but we do not import covering rates for sampling models into posterior geometry: the two are related only through an explicit map and tail control, and Proposition~\ref{prop:no_unif_cover} shows the translation can fail.

\section{Further converses}
\label{app:converses}

\subsection{Proof of the posterior small-ball bound}
\begin{proof}[Proof of Theorem \ref{thm:smallball}]
The bound compares the joint law of the random posterior and its
reconstruction with the corresponding product law on a single event, the
ball of radius $t$. Under the joint law the ball is likely, because the
average distortion is small; under the product law it is unlikely, because
no single reconstruction captures more than $\mathcal Q_Z(t)$ of the
predictive mass.

Let $S$ be any summary with $\Delta_\KL(S)\le\varepsilon$, put
$\widehat Z:=\Pi_S$, and define the event
$E_t:=\{(z,Q):\KL(z\Vert Q)\le t\}$ in the product of two copies of the
space of probability measures on $\Theta$. Write
\[
 p:=P_{Z,\widehat Z}(E_t)=\mathbb P\big(\KL(Z\Vert\widehat Z)\le t\big),
 \qquad
 q:=(P_Z\otimes P_{\widehat Z})(E_t).
\]
For the product probability, Tonelli's theorem integrates the ball
probability of a fixed reconstruction against the law of $\widehat Z$:
\[
 q=\int\mathbb P\{\KL(Z\Vert Q)\le t\}\,P_{\widehat Z}(dQ)
 \ \le\ \sup_Q\mathbb P\{\KL(Z\Vert Q)\le t\}=\mathcal Q_Z(t).
\]
For the joint probability, Markov's inequality applied to the nonnegative
variable $\KL(Z\Vert\widehat Z)$ gives
\[
 1-p=\mathbb P\big(\KL(Z\Vert\widehat Z)>t\big)
 \le\frac{\mathbb E\,\KL(Z\Vert\widehat Z)}{t}=\frac{\Delta_\KL(S)}{t}\le\frac\varepsilon t,
\]
so $p\ge a:=1-\varepsilon/t>0$.

Now compare the two laws. Since $Z$ is a function of $Y$ and $\widehat Z$
a function of $S$, data processing gives $I(Y;S)\ge I(Z;\widehat Z)$.
The mutual information $I(Z;\widehat Z)=\KL(P_{Z,\widehat Z}\Vert P_Z\otimes P_{\widehat Z})$
can only decrease when both laws are reduced to the two-cell partition
$\{E_t,E_t^{\mathrm c}\}$, where it becomes the binary divergence of the
cell probabilities:
\[
 I(Z;\widehat Z)\ \ge\ \kld(p,q)
 =p\log\frac pq+(1-p)\log\frac{1-p}{1-q}.
\]
Dropping the nonnegative term $(1-p)\log\frac1{1-q}$ and rearranging,
\[
 \kld(p,q)\ \ge\ p\log\frac1q-h_2(p)
 \ \ge\ p\log\frac1{\mathcal Q_Z(t)}-\log2
 \ \ge\ a\log\frac1{\mathcal Q_Z(t)}-\log2,
\]
using $q\le\mathcal Q_Z(t)<1$, $h_2\le\log2$, and $p\ge a$ together with
$\log(1/\mathcal Q_Z(t))>0$. Since $S$ was an arbitrary feasible summary
and rates are nonnegative, this yeilds the bound.

For the strengthening, recall that $\kld(p,q)$ is nondecreasing in $p$ on
$p\ge q$ and nonincreasing in $q$ on $q\le p$. If $a>\mathcal Q_Z(t)$, then
$p\ge a>\mathcal Q_Z(t)\ge q$, so both monotonicity properties apply and
$\kld(p,q)\ge\kld(a,q)\ge\kld(a,\mathcal Q_Z(t))$.

No metric or density structure was used: only the measurability of the
event $E_t$ and the nonnegativity of the distortion, so the argument
applies verbatim to any jointly measurable nonnegative distortion.
\end{proof}

\subsection{A Laplace form}

For $s>0$ let $\mathcal A_Z(s):=\sup_Q\mathbb{E}\exp\{-s\KL(Z\Vert Q)\}$.

\begin{proposition}
\label{prop:laplace}
$R^\star(\varepsilon)\ \ge\ \big[\sup_{s>0}\{-s\varepsilon-\log\mathcal A_Z(s)\}\big]_+$, and the statement holds for any measurable nonnegative distortion and arbitrary reconstructions.
\end{proposition}

\begin{proof}
The tool is the Gibbs variational inequality: for probability measures
$P\ll\mu$ and a measurable $f\le0$,
\begin{equation}
\label{eq:gibbs_variational}
 \KL(P\Vert\mu)\ \ge\ \int f\,dP-\log\int e^f\,d\mu .
\end{equation}
For bounded $f$ this follows by writing $\mu_f(dz):=e^{f(z)}\mu(dz)/\int e^f\,d\mu$
and computing
\[
 \KL(P\Vert\mu)-\Big[\int f\,dP-\log\int e^f\,d\mu\Big]
 =\int\log\Big(\frac{dP}{d\mu}\,e^{-f}\int e^f\,d\mu\Big)dP
 =\KL(P\Vert\mu_f)\ \ge\ 0 .
\]
For unbounded $f\le0$, apply this to the truncations $f_N:=\max\{f,-N\}$;
as $N\to\infty$, $\int f_N\,dP\downarrow\int f\,dP$ by monotone convergence
and $\int e^{f_N}d\mu\to\int e^f\,d\mu$ by dominated convergence with the
bound $e^{f_N}\le1$, so the right side of \eqref{eq:gibbs_variational}
for $f_N$ converges to that for $f$.

Fix $s>0$, a summary $S$, and a measurable reconstruction $Q_S$. Apply
\eqref{eq:gibbs_variational} with $\mu:=P_Z$, $P:=P_{Z\mid S=s_0}$, and
$f(z):=-s\KL(z\Vert Q_{s_0})\le0$:
\begin{align*}
 \KL(P_{Z\mid S=s_0}\Vert P_Z)
 &\ \ge\ -s\,\mathbb E\big[\KL(Z\Vert Q_S)\,\big|\,S=s_0\big]
 -\log\mathbb E_{Z\sim P_Z}\exp\{-s\KL(Z\Vert Q_{s_0})\}\\
 &\ \ge\ -s\,\mathbb E\big[\KL(Z\Vert Q_S)\,\big|\,S=s_0\big]-\log\mathcal A_Z(s),
\end{align*}
where the last step bounds the Laplace transform of the fixed
reconstruction $Q_{s_0}$ by its supremum over reconstructions. Averaging
over $s_0$ with respect to the law of $S$ turns the left side into
$I(Z;S)=\mathbb E\,\KL(P_{Z\mid S}\Vert P_Z)$, and data processing gives
\[
 I(Y;S)\ \ge\ I(Z;S)\ \ge\ -s\,\mathbb E\,\KL(Z\Vert Q_S)-\log\mathcal A_Z(s).
\]
Now take $Q_S:=\Pi_S$ and a feasible summary, so that
$\mathbb E\,\KL(Z\Vert Q_S)=\Delta_\KL(S)\le\varepsilon$; then
$I(Y;S)\ge-s\varepsilon-\log\mathcal A_Z(s)$. Taking the supremum over
$s>0$ and the infimum over feasible summaries gives the claim, and the
bracket $[\cdot]_+$ records that rates are nonnegative. The quantity
$\mathcal A_Z(s)$ is strictly positive, since the choice $Q=\pi$ gives
$\KL(Z\Vert\pi)<\infty$ almost surely because
$\mathbb E\,\KL(Z\Vert\pi)=I(\theta;Y)<\infty$; and it is at most one, so
$-\log\mathcal A_Z(s)\ge0$.
\end{proof}

This is the general Shannon lower bound of \cite{riegler2018rate} with the prior-predictive law of the random posterior as reference measure, in which case the generalized entropy term vanishes. It can be stronger than the small ball bound of Theorem \ref{thm:smallball} because it integrates the whole distortion profile rather than thresholding it.

\subsection{Matching exponents under small-ball regularity}

Suppose $\mathcal Q_Z(t)\le Ct^\alpha$ for all $t>0$. For any reconstruction
$Q$, put $X:=\KL(Z\Vert Q)\ge0$ and write the Laplace transform as a
layer-cake integral, substituting $u=e^{-st}$:
\begin{align*}
 \mathbb E\,e^{-sX}=\int_0^1\mathbb P\big(e^{-sX}\ge u\big)\,du
 &=\int_0^\infty\mathbb P(X\le t)\,se^{-st}\,dt\\
 &\le\int_0^\infty Ct^\alpha se^{-st}\,dt
 =C\,\Gamma(\alpha+1)\,s^{-\alpha}.
\end{align*}
Taking the supremum over $Q$ gives $\mathcal A_Z(s)\le C\Gamma(\alpha+1)s^{-\alpha}$,
so Proposition~\ref{prop:laplace} yields
$R^\star(\varepsilon)\ge-s\varepsilon+\alpha\log s-\log\{C\Gamma(\alpha+1)\}$
for every $s>0$. The right side is maximized at $s=\alpha/\varepsilon$, where
it equals
\[
R^\star(\varepsilon)\ \ge\ \alpha\log\frac{\alpha}{\varepsilon}-\alpha-\log\{C\Gamma(\alpha+1)\}.
\]
If in addition $D^{\mathrm{av}}_M\le KM^{-1/\alpha}$, Theorem~\ref{thm:avg_quant} gives the matching upper bound and $R^\star(\varepsilon)=\alpha\log(1/\varepsilon)+O(1)$ as $\varepsilon\downarrow0$. These are explicit sufficient conditions for a matching exponent for a fixed law, not a general theorem; for a sequence of experiments indexed by $n$ all dependence of $C,K,\alpha$ on $n$ must be tracked, and the conditions exclude a finite posterior source at resolutions below its atom scale.

\subsection{A mass-sensitive cluster converse}

\begin{proposition}
\label{prop:cluster}
Let $A_1,\dots,A_M$ be disjoint measurable data sets with $M\ge2$, $A=\bigcup_jA_j$, $m=P_Y(A)>0$ and $q_j=P_Y(A_j)/m$. Suppose $\Hel(\Pi_y,P_j)\le\rho$ for $y\in A_j$ and $\Hel(P_i,P_j)\ge2(\rho+\eta)$ for $i\ne j$, with $\eta>0$. Then every summary satisfies
\[
\Delta_{\Hel}(S)\ \ge\ m\eta^2\left[\frac{H(q)-I(Y;S)/m-\log2}{\log M}\right]_+,
\qquad
\Delta_\KL(S)\ \ge\ 2\Delta_{\Hel}(S).
\]
\end{proposition}

\begin{proof}
The argument turns a summary into a decoder of the cluster index and
applies Fano's inequality. An accurate summary posterior identifies the
cluster; so if the cluster is hard to identify from $S$, the summary
posterior is often inaccurate.

On the event $A$, let $J$ be the cluster index, $J=j$ when $Y\in A_j$,
and set $J:=0$ off $A$. Define a decoder $\widehat J(S)$ as the unique
$j$ with $\Hel(P_j,\Pi_S)<\rho+\eta$ when such a $j$ exists, and
arbitrarily otherwise. Uniqueness holds because two such indices $j\ne j'$
would give $\Hel(P_j,P_{j'})\le\Hel(P_j,\Pi_S)+\Hel(\Pi_S,P_{j'})<2(\rho+\eta)$,
contradicting the separation hypothesis. If $Y\in A_j$ and
$\Hel(\Pi_Y,\Pi_S)<\eta$, then
\[
 \Hel(P_j,\Pi_S)\le\Hel(P_j,\Pi_Y)+\Hel(\Pi_Y,\Pi_S)<\rho+\eta,
\]
so the decoder outputs $j$ correctly. Hence the conditional error
probability $e:=\mathbb P(\widehat J\ne J\mid A)$ satisfies
\[
 e\ \le\ \mathbb P\big(\Hel(\Pi_Y,\Pi_S)\ge\eta\,\big|\,A\big).
\]

Fano's inequality in its nonuniform form states that for a variable $J$
with $M$ possible values and any estimator $\widehat J$ with error
probability $e$, $H(J\mid\widehat J)\le h_2(e)+e\log(M-1)\le\log2+e\log M$.
Under the conditional law given $A$, and since $\widehat J$ is a function
of $S$,
\[
 I(J;S\mid A)=H(J\mid A)-H(J\mid S,A)\ \ge\ H(q)-H(J\mid\widehat J,A)
 \ \ge\ H(q)-\log2-e\log M .
\]
To relate this to the rate, note that $(\mathbf 1_A,J)$ is a function of
$Y$, so by data processing and the chain rule,
\[
 I(Y;S)\ \ge\ I(\mathbf 1_A,J;S)\ \ge\ I(J;S\mid\mathbf 1_A)
 =m\,I(J;S\mid A)+(1-m)\,I(J;S\mid A^{\mathrm c})=m\,I(J;S\mid A),
\]
because $J$ is constant off $A$. Combining the last two displays and
solving for $e$,
\[
 e\ \ge\ \frac{H(q)-I(Y;S)/m-\log2}{\log M}.
\]
Finally, restrict the Hellinger distortion to $A$ and apply Markov's
inequality to the event $\{\Hel(\Pi_Y,\Pi_S)\ge\eta\}$:
\[
 \Delta_{\Hel}(S)\ \ge\ m\,\mathbb E\big[\Hel^2(\Pi_Y,\Pi_S)\,\big|\,A\big]
 \ \ge\ m\eta^2\,\mathbb P\big(\Hel(\Pi_Y,\Pi_S)\ge\eta\,\big|\,A\big)
 \ \ge\ m\eta^2e .
\]
Inserting the lower bound on $e$, and noting that $\Delta_{\Hel}\ge0$
makes the positive part harmless, gives the first display.

For the second display, compare the two divergences pointwise. By Jensen's
inequality for the concave logarithm,
\begin{align*}
 \KL(P\Vert Q)=-2\,\mathbb E_P\log\sqrt{q/p}
 &\ \ge\ -2\log\mathbb E_P\sqrt{q/p}
 =-2\log\int\sqrt{pq}\\
 &=-2\log\big(1-\Hel^2(P,Q)\big)\ \ge\ 2\Hel^2(P,Q),
\end{align*}
the last step using $-\log(1-x)\ge x$. Taking expectations with
$(P,Q)=(\Pi_Y,\Pi_S)$ gives $\Delta_\KL(S)\ge2\Delta_{\Hel}(S)$.
\end{proof}

Proposition~\ref{prop:cluster} is not dominated by Theorem~\ref{thm:smallball}: it depends on the entropy $H(q)$ of the cluster weights, whereas the small-ball quantity can be controlled by the single largest cluster. It is the appropriate statement when the experiment concentrates on a few well-separated regions of unequal probability. If the law is genuinely uniform on $M$ separated posterior atoms, setting $m=1$, $\rho=0$ and $q_j=1/M$ recovers a finite Fano converse for that experiment.

\subsection{Evaluating the converses}

Both $\mathcal Q_Z(t)$ and $\mathcal A_Z(s)$ are suprema over reconstructions, and the converses require \emph{upper} bounds on them. Maximizing over a restricted family, or reporting the best value returned by a local optimizer, bounds the supremum from below and therefore yields an optimistic quantity that is not a valid converse. Repeated restarts do not repair this. Any numerical value should either certify the supremum or be reported as a candidate.

\section{Parametric rates: hypotheses and achievability}
\label{app:parametric}

\paragraph{The prior-averaged expansion.}
Lemma~\ref{lem:cb_benchmark} uses the prior-averaged entropy-risk theorem of \citet[Theorem~1, Eq.~(2.3)]{clarkeBarron1994}. Its hypotheses include an interior compact prior support with the stated continuity and integrability conditions, likelihood smoothness and score/Hessian moment bounds, positive-definite Fisher information, and identifiability. Under those conditions,
\begin{equation}
\label{eq:cb_expansion}
I(\theta;Y^n)=\frac d2\log\frac{n}{2\pi e}+h(\pi)+\frac12\mathbb E_\pi\log\det\mathcal I(\theta)+o(1).
\end{equation}
Combining this equality with the posterior-loss identity proves the corollary: for a summary with $\Delta_\KL(S_n)\le\varepsilon$, Lemma~\ref{lem:identity} and data processing give $I(Y^n;S_n)\ge I(\theta;S_n)=I(\theta;Y^n)-\Delta_\KL(S_n)\ge I(\theta;Y^n)-\varepsilon$, and substituting \eqref{eq:cb_expansion} gives \eqref{eq:cb_rate_bound}. Pointwise asymptotic normality alone does not justify averaging a remainder under the prior.

\subsection{A mean-square route to the lower bound}

\begin{proposition}
\label{prop:rootn}
Suppose $h(\pi)$ is finite and some estimator $\widehat\theta_n(Y^n)$ has $0<D_n=\mathbb{E}\Vert \theta-\widehat\theta_n\Vert^2<\infty$. Then $I(\theta;Y^n)\ge h(\pi)-\frac d2\log\frac{2\pi eD_n}{d}$. If $\limsup_n nD_n\le K\in(0,\infty)$, every summary with $\Delta_\KL(S)\le\varepsilon$ satisfies $I(Y^n;S)\ge\frac d2\log n+h(\pi)-\frac d2\log\frac{2\pi eK}{d}-\varepsilon-o(1)$.
\end{proposition}

\begin{proof}
The first claim bounds the conditional entropy of $\theta$ given the data
by comparing each posterior with a Gaussian centred at the estimator. If
$I(\theta;Y^n)=\infty$ there is nothing to prove, so assume it is finite;
then, with $h(\pi)$ finite, $I(\theta;Y^n)=h(\pi)-h(\theta\mid Y^n)$, where
$h(\theta\mid Y^n)=\mathbb E\,h(\Pi_Y)$ and a posterior without a Lebesgue
density is assigned entropy $-\infty$.

Fix a data value $y$ and let $G_y:=N\big(\widehat\theta_n(y),(D_n/d)I_d\big)$
with density $g_y$. If $\Pi_y$ has a density, nonnegativity of
$\KL(\Pi_y\Vert G_y)=-h(\Pi_y)-\mathbb E_{\Pi_y}\log g_y$ gives
\[
 h(\Pi_y)\ \le\ -\mathbb E_{\Pi_y}\log g_y
 =\frac d2\log\frac{2\pi D_n}{d}+\frac{d}{2D_n}\,\mathbb E_{\Pi_y}\Vert\theta-\widehat\theta_n(y)\Vert^2,
\]
and the inequality is trivial otherwise. Averaging over $y$ and using
$\mathbb E\Vert\theta-\widehat\theta_n\Vert^2=D_n$,
\[
 h(\theta\mid Y^n)\ \le\ \frac d2\log\frac{2\pi D_n}{d}+\frac{d}{2D_n}D_n
 =\frac d2\log\frac{2\pi eD_n}{d}.
\]
Subtracting from $h(\pi)$ gives $I(\theta;Y^n)\ge h(\pi)-\frac d2\log\frac{2\pi eD_n}{d}$.

For the second claim, let $\Delta_\KL(S)\le\varepsilon$. By
Lemma~\ref{lem:identity} and data processing,
\[
 I(Y^n;S)\ \ge\ I(\theta;S)=I(\theta;Y^n)-\Delta_\KL(S)
 \ \ge\ h(\pi)-\frac d2\log\frac{2\pi eD_n}{d}-\varepsilon .
\]
The hypothesis $\limsup_nnD_n\le K$ gives $D_n\le(K+o(1))/n$, hence
\[
 -\frac d2\log\frac{2\pi eD_n}{d}\ \ge\ -\frac d2\log\frac{2\pi e(K+o(1))}{nd}
 =\frac d2\log n-\frac d2\log\frac{2\pi eK}{d}-o(1),
\]
which gives the stated bound.
\end{proof}

If $D_n=0$, then $\theta$ is determined by the data and, for a continuous prior, $I(\theta;Y^n)=\infty$. For positive error, the proposition replaces the Clarke--Barron conditions by a posterior localization hypothesis at scale $n^{-1/2}$, which is implied by regular Bernstein--von Mises theory together with uniform integrability. Weak posterior convergence alone does not imply it.

\subsection{Achievability under a quadratic posterior chart}

\begin{assumption}[Quadratic posterior chart]
\label{ass:chart}
There is a measurable statistic $T_n(Y^n)\in\R^d$ and a constant $L_n<\infty$ with
$\KL(\Pi_y\Vert\Pi_{y'})\le L_n\Vert T_n(y)-T_n(y')\Vert^2$ for all $y,y'$ in a set of full probability. Write $V_n=\mathbb{E}\Vert T_n-\mathbb{E}T_n\Vert^2<\infty$.
\end{assumption}

The assumption implies that observations with equal chart coordinates have equal posteriors, and is therefore considerably stronger than the existence of an asymptotically normal estimator. We verify it for an exponential family class.

\begin{lemma}
\label{lem:chart_cell}
Under Assumption~\ref{ass:chart}, every deterministic quantizer $S=g(T_n)$ satisfies
$\Delta_\KL(S)\le2L_n\,\mathbb{E}\Vert T_n-\mathbb{E}[T_n\mid S]\Vert^2\le2L_n\,\mathbb{E}\Vert T_n-t_S\Vert^2$ for any cell representatives $t_S$.
\end{lemma}

\begin{proof}
Write $Y$ for $Y^n$. The cell posterior $\Pi_S$ is the average of the
observation posteriors over the cell, and the chart assumption controls
the divergence between two observation posteriors; the proof compares
$\Pi_Y$ with a random second observation from the same cell.

Given $(Y,S)$, let $Y'$ be drawn from the conditional law $P_{Y\mid S}$,
independently of $Y$. Then the barycenter formula reads
$\Pi_S=\mathbb E[\Pi_Y\mid S]=\mathbb E[\Pi_{Y'}\mid Y,S]$. The divergence
$Q\mapsto\KL(P\Vert Q)$ is convex in its second argument, so the integral
form of Jensen's inequality gives, almost surely,
\begin{align*}
 \KL(\Pi_Y\Vert\Pi_S)=\KL\big(\Pi_Y\,\big\Vert\,\mathbb E[\Pi_{Y'}\mid Y,S]\big)
 &\ \le\ \mathbb E\big[\KL(\Pi_Y\Vert\Pi_{Y'})\,\big|\,Y,S\big]\\
 &\ \le\ L_n\,\mathbb E\big[\Vert T_n(Y)-T_n(Y')\Vert^2\,\big|\,Y,S\big],
\end{align*}
where the last step is Assumption~\ref{ass:chart} applied to the pair
$(Y,Y')$. Taking expectations,
\[
 \Delta_\KL(S)\ \le\ L_n\,\mathbb E\Vert T_n-T_n'\Vert^2,\qquad T_n':=T_n(Y').
\]
Given $S$, the vectors $T_n$ and $T_n'$ are independent with the same
conditional law, so the expected squared difference is twice the trace of
the conditional covariance:
\[
 \mathbb E\big[\Vert T_n-T_n'\Vert^2\,\big|\,S\big]
 =2\,\mathbb E\big[\Vert T_n-\mathbb E[T_n\mid S]\Vert^2\,\big|\,S\big].
\]
This gives the first inequality. For the second, the conditional mean
minimizes conditional expected squared error among $\sigma(S)$-measurable
vectors:
$\mathbb E[\Vert T_n-t_S\Vert^2\mid S]=\mathbb E[\Vert T_n-\mathbb E[T_n\mid S]\Vert^2\mid S]+\Vert\mathbb E[T_n\mid S]-t_S\Vert^2$.
\end{proof}

\begin{theorem}
\label{thm:chart_achiev}
Under Assumption~\ref{ass:chart} with $L_n>0$, for every $\varepsilon>0$ there is a countable-valued deterministic summary with $\Delta_\KL(S)\le\varepsilon$ and
\[
I(Y^n;S)\ \le\ \frac d2\log\left\{2\pi e\left[\left(\sqrt{\frac{L_nV_n}{2\varepsilon}}+\frac12\right)^{\!2}+\frac1{12}\right]\right\}.
\]
In particular, if $L_n=O(n)$ and $V_n=O(1)$ then $R^\star_n(\varepsilon)\le\frac d2\log n+O_{d,\varepsilon}(1)$, which with Lemma~\ref{lem:cb_benchmark} gives the rate $\frac d2\log n+O(1)$.
\end{theorem}

\begin{proof}
The summary is a cubic lattice quantizer of the chart statistic. We first
choose the cell width so that the distortion bound of
Lemma~\ref{lem:chart_cell} equals $\varepsilon$, and then bound the entropy
of the lattice label by the entropy of a continuous variable with the
same covariance up to the dithering term.

Set $\delta^2:=2\varepsilon/(L_nd)$ and $S:=\operatorname{round}(T_n/\delta)\in\mathbb Z^d$,
rounding each coordinate to the nearest integer. Each coordinate of
$T_n-\delta S$ has absolute value at most $\delta/2$, so
$\Vert T_n-\delta S\Vert^2\le d\delta^2/4$. Lemma~\ref{lem:chart_cell}
with the cell representatives $t_S=\delta S$ gives
\[
 \Delta_\KL(S)\ \le\ 2L_n\cdot\frac{d\delta^2}{4}=\frac{L_nd\delta^2}{2}=\varepsilon .
\]

For the rate, $I(Y^n;S)=H(S)$ because $S$ is a deterministic
countable-valued quantizer (Lemma~\ref{lem:deterministic_atomic}), so it
suffices to bound $H(S)$. Let $U$ be uniform on $[-\frac12,\frac12]^d$ and
independent of everything else. The dithered variable $W:=S+U$ has a
density equal to $\mathbb P(S=k)$ on each unit cube $k+[-\frac12,\frac12)^d$,
so its differential entropy is
\[
 h(W)=-\sum_{k\in\mathbb Z^d}\mathbb P(S=k)\log\mathbb P(S=k)=H(S).
\]
Among $\mathbb R^d$-valued variables with covariance trace $v$, the
isotropic Gaussian maximizes differential entropy, which gives
$h(W)\le\frac d2\log\frac{2\pi ev}{d}$ with $v=\operatorname{tr}\operatorname{Cov}(W)$.
By independence,
$\operatorname{tr}\operatorname{Cov}(W)=\operatorname{tr}\operatorname{Cov}(S)+\operatorname{tr}\operatorname{Cov}(U)
=\operatorname{tr}\operatorname{Cov}(S)+d/12$, since each coordinate of $U$
has variance $1/12$.

It remains to bound $\operatorname{tr}\operatorname{Cov}(S)=\mathbb E\Vert S-\mathbb ES\Vert^2$.
Centring at $\mathbb E[T_n]/\delta$ instead of $\mathbb ES$ can only increase
the mean squared norm, and the triangle inequality in $L^2$ gives
\begin{align*}
 \sqrt{\operatorname{tr}\operatorname{Cov}(S)}
 &\ \le\ \sqrt{\mathbb E\Vert S-\mathbb E[T_n]/\delta\Vert^2}\\
 &\ \le\ \sqrt{\mathbb E\Vert S-T_n/\delta\Vert^2}+\sqrt{\mathbb E\Vert T_n/\delta-\mathbb E[T_n]/\delta\Vert^2}
 \ \le\ \frac{\sqrt d}{2}+\frac{\sqrt{V_n}}{\delta}.
\end{align*}
Substituting $\delta=\sqrt{2\varepsilon/(L_nd)}$ gives
$\sqrt{V_n}/\delta=\sqrt d\,\sqrt{L_nV_n/(2\varepsilon)}$, hence
\[
 \operatorname{tr}\operatorname{Cov}(S)\ \le\ d\Big(\sqrt{\frac{L_nV_n}{2\varepsilon}}+\frac12\Big)^2,
 \qquad
 \frac{\operatorname{tr}\operatorname{Cov}(W)}{d}\ \le\ \Big(\sqrt{\frac{L_nV_n}{2\varepsilon}}+\frac12\Big)^2+\frac1{12}.
\]
Inserting this into the maximum-entropy bound gives the display; in
particular the finite second moment of $T_n$ makes $H(S)$ finite. If
$L_n=O(n)$ and $V_n=O(1)$, the bracket is $O(n/\varepsilon)$ and the bound
reads $\frac d2\log n+O_{d,\varepsilon}(1)$.
\end{proof}

A fixed-cardinality variant follows from the extended Pierce lemma \cite[Thm.~1.1(b)]{elnmeir2020greedy}: if $\mathbb{E}\Vert T_n\Vert^{2+\gamma}<\infty$ for some $\gamma>0$ then $D^{\mathrm{av}}_M\le2L_n\kappa^2_{d,2,\gamma}(\mathbb{E}\Vert T_n\Vert^{2+\gamma})^{2/(2+\gamma)}M^{-2/d}$, whence $\log N^{\mathrm{av}}(\varepsilon)\le\frac d2\log L_n+\frac d2\log(1/\varepsilon)+O(1)$. The displayed $O(1)$ is uniform over a sequence of experiments only when the required $(2+\gamma)$ moments are uniformly bounded. The entropy-constrained bound above needs only $V_n<\infty$.

\begin{remark}
    The deterministic quantizer of Theorem \ref{thm:chart_achiev} attains the leading order $\tfrac d2\log n$ but not the optimal frontier: by Theorem \ref{thm:separation}, its rate exceeds $R^\star_n(\varepsilon)$ by a strictly positive amount, which is bounded in $n$ and therefore absorbed in the $O(1)$.
\end{remark}
\subsection{A verified non-Gaussian class}

\begin{proposition}
\label{prop:expfam}
In one dimension let $p(y\mid x)=h_0(y)\exp\{xy-A(x)\}$ with $A$ finite and twice differentiable on $\R$ and $0<\kappa\le A''(x)\le K<\infty$ for all $x$. Let $\theta\sim N(0,\tau^2)$ and take $n$ conditionally independent observations. Then,
\begin{enumerate}[label=(\roman*),leftmargin=*,itemsep=3pt]
\item $T_n=\bar Y$ satisfies Assumption~\ref{ass:chart} with $L_n=n^2/\{2(\tau^{-2}+n\kappa)\}$ and $V_n\le K^2\tau^2+K/n$.
\item $\widehat\theta_n=(A')^{-1}(\bar Y)$ has Bayes mean-square error at most $K/(\kappa^2n)$.
\end{enumerate}
Consequently, $R^\star_n(\varepsilon)=\frac12\log n+O_\varepsilon(1)$ for each fixed $\varepsilon>0$, and $\frac d2\log n+O_{d,\varepsilon}(1)$ for $d$ independent such coordinates.
\end{proposition}

\begin{proof}
We verify the chart assumption through the log-partition function of the
posterior, then bound the chart variance and the estimator's error, check
that the Bayes information is finite, and finally combine the lower and
upper bounds.

Write $k_n:=\tau^{-2}+n\kappa$. The likelihood of $n$ observations depends
on the data only through $t=\bar Y$, and the posterior $P_t$ given
$T_n=t$ has density proportional to
$\exp\{-\psi_t(x)\}$ with
\[
 \psi_t(x):=\frac{x^2}{2\tau^2}-ntx+nA(x),\qquad
 \psi_t''(x)=\tau^{-2}+nA''(x)\ \ge\ k_n .
\]
Thus $P_t$ is strongly log-concave with parameter $k_n$. The scalar
Brascamp--Lieb inequality states that a density proportional to
$e^{-\psi}$ with $\psi''\ge k>0$ satisfies
$\operatorname{Var}(g(X))\le\mathbb E[g'(X)^2/\psi''(X)]$ for smooth $g$
\cite[Prop.~10.1]{saumard2014logconcavity}; with $g(x)=x$ it gives
$\operatorname{Var}_{P_t}(\theta)\le\mathbb E_{P_t}[1/\psi_t''(\theta)]\le k_n^{-1}$.

Define the posterior log-partition function
\[
 F_n(t):=\log\int_{\mathbb R}\exp\{-\psi_t(x)\}\,dx .
\]
Since $A(x)\ge A(0)+A'(0)x+\kappa x^2/2$, the integrand is dominated by a
Gaussian-type function, so $F_n$ is finite for every $t$ and can be
differentiated under the integral sign; the derivatives multiply the
integrand by polynomials in $x$, which the same quadratic bound keeps
integrable. Differentiating twice,
\[
 F_n'(t)=n\,\mathbb E_{P_t}[\theta],\qquad
 F_n''(t)=n^2\operatorname{Var}_{P_t}(\theta)\ \le\ \frac{n^2}{k_n}.
\]
The posterior densities are $p_t(x)=\exp\{-\psi_t(x)-F_n(t)\}$, so
$\log\{p_t(x)/p_u(x)\}=n(t-u)x-F_n(t)+F_n(u)$, and integrating against $P_t$,
\[
 \KL(P_t\Vert P_u)=n(t-u)\,\mathbb E_{P_t}[\theta]-F_n(t)+F_n(u)
 =F_n(u)-F_n(t)-F_n'(t)(u-t).
\]
This is the second-order Taylor remainder of $F_n$, so by the mean value
form of the remainder and the bound on $F_n''$,
\[
 \KL(P_t\Vert P_u)=\tfrac12F_n''(\xi)(u-t)^2\ \le\ \frac{n^2}{2k_n}(u-t)^2
\]
for some $\xi$ between $t$ and $u$. This verifies Assumption~\ref{ass:chart}
with $T_n=\bar Y$ and $L_n=n^2/(2k_n)$.

For the chart variance, the exponential-family identities give
$\mathbb E[T_n\mid\theta]=A'(\theta)$ and
$\operatorname{Var}(T_n\mid\theta)=A''(\theta)/n\le K/n$. By the law of
total variance,
\[
 V_n=\operatorname{Var}(T_n)=\mathbb E\operatorname{Var}(T_n\mid\theta)+\operatorname{Var}(A'(\theta))
 \ \le\ \frac Kn+\operatorname{Var}(A'(\theta)).
\]
Since $|A''|\le K$, the function $A'$ is $K$-Lipschitz, and for an
independent copy $\theta'$ of $\theta$,
\[
 \operatorname{Var}(A'(\theta))=\tfrac12\mathbb E\big[(A'(\theta)-A'(\theta'))^2\big]
 \ \le\ \tfrac12K^2\,\mathbb E\big[(\theta-\theta')^2\big]=K^2\tau^2 .
\]
Hence $V_n\le K^2\tau^2+K/n$.

For the estimator, $A''\ge\kappa>0$ makes $A'$ strictly increasing with
$A'(x)\to\pm\infty$ as $x\to\pm\infty$, so $A':\mathbb R\to\mathbb R$ is a
bijection, and its inverse has derivative $1/A''\le1/\kappa$, hence is
$\kappa^{-1}$-Lipschitz. With $\widehat\theta_n=(A')^{-1}(T_n)$ and
$\theta=(A')^{-1}(A'(\theta))$,
\[
 \mathbb E\big[(\widehat\theta_n-\theta)^2\,\big|\,\theta\big]
 \ \le\ \kappa^{-2}\,\mathbb E\big[(T_n-A'(\theta))^2\,\big|\,\theta\big]
 =\frac{A''(\theta)}{\kappa^2n}\ \le\ \frac{K}{\kappa^2n}.
\]
Averaging over $\theta$ gives the Bayes mean-square error bound. This
error is positive, because the posterior has positive variance, so
Proposition~\ref{prop:rootn} applies with $\limsup_nnD_n\le K/\kappa^2$.

The experiment has finite Bayes information. For any fixed law $Q$ on the
data space, $\mathbb E\,\KL(P_{Y^n\mid\theta}\Vert Q)=I(\theta;Y^n)+\KL(P_{Y^n}\Vert Q)$,
so $I(\theta;Y^n)\le\mathbb E\,\KL(P_{Y^n\mid\theta}\Vert Q)$. Take $Q$ to be
the sampling law at $x=0$. For one observation,
$\KL(p(\cdot\mid x)\Vert p(\cdot\mid0))=\mathbb E_x[xY-A(x)+A(0)]=xA'(x)-A(x)+A(0)$,
and Taylor's theorem with $A''\le K$ gives $xA'(x)-A(x)+A(0)\le Kx^2/2$.
The $n$ observations are conditionally independent, so
\[
 I(\theta;Y^n)\ \le\ n\,\mathbb E\{\theta A'(\theta)-A(\theta)+A(0)\}
 \ \le\ \frac{nK}{2}\mathbb E[\theta^2]=\frac{nK\tau^2}{2}<\infty .
\]
Also $h(\pi)=\frac12\log(2\pi e\tau^2)$ is finite.

We combine the bounds. Proposition~\ref{prop:rootn} with $d=1$ gives, for
every summary with $\Delta_\KL(S)\le\varepsilon$,
$I(Y^n;S)\ge\frac12\log n+h(\pi)-\frac12\log(2\pi eK/\kappa^2)-\varepsilon-o(1)$,
so $R^\star_n(\varepsilon)\ge\frac12\log n-O_\varepsilon(1)$. Theorem~\ref{thm:chart_achiev}
with $L_n=n^2/(2k_n)\le n/(2\kappa)=O(n)$ and $V_n=O(1)$ gives
$R^\star_n(\varepsilon)\le\frac12\log n+O_\varepsilon(1)$. Together,
$R^\star_n(\varepsilon)=\frac12\log n+O_\varepsilon(1)$.

For $d$ independent coordinates $(\theta_j,Y^n_j)$, the posterior is the
product of the coordinate posteriors, and the divergence between two
product measures is the sum of the coordinate divergences. Hence
Assumption~\ref{ass:chart} holds for $T_n=(\bar Y_1,\ldots,\bar Y_d)$ with
the same $L_n$, since
$\KL(\Pi_y\Vert\Pi_{y'})=\sum_j\KL(P_{t_j}\Vert P_{t_j'})\le L_n\sum_j(t_j-t_j')^2=L_n\Vert T_n(y)-T_n(y')\Vert^2$.
The chart variances, the squared estimation errors, the prior entropies,
and the Bayes informations all add over coordinates, so the same two
bounds give $R^\star_n(\varepsilon)=\frac d2\log n+O_{d,\varepsilon}(1)$.
\end{proof}

The condition is not vacuous outside Gaussian data: taking $h_0$ to be an equal mixture of $N(\pm b,\sigma^2)$ gives $A(x)=\frac12\sigma^2x^2+\log\cosh(bx)$ and hence $\kappa=\sigma^2$, $K=\sigma^2+b^2$, with a genuinely non-Gaussian two-component observation law. Nor is it automatic: Bernoulli curvature degenerates and Poisson curvature is unbounded.

\subsection{Proof of Proposition \ref{thm:gaussian_exact}}

\begin{proof}
The proof has a converse part, which lower-bounds the distortion of every
channel of rate $R$ by a conditional entropy-power argument, and an
achievability part, which shows that adding Gaussian noise to the
posterior mean makes every inequality of the converse an equality.

We first record the conjugate computation. Given $Y^n$, the posterior is
$\theta\mid Y^n\sim N(m,vI_d)$ with
\[
 m=\rho\bar Y,\qquad \rho=\frac{n\tau^2}{\sigma^2+n\tau^2},\qquad
 v=\frac{\tau^2\sigma^2}{\sigma^2+n\tau^2},
\]
since the posterior precision is $\tau^{-2}+n\sigma^{-2}=1/v$ and the
posterior mean is $v\cdot n\bar Y/\sigma^2$. Marginally
$\bar Y\sim N(0,(\tau^2+\sigma^2/n)I_d)$, so $m\sim N(0,s^2I_d)$ with
\[
 s^2=\rho^2\Big(\tau^2+\frac{\sigma^2}{n}\Big)=\frac{n\tau^4}{\sigma^2+n\tau^2},
 \qquad
 \frac{s^2}{v}=\frac{n\tau^2}{\sigma^2}=a .
\]
Define $\zeta:=\theta-m$. Since the conditional law of $\zeta$ given $Y^n$
is $N(0,vI_d)$ whatever the value of $Y^n$, the vector $\zeta$ is
independent of $Y^n$; and because a summary's auxiliary randomness is
independent of $(\theta,Y^n)$, the same conditional law holds given
$(Y^n,S)$, so $\zeta$ is independent of $(Y^n,S)$ as well. Thus
$\theta=m+\zeta$ with $m$ a function of $Y^n$ and $\zeta\sim N(0,vI_d)$
independent of $(m,S)$.

For the converse, let $S$ be any summary with $I(Y^n;S)=R<\infty$. Fix a
value $s_0$ of $S$. Given $S=s_0$, the vectors $m$ and $\zeta$ are
independent, so the entropy power inequality
$e^{2h(X_1+X_2)/d}\ge e^{2h(X_1)/d}+e^{2h(X_2)/d}$ for independent
$X_1,X_2$ in $\mathbb R^d$, with the convention $e^{2h(X_1)/d}=0$ when
$X_1$ has no density, gives
\[
 h(\theta\mid S=s_0)\ \ge\ \frac d2\log\big\{e^{2h(m\mid S=s_0)/d}+e^{2h(\zeta)/d}\big\}
 =\frac d2\log\big\{e^{2h(m\mid S=s_0)/d}+2\pi ev\big\},
\]
using $h(\zeta)=\frac d2\log(2\pi ev)$. The function
$x\mapsto\log(e^{2x/d}+2\pi ev)$ is convex and nondecreasing, so averaging
over $s_0$ and applying Jensen's inequality,
\[
 h(\theta\mid S)=\mathbb E\,h(\theta\mid S=s_0)
 \ \ge\ \frac d2\log\big\{e^{2h(m\mid S)/d}+2\pi ev\big\},
 \qquad h(m\mid S):=\mathbb E\,h(m\mid S=s_0).
\]
Next, $h(m\mid S)=h(m)-I(m;S)$, and $m$ is a function of $Y^n$, so
$I(m;S)\le I(Y^n;S)=R$ and
\[
 h(m\mid S)\ \ge\ h(m)-R=\frac d2\log(2\pi es^2)-R,
 \qquad\text{hence}\qquad
 e^{2h(m\mid S)/d}\ \ge\ 2\pi es^2e^{-2R/d}.
\]
Substituting into the previous display,
\[
 h(\theta\mid S)\ \ge\ \frac d2\log\big\{2\pi es^2e^{-2R/d}+2\pi ev\big\}
 =\frac d2\log(2\pi ev)+\frac d2\log\Big\{1+\frac{s^2}{v}e^{-2R/d}\Big\}.
\]
The first term is $h(\theta\mid Y^n)$, the entropy of $N(m,vI_d)$. By
Lemma~\ref{lem:identity}, and since all entropies here are finite,
\[
 \Delta_\KL(S)=I(\theta;Y^n)-I(\theta;S)=h(\theta\mid S)-h(\theta\mid Y^n)
 \ \ge\ \frac d2\log\{1+ae^{-2R/d}\}.
\]
The right side is decreasing in $R$, so every channel with
$I(Y^n;S)\le R$ satisfies the same bound, and
$D^\star_n(R)\ge\frac d2\log(1+ae^{-2R/d})$.

For achievability at $R>0$, let $S:=m+N$ with $N\sim N(0,\lambda I_d)$
independent of everything else and
\[
 \lambda:=\frac{s^2}{e^{2R/d}-1}.
\]
Since $Y^n\to m\to S$ is a Markov chain and $m$ is a function of $Y^n$,
$I(Y^n;S)=I(m;S)$, and the Gaussian channel formula gives
\[
 I(m;S)=\frac d2\log\Big(1+\frac{s^2}{\lambda}\Big)=\frac d2\log e^{2R/d}=R .
\]
Given $S$, the vector $m$ is Gaussian with covariance
$\frac{s^2\lambda}{s^2+\lambda}I_d=s^2\big(1-\frac{s^2}{s^2+\lambda}\big)I_d=s^2e^{-2R/d}I_d$,
because $s^2/(s^2+\lambda)=1-e^{-2R/d}$ for this $\lambda$. Hence
$\theta=m+\zeta$ given $S$ is Gaussian with covariance
$(s^2e^{-2R/d}+v)I_d$, and
\[
 \Delta_\KL(S)=h(\theta\mid S)-h(\theta\mid Y^n)
 =\frac d2\log\frac{s^2e^{-2R/d}+v}{v}=\frac d2\log\{1+ae^{-2R/d}\}.
\]
So the converse bound is attained, and $D^\star_n(R)=\frac d2\log(1+ae^{-2R/d})$
for $R>0$; at $R=0$ the constant summary has distortion
$I(\theta;Y^n)=\frac d2\log(1+a)$, which agrees with the formula.

For the rate function, $R\mapsto\frac d2\log(1+ae^{-2R/d})$ is a continuous
strictly decreasing bijection from $[0,\infty)$ onto $(0,\frac d2\log(1+a)]$.
For $0<\varepsilon<\frac d2\log(1+a)$, the converse shows that every channel
with distortion at most $\varepsilon$ has rate at least the solution $R$ of
$\frac d2\log(1+ae^{-2R/d})=\varepsilon$, and the noisy channel at that rate
has distortion exactly $\varepsilon$. Solving,
\[
 ae^{-2R/d}=e^{2\varepsilon/d}-1
 \quad\Longleftrightarrow\quad
 R=\frac d2\log\frac{a}{e^{2\varepsilon/d}-1},
\]
which is the second formula in \eqref{eq:gaussian_exact}. Finally, the
law of the random posterior $N(m,vI_d)$ is non-atomic because the law of
$m$ is, so Proposition~\ref{prop:shape} gives $R^\star_n(0)=\infty$.
\end{proof}

Under Assumption~\ref{ass:chart} the same model admits $T_n=m$ with $L_n=1/(2v)=O(n)$, since $$\KL(N(m,vI_d)\Vert N(m',vI_d))=\Vert m-m'\Vert^2/(2v),$$ so average quantization attains the same leading order even though by Proposition~\ref{prop:no_unif_cover} no finite uniform cover exists.

\section{Details of the parametric rate--distortion calculations}
\label{app:parametric-rd-details}

\begin{figure}[t]
\centering
\includegraphics[width=\linewidth]{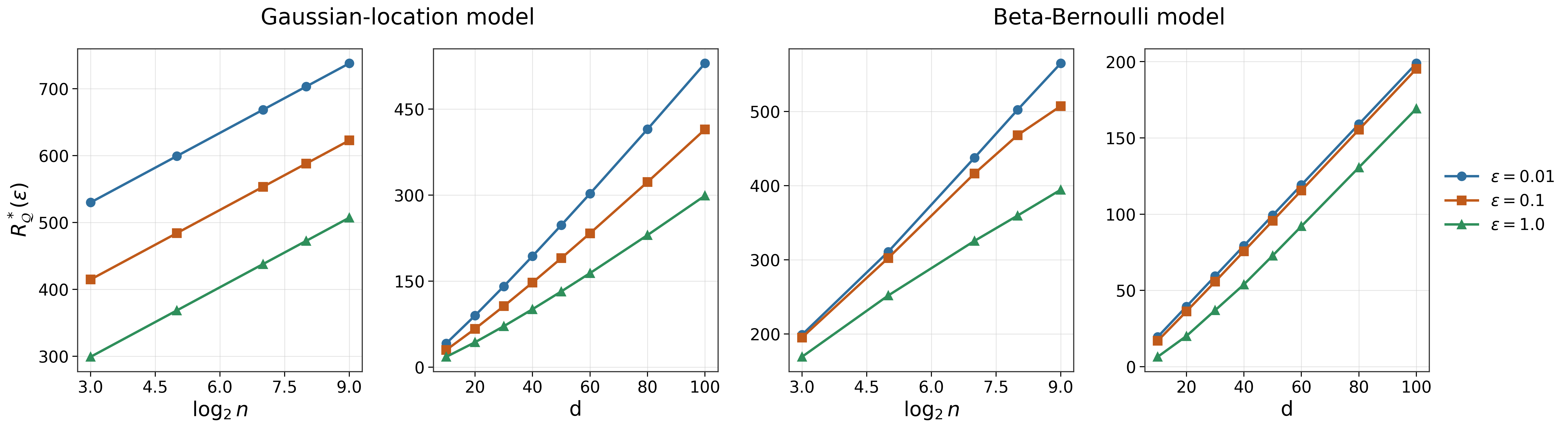}
\caption{Analytic rate calibration in parametric models. Gaussian panels evaluate the exact frontier in Proposition~\ref{thm:gaussian_exact}; Beta--Bernoulli panels report achievable rates within the specified noisy-count channel grid. Distortion budgets $\varepsilon\in\{0.01,0.1,1\}$ are total budgets. Left to right: Gaussian rate versus $n$ at $d=100$, Gaussian rate versus $d$ at $n=8$, and the corresponding Beta--Bernoulli quantities. The Beta--Bernoulli curves are not optimized over all channels.}
\label{fig:parametric-rate-distortion-bounds}
\end{figure}

Figure~\ref{fig:parametric-rate-distortion-bounds} uses the posterior distortion and information rate of Definition~\ref{def:intrinsic_rate_new}, in nats, with total distortion budgets rather than per-coordinate budgets.

\paragraph{Gaussian location model.}
For the Gaussian model in Proposition~\ref{thm:gaussian_exact}, the sufficient mean $T=n^{-1}\sum_iY_i$ has $I(\theta;Y^n)=\tfrac d2\log(1+n\tau^2/\sigma^2)$. We use $(\tau,\sigma)=(1,1)$ in the figure and evaluate the optimal channel
\[
S_\lambda = T + \eta,
\qquad
\eta \sim \mathcal N(0,\lambda I_d),
\]
with $\eta$ independent of $(\theta,Y^n)$. Its rate and distortion are
\begin{equation}
\label{eq:IYS_appendix}
R_\lambda
=
I(Y^n;S_\lambda)
=
I(T;S_\lambda)
=
\frac{d}{2}\log\!\left(1+\frac{\tau^2+\sigma^2/n}{\lambda}\right),
\end{equation}
and
\[
\Delta_\lambda
=
I(\theta;Y^n\mid S_\lambda)
=
I(\theta;T\mid S_\lambda)
=
\frac{d}{2}
\left[
\log\!\left(1+\frac{n\tau^2}{\sigma^2}\right)
-
\log\!\left(1+\frac{\tau^2}{\sigma^2/n+\lambda}\right)
\right].
\]
For $0<\varepsilon<I(\theta;Y^n)$, distortion increases and rate decreases with $\lambda$, so the optimal noise variance solves $\Delta_\lambda=\varepsilon$:
\[
\lambda_\varepsilon
=
\frac{\tau^2}{
\left(1+\frac{n\tau^2}{\sigma^2}\right)e^{-2\varepsilon/d}-1
}
-
\frac{\sigma^2}{n}.
\]
Proposition~\ref{thm:gaussian_exact} and its proof in Appendix~\ref{app:parametric} establish optimality against arbitrary channels. At fixed $d,\varepsilon,\tau,\sigma$, \eqref{eq:IYS_appendix} gives
\[
R_{\lambda_\varepsilon}
=
\frac d2
\log n
+
O(1)\qquad(n\to\infty).
\]
The figure uses the exact rate, including its fixed-total-tolerance $\tfrac d2\log d$ growth as $d\to\infty$ at fixed $n$; per-coordinate tolerance instead gives linear growth in $d$.

\paragraph{Beta--Bernoulli model.}
For the second experiment, we consider independent coordinates
\[
\theta_j \sim \mathrm{Beta}(a,b),
\qquad
Y_{ij}\mid \theta_j \sim \mathrm{Bernoulli}(\theta_j),
\qquad
j=1,\dots,d,\quad i=1,\dots,n,
\]
with $(a,b)=(5,5)$ in the figure. For each coordinate, the sufficient statistic is the count
\[
C_j := \sum_{i=1}^n Y_{ij}.
\]
We evaluate a noisy-count summary channel of the form
\[
S_j \sim q_\rho(\cdot \mid C_j),
\qquad
q_\rho(s\mid c)
\propto
\exp\!\left(-\frac{(s-c)^2}{2\rho^2}\right),
\qquad s\in\{0,\dots,n\},
\]
applied independently across coordinates. Because the model factorizes across $j$, both rate and distortion factorize:
\[
R_{\rho,d} = d\,R^{(1)}_\rho,
\qquad
\Delta_{\rho,d} = d\,\Delta^{(1)}_\rho.
\]

In one dimension, the prior-predictive law of the count is Beta--Binomial:
\[
p(c)
=
\binom{n}{c}\frac{B(a+c,b+n-c)}{B(a,b)},
\qquad c=0,\dots,n.
\]
The achieved rate is
\[
R^{(1)}_\rho
=
I(C;S)
=
\sum_{c=0}^n \sum_{s=0}^n
p(c)\,q_\rho(s\mid c)
\log \frac{q_\rho(s\mid c)}{p_\rho(s)},
\qquad
p_\rho(s):=\sum_{c=0}^n p(c)q_\rho(s\mid c).
\]
The full posterior given the sufficient statistic is
\[
\theta \mid C=c \sim \mathrm{Beta}(a+c,b+n-c),
\]
whereas the posterior given the summary value $S=s$ is the mixture
\[
p(\theta\mid s)
=
\sum_{c=0}^n p(c\mid s)\,\mathrm{Beta}(a+c,b+n-c).
\]
We compute the one-dimensional distortion as
\[
\Delta^{(1)}_\rho
=
\E{\KL\!\bigl(p(\theta\mid C)\,\|\,p(\theta\mid S)\bigr)}
=
H(\theta\mid S)-H(\theta\mid C),
\]
where the entropies are evaluated by Gauss--Legendre quadrature on $[0,1]$. In the experiments we use $300$ quadrature nodes and the grid
\(
\rho \in \cG= \{0.25, 0.35, 0.5, 0.7, 1.0, 1.4, 2.0, 2.8, 4.0, 5.6, 8.0, 11.3, 16.0, 22.6, 32.0, 45.2, 64.0\}.
\)
Since we do not invert $\rho \mapsto \Delta_\rho$ analytically, the plotted Beta--Bernoulli curves are grid-based achievable upper bounds:
\[
R^{\mathrm{BB}}_{\mathrm{grid}}(\varepsilon;n,d)
:=
\min_{\rho \in \mathcal G:\,\Delta_{\rho,d}\le \varepsilon} R_{\rho,d}.
\]

\paragraph{Numerical details.}
The Gaussian curves require only closed-form evaluation. The Beta--Bernoulli curves require only finite sums over $c,s\in\{0,\dots,n\}$ together with one-dimensional quadrature for the mixture posterior entropy. No neural networks are trained for these figures. When varying $n$, we use the values shown on the horizontal axis of Figure~\ref{fig:parametric-rate-distortion-bounds}; when varying $d$, we analogously use the displayed values of $d$.

\section{Experimental protocol and additional results}
\label{app:experiments-v6}

\subsection{Training and checkpoint selection}
The simulator experiments use a 256-dimensional encoder summary, a 128-bit code with fixed prior activation probability $0.01$, and a conditional normalizing flow with conditioner width 32 and 32 flow layers. Optimization uses Adam with learning rate $5\times10^{-5}$ throughout g-and-k. On the other simulators the initial rate is $10^{-4}$ for $\beta<0.1$ and $5\times10^{-5}$ otherwise, with the training schedule's decay with 72,000 steps, each with fresh batch of 128 samples. All Ricker configurations share the same fitted $\log(1+x)$ input normalizer.


The stochastic BNPE samples its code during training and deployment. The deterministic one samples Bernoulli codes with a temperature that decreases log-linearly from one to $0.1$ between $0.1E$ and $0.8E$, and then uses thresholded codes for stabilization. Both models use a straight-through gradient during training, and evaluation of the deterministic BNPE always uses a hard threshold. 

\subsection{Evaluation loss and information-rate estimates}
Each checkpoint is scored on five simulator sets, with  $M=2^{18}=262{,}144$ samples. NLL averages eight deployed code draws per observation for stochastic deployment and uses one pass for hard deployment. The rate calculation uses 65,536 sampled codes and the empirical mixture over the full reference set. 

For Bernoulli conditional probabilities $q_i(b)=q(b\mid y_i)$, define $\bar q_M(b)=M^{-1}\sum_iq_i(b)$. The empirical-mixture estimate is
\begin{equation}
 \widehat R_M\approx
  \mathbb{E}_{b\sim\bar q_M}[-\log\bar q_M(b)]
  -\frac{1}{M}\sum_i H(q_i).
\end{equation}
The finite empirical mixture's information is bounded by $\log M\simeq12.477$. Evaluation clips Bernoulli probabilities to $[10^{-5},1-10^{-5}]$.

A a two-hidden-layer MADE \citep{germain2015made} model of width 512 is fitted on an independently simulated dataset. For a fixed fitted marginal $r$, the population identity
\begin{equation}
 \mathbb{E}_{Y,B}[-\log r(B)]-H(B\mid Y)
   =I(Y;B)+\operatorname{KL}(p_B\Vert r)
\end{equation}
gives a variational upper bound.

\begin{figure}
 \centering
 \includegraphics[width=\linewidth]{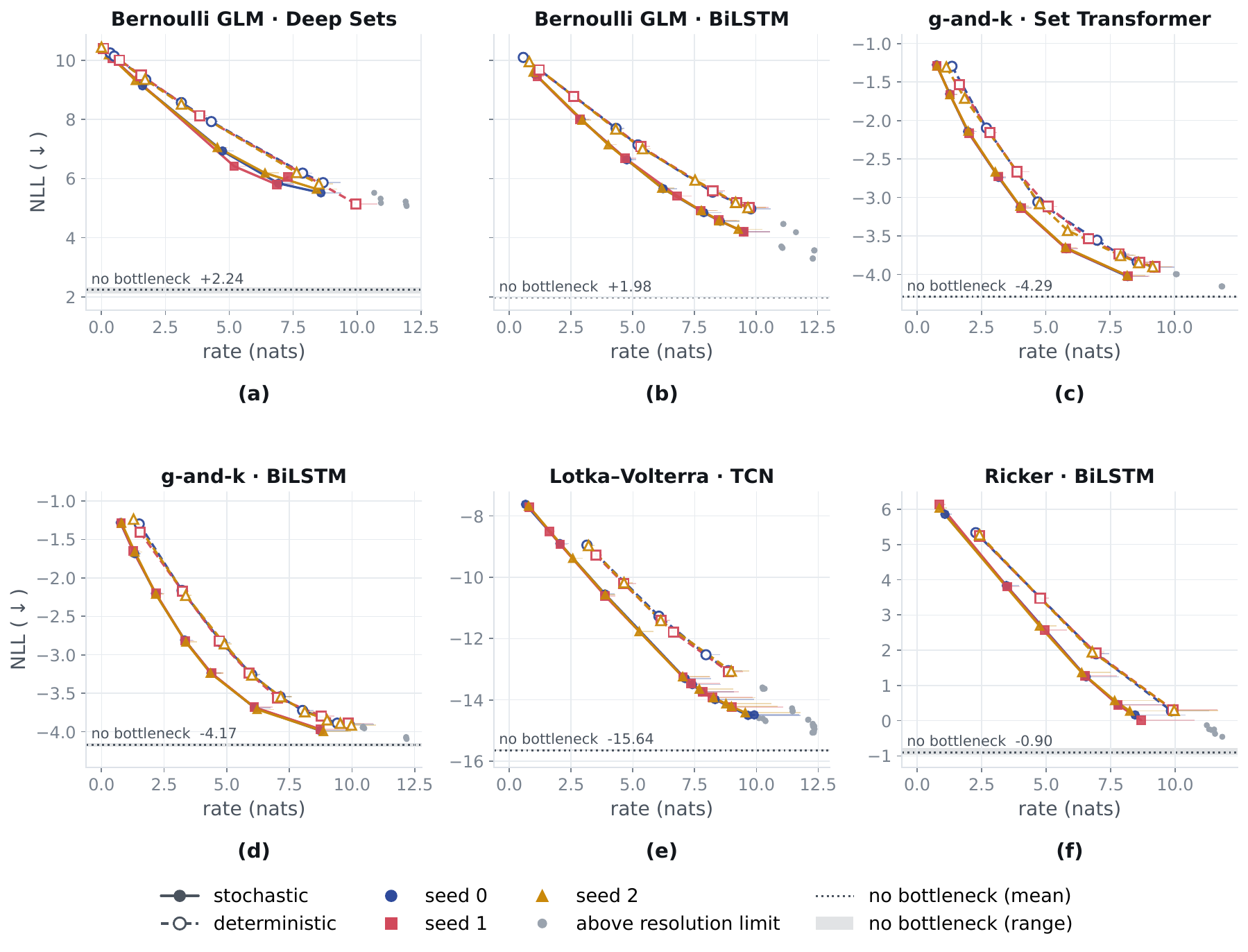}
 \caption{The six simulator--encoder combinations omitted from Figure~\ref{fig:arms-v6}.}
 \label{fig:arms-rest-v6}
\end{figure}

\subsection{Count benchmark computation}
The Beta--Bernoulli experiment enumerates count and code states and uses 512-point quadrature for the remaining integral over $\theta$. The deterministic comparison enumerates all partitions of the nine count values. The convex hull includes independent time-sharing between deterministic count encoders. Iterative information-bottleneck updates provide attainable stochastic channels. Neural training directly penalizes finite-sum $I(C;B)$, unlike the simulator experiments' sparse-prior KL surrogate. Reported neural summaries use five training seeds per arm and penalty coefficient. 

\subsection{Gaussian diagnostic with a fixed-covariance decoder}
\label{app:gaussian-diagnostic}
This controlled benchmark measures how efficiently a binary code preserves the
posterior mean when the decoder covariance is fixed, therefore, the family can not achieve the optimal decoder. 

Let
$\theta\sim N(0,I_4)$ and $Y_i\mid\theta\sim N(\theta,I_4)$ for $i=1,\ldots,4$.
Writing $T=\bar Y$, the full posterior is $N(\mu(T),vI_4)$ with
$\mu(T)=0.8T$ and $v=0.2$. We introduce a binary encoder with bit cardinality 16 and a Gaussian decoder with fixed posterior variance
$q_\phi(\theta\mid B)=N(g_\phi(B),vI_4)$. Both networks use two hidden
layers of width 256 with ReLU activations. The objective is NLL plus the
sparse-prior Bernoulli KL penalty with prior activation probability $0.01$.
Training uses 4000 Adam updates, batch size 4096, initial
learning rate $10^{-3}$ with cosine decay to $10^{-5}$, and training seed zero.

Let $g^\star(B)=\mathbb E[\mu(T)\mid B]$. The population decomposition is
\begin{align}
 \mathrm{NLL}_\phi-h(\theta\mid T)
 &=D_{\mathrm{mean}}(B)+G_\phi,\\
 D_{\mathrm{mean}}(B)
 &:=\frac{\mathbb E\|\mu(T)-g^\star(B)\|^2}{2v},\qquad
 G_\phi:=\frac{\mathbb E\|g^\star(B)-g_\phi(B)\|^2}{2v}.
\end{align}
Here $h(\theta\mid T)=2.45688$ nats. This separates the code's posterior-mean
error from the learned decoder's remaining mean error. The unrestricted
posterior loss satisfies (Lemma \ref{lem:identity}):
\begin{equation}
 \Delta_\KL(B)=D_{\mathrm{mean}}(B)
 -\mathbb E_B\KL\!\left(\Pi_B\,\Vert\,N(g^\star(B),vI_4)\right).
\end{equation}
The subtracted term includes covariance mismatch and any non-Gaussian shape
of the coarsened posterior. At zero rate, $D_{\mathrm{mean}}=8$, whereas
$\Delta_\KL=2\log5\simeq3.2189$. Thus the diagnostic uses a restricted
decoder family whose error remains meaningful even when its mean is fitted
optimally.

Since $\mu(T)\sim N(0,0.8I_4)$, Gaussian squared-error rate--distortion gives
the exact unrestricted-channel reference
$D_{\mathrm{mean}}(R)=8e^{-R/2}$, which is a lower bound for the finite binary
channels tested here. It is distinct from the posterior-KL frontier
$D^\star(R)=2\log(1+4e^{-R/2})$ in Proposition~\ref{thm:gaussian_exact}.

\begin{figure}[t]
 \centering
 \includegraphics[width=0.5\linewidth]{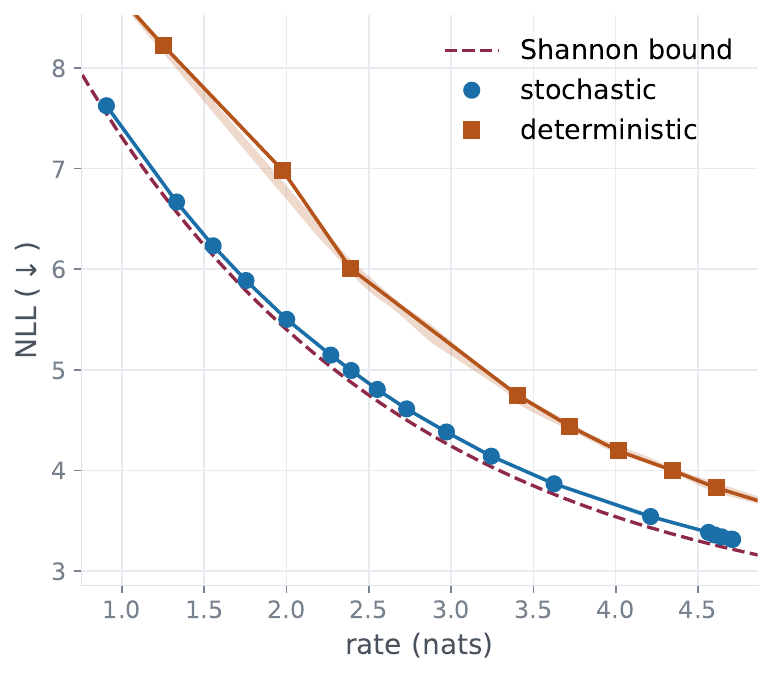}
 \caption{Gaussian posterior-mean diagnostic}
 \label{fig:gaussian-diagnostic}
\end{figure}

\subsection{Additional posterior diagnostics}
Neural ratio estimation offers a complementary route to measuring retained parameter information \citep{hermans2020ratio}. For the exact joint-to-product ratio $r(\theta,s)=p(\theta,s)/(\pi(\theta)p(s))$, its joint expected log ratio is $I(\theta;S)$. Comparing this quantity for full observations and summaries would target the information loss in Lemma~\ref{lem:identity}. With learned ratios, however, the difference also includes both estimators' approximation errors. The conjugate experiment directly evaluates posterior loss and decoder mismatch; the simulator results report NLL and code rate. 

\section{Related work and connections}
\label{app:related_works}
\label{sec:related}

\paragraph{Summary statistics for likelihood-free inference.}
Summary construction includes dimension reduction and subset selection for ABC
\citep{blum2013comparative,nunes2010optimal}, regression-based statistics
\citep{fearnhead2012constructing,jiang2017learning}, Fisher-information and
score-based compression \citep{charnock2018automatic,alsing2018generalized},
and neural summary learning
\citep{chen2020neural,chen2023learning}.
Exchangeable and convolutional encoders adapt summaries to observation
structure \citep{chan2018likelihood,aakesson2021convolutional}; amortized
posterior estimators often also train a summary jointly with its density decoder
\citep{radev2020bayesflow}.
Closest to the fidelity criterion here, \citet{hoffmann2026unifying} connect
expected posterior entropy, parameter--summary information, and expected
posterior KL loss. \citet{sui2023evaluating} evaluate cosmological summaries
through parameter--summary information. We place this posterior fidelity
alongside an explicit data--summary information budget, asking how much
compression is compatible with posterior preservation and comparing learned
encoders at matched estimated rates.

\paragraph{Posterior preservation and the information bottleneck.}
The information bottleneck (IB) trades information retained about a target
against information retained about the observation
\citep{tishby2000information,giladbachrach2003information}.
Under $\theta\to Y\to S$, the identity
$\Delta_{\mathrm{KL}}(S)=I(\theta;Y)-I(\theta;S)$ identifies our intrinsic
rate with this classical frontier, expressed in posterior-error coordinates.
Conditional-distribution KL distortion and posterior-mixture reconstruction
are classical IB constructions \citep{harremoes2007revisited}.
\citet{courtade2014multiterminal} explicitly study rate-limited posterior
compression under expected KL as a shifted logarithmic-loss coding problem.
Their finite-alphabet operational results concern coding across independent
experiments; our frontier concerns one use of a summary channel.
The sufficient-statistic reduction of \citet{armstrong2026source} allows
the observation to be replaced by its random posterior without changing the
unrestricted channel frontier. These foundations identify the object whose
representation cost governs posterior-preserving summaries.

\paragraph{Posterior quantization and converses.}
Our average posterior quantization result uses the nearest-center and centroid arguments often also used in $k$-means clustering problems; there the KL error replaced by Euclidean distance. Similar arguments are also used in the classical information theory literature that studies entropy-constrained and Bregman quantization
\citep{chou1989entropy,banerjee2005bregman}. The posterior mixture is the KL
instance of the Bregman centroid, connecting deterministic summaries to
average quantization of the random posterior.
Our converses specialize arbitrary-prior small-ball information bounds
\citep{xu2016information} and general Shannon lower bounds
\citep{riegler2018rate} to this source. They retain the predictive mass of
distinguishable posterior states, which support covering and packing numbers
alone do not capture. The resulting question is how finely the \emph{law of
the random posterior}, under a fixed Bayesian experiment, must be represented
to meet a prescribed average distortion.

\citet{haussler1997mutual} identify Bayes information with cumulative Bayesian
predictive relative-entropy risk and bound it using prior-weighted geometry
of sampling distributions. The compensation inequality transfers any such
lower bound directly to $R_n^\star(\varepsilon)\ge
I(\theta;Y_{1:n})-\varepsilon$.
In particular, suppose $Y_{1:n}\mid\theta\sim P_\theta^{\otimes n}$ satisfies
their dominated-model setup. With this paper's normalization
$\Hel^2(P,Q)=\tfrac12\int(\sqrt p-\sqrt q)^2$, their Corollary~1(1)
(p.~2466, with $\alpha=1/2$) gives
\[
R_n^\star(\varepsilon)\ \ge
\left[-\int\pi(d\theta)\log\!\int\pi(d\vartheta)
  \exp\{-n\Hel^2(P_\theta,P_\vartheta)\}
  -\varepsilon\right]_+.
\]
This lower bound needs no uniform likelihood-ratio bound. Their corresponding
Hellinger upper bounds require additional control (Corollary~1 or Theorem~3);
mutual absolute continuity alone is insufficient. Their Section~7 takes a
supremum over priors to obtain minimax bounds from sampling-model covering
and packing numbers. Their Theorem~4(3), pp.~2482--2483, likewise concerns
minimax predictive risk. These results control information acquired about a
parameter from repeated sampling. Our posterior-space bounds control
compression under a fixed predictive law, and the fixed-prior parametric
converse uses the prior-averaged Clarke--Barron expansion below.

\paragraph{Parametric information and approximate sufficiency.}
\citet{clarkeBarron1994} give the regular, prior-averaged Bayes-information
expansion underlying the $\tfrac d2\log n$ converse in
Lemma~\ref{lem:cb_benchmark}.
\citet{hayashi2018minimum} obtain the same leading coefficient for
finite-memory approximate sufficient-statistic coding under their
reconstruction and alphabet-size criteria. Here distortion is posterior KL
averaged under a fixed prior, and rate is data--summary mutual information;
matching upper bounds follow under the stated posterior-chart conditions.
The exact Gaussian benchmark specializes Gaussian IB
\citep{chechik2005gaussianib,asoodeh2020bottleneck}, making the dependence on
sample size, dimension, and posterior tolerance explicit.

\paragraph{Stochastic summaries and deterministic time-sharing.}
Stochastic assignments are central to classical IB, and finite-alphabet
entropy tradeoffs have long been studied through convex representations
\citep{tishby2000information,witsenhausen1975conditional,wyner1973theorem}.
Agglomerative IB expresses the relevance cost of merging cells through
weighted Jensen--Shannon divergence \citep{slonim2000agglomerative}.
For deterministic targets, the IB frontier is piecewise linear and can be
attained by visible time-sharing of deterministic summaries
\citep{kolchinsky2019caveats}; full-support experiments can also have linear
segments \citep{benger2023cardinality}.
Theorem~\ref{thm:separation} compares stochastic summaries with the entire
visible time-sharing envelope of deterministic summaries. For finite
observations, rare confusion between two cells yields a support-sensitive
refinement of the weighted Jensen--Shannon merge cost. The general theorem
replaces this pair-specific control by a uniform, experiment-level bound for
random relabeling. Combined with finite-quantizer approximation and uniform
integrability, it yields strict improvement at every interior rate for
standard Borel experiments satisfying the common-support condition, without
requiring finite pairwise posterior KL or attainment of the deterministic
envelope.

\paragraph{The compression cost and the role of randomness.}
The deterministic IB of \citet{strouse2017deterministic} prices output entropy,
changing the optimization problem. For ordinary additive coding-cost
Lagrangians, encoder randomness can be fixed without increasing the objective;
distribution-preserving reconstruction constraints can instead make shared
randomness useful \citep{theis2021advantages}.
Our separation theorem charges mutual information and permits private encoder
randomness. Relabeling preserves the output distribution while reducing
mutual information, so its gain depends on that choice of cost.

\paragraph{Learning and evaluating an information budget.}
The training penalty builds on neural variational IB \citep{alemi2016deep}
and task-oriented inference compression \citep{shao2022learning}.
The binary-bottleneck protocol makes the summary channel explicit and permits
comparisons across architectures and stochastic or deterministic modes.
We measure deployed rate separately from the variational training penalty;
the empirical rate estimates and decoder losses are also distinguished from
the population posterior-information frontier.

\end{document}